\def\year{2021}\relax
\documentclass[letterpaper]{article} 
\usepackage[utf8]{inputenc} 
\usepackage{lipsum} 

\usepackage{etoolbox}
\newbool{includeappendix}
\setbool{includeappendix}{true}

\ifdefined\isoverfull
\else
\fi

\newcommand{\eg}{e.g., }
\newcommand{\ie}{i.e., }

\usepackage{acro} 

\DeclareAcronym{cli} {
    short = CLI,
    long = Command Line Interface,
    class = abbrev
}

\usepackage{subcaption}

\usepackage[usenames, dvipsnames]{color} 
\usepackage{xcolor} 

\definecolor{my-full-blue}{HTML}{1F77B4}

\definecolor{my-full-orange}{HTML}{FF7F0E}

\definecolor{my-full-green}{HTML}{2CA02C}

\definecolor{my-full-red}{HTML}{d62728}

\definecolor{my-full-purple}{HTML}{9467bd}

\colorlet{my-blue}{my-full-blue!30}
\colorlet{my-orange}{my-full-orange!30}
\colorlet{my-green}{my-full-green!30}
\colorlet{my-red}{my-full-red!30}
\colorlet{my-purple}{my-full-purple!30}

\usepackage{listings}

\usepackage{textcomp}

\usepackage{xcolor}

\usepackage[scaled=0.8]{beramono}

\definecolor{ckeyword}{HTML}{7F0055}
\definecolor{ccomment}{HTML}{3F7F5F}
\definecolor{cstring}{HTML}{2A0099}

\lstdefinestyle{numbers}{
	numbers=left,
	framexleftmargin=20pt,
	numberstyle=\tiny,
	firstnumber=auto,
	numbersep=1em,
	xleftmargin=2em
}

\lstdefinestyle{layout}{
	frame=none,
	captionpos=b,
}

\lstdefinestyle{comment-style}{
	morecomment=[l]//,
	morecomment=[s]{/*}{*/},
	commentstyle={\color{ccomment}\itshape},
}

\lstdefinestyle{string-style}{
	morestring=[b]",%
	morestring=[b]',%
	stringstyle={\color{cstring}},
	showstringspaces=false,%
}

\lstdefinestyle{keyword-style}{
	keywordstyle={\ttfamily\bfseries},
	morekeywords={
		function,
		constructor,
		int,
		bool,
		return,
		returns,
		uint
	},
	morekeywords = [2]{},
	keywordstyle = [2]{\text},
	sensitive=true,
}

\lstdefinestyle{input-encoding}{
	inputencoding=utf8,
	extendedchars=false,
	literate=
	{ℝ}{$\reals$}1%
	{→}{$\rightarrow$}1%
	{α}{$\alpha$}1%
	{β}{$\beta$}1%
	{λ}{$\lambda$}1%
	{θ}{$\theta$}1%
	{ϕ}{$\phi$}1%
}

\lstdefinestyle{escaping}{
	moredelim={**[is][\color{blue}]{\%}{\%}},
	escapechar=|,
	mathescape=true
}

\lstdefinestyle{default-style}{
	basicstyle=\fontencoding{T1}\ttfamily\footnotesize,
	style=numbers,
	style=layout,
	style=comment-style,
	style=string-style,
	style=keyword-style,
	style=input-encoding,
	style=escaping,
	tabsize=2,
	upquote=true
}

\lstdefinelanguage{BASIC}{
	language=C++,
	style=default-style
}[keywords,comments,strings]%

\usepackage{tikz}

\usetikzlibrary{arrows}
\usetikzlibrary{automata}
\usetikzlibrary{calc}
\usetikzlibrary{backgrounds}
\usetikzlibrary{decorations.markings}
\usetikzlibrary{decorations.pathmorphing}
\usetikzlibrary{decorations.pathreplacing}
\usetikzlibrary{fit}
\usetikzlibrary{patterns}
\usetikzlibrary{positioning}
\usetikzlibrary{shadows}
\usetikzlibrary{shapes}
\usetikzlibrary{shapes.geometric}

\usepackage{aaai21}  
\usepackage{times}  
\usepackage{helvet} 
\usepackage{courier}  
\usepackage[hyphens]{url}  
\usepackage{graphicx} 
\usepackage{natbib}  
\usepackage{caption} 
\usepackage{microtype}
\usepackage{booktabs} 
\usepackage{multirow}

\usetikzlibrary{positioning}

\usepackage{url}
\usepackage{adjustbox}

\usepackage[ruled,linesnumbered]{algorithm2e}
\usepackage{amsmath, amssymb, amsthm}
\usepackage{amsfonts}       
\usepackage{nicefrac}       
\usepackage{multirow}
\usepackage{siunitx}
\usepackage{thmtools}
\usepackage{thm-restate}

\theoremstyle{definition}

\definecolor{myblue}{HTML}{3399CC}
\definecolor{mygreen}{HTML}{00AA00}

\usepackage{amsmath,amsfonts,bm}

\def\1{\bm{1}}

\DeclareMathAlphabet{\mathsfit}{\encodingdefault}{\sfdefault}{m}{sl}
\SetMathAlphabet{\mathsfit}{bold}{\encodingdefault}{\sfdefault}{bx}{n}

\def\sR{{\mathbb{R}}}

\DeclareMathOperator*{\argmin}{arg\,min}

\usepackage[capitalize, noabbrev]{cleveref}

\newcommand{\app}[1]{%
\ifbool{includeappendix}{\cref{#1}}{the appendix}%
}
\newcommand{\App}[1]{%
\ifbool{includeappendix}{\cref{#1}}{The appendix}%
}

\Crefname{lem}{Lemma}{Lemmas}

\title{Efficient Certification of Spatial Robustness}
\author{
    Anian Ruoss, Maximilian Baader, Mislav Balunovi\'{c}, Martin Vechev \\
}
\affiliations{
    Department of Computer Science \\
    ETH Zurich \\
    anruoss@ethz.ch \\
    \{mbaader, mislav.balunovic, martin.vechev\}@inf.ethz.ch
}

\begin{document}

    \maketitle


    \begin{abstract}
        Recent work has exposed the vulnerability of computer vision models to vector
field attacks.
Due to the widespread usage of such models in safety-critical applications, it
is crucial to quantify their robustness against such spatial transformations.
However, existing work only provides empirical robustness quantification against
vector field deformations via adversarial attacks, which lack provable
guarantees.
In this work, we propose novel convex relaxations, enabling us, for the
first time, to provide a certificate of robustness against vector field
transformations.
Our relaxations are model-agnostic and can be leveraged by a wide range
of neural network verifiers.
Experiments on various network architectures and different datasets demonstrate
the effectiveness and scalability of our method.

    \end{abstract}

    \section{Introduction}
\label{sec:introduction}

It was recently shown that neural networks are susceptible not only to standard
noise-based adversarial perturbations~\cite{szegedy2014intriguing,
goodfellow2015explaining, carlini2016towards, madry2018towards} but also to
\emph{spatially transformed} images that are visually indistinguishable from
the original~\cite{kanbak2018geometric, alaifari2018adef, xiao2018spatially,
engstrom2019exploring}.
Such spatial attacks can be modeled by smooth \emph{vector fields} that
describe the displacement of every pixel.
Common geometric transformations, \eg rotation and translation, are particular
instances of these smooth vector fields, which indicates that they capture
a wide range of naturally occurring image transformations.

Since the vulnerability of neural networks to spatially transformed adversarial
examples can pose a security threat to computer vision systems relying on such
models, it is critical to quantify their robustness against spatial
transformations.
A common approach to estimate neural network robustness is to measure the
success rate of strong attacks~\cite{carlini2016towards, madry2018towards}.
However, many networks that are indeed robust against these attacks were later
broken using even more sophisticated attacks~\cite{athalye2018robustness,
athalye2018obfuscated, engstrom2018evaluating, tramer2020adaptive}.
The key issue is that such attacks do not provide provable robustness guarantees.

\begin{figure}
    \begin{center}
        \begin{subfigure}[t]{.19\linewidth}
            \centering
            \includegraphics[width=\textwidth]{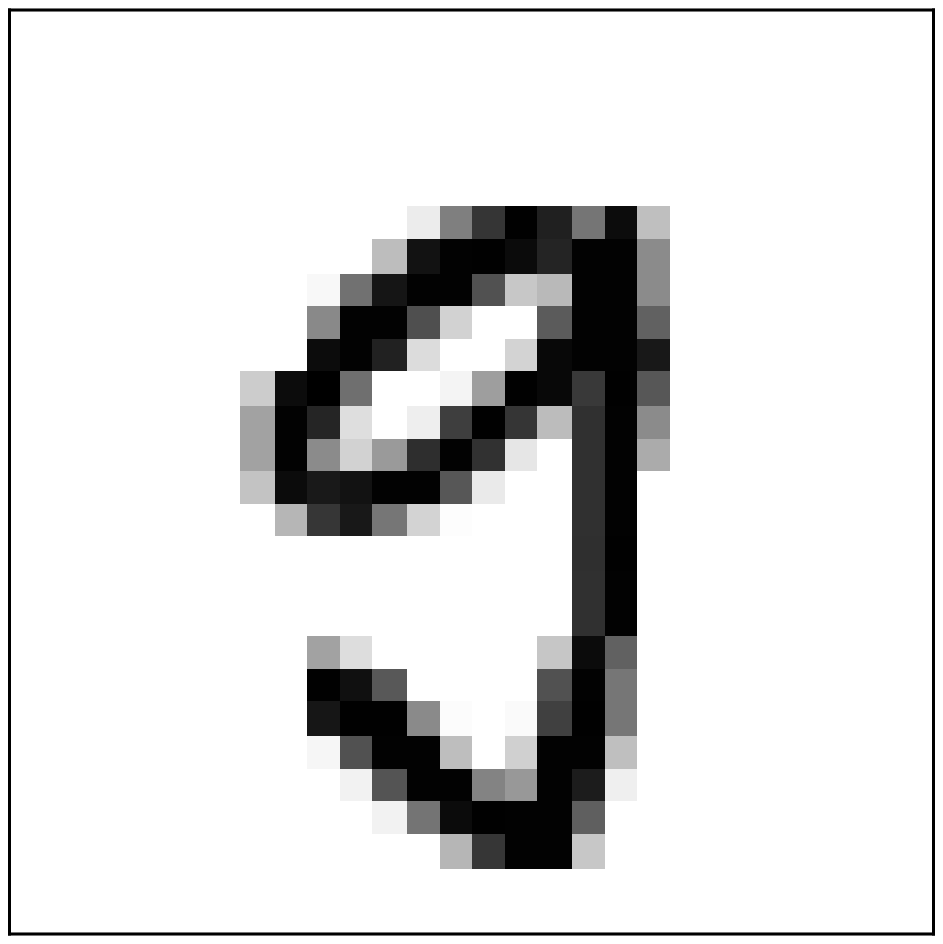}
            \subcaption{Original}
            \label{fig:intro:orig}
        \end{subfigure}
        \begin{subfigure}[t]{.38\linewidth}
            \centering
            \includegraphics[width=0.5\textwidth]{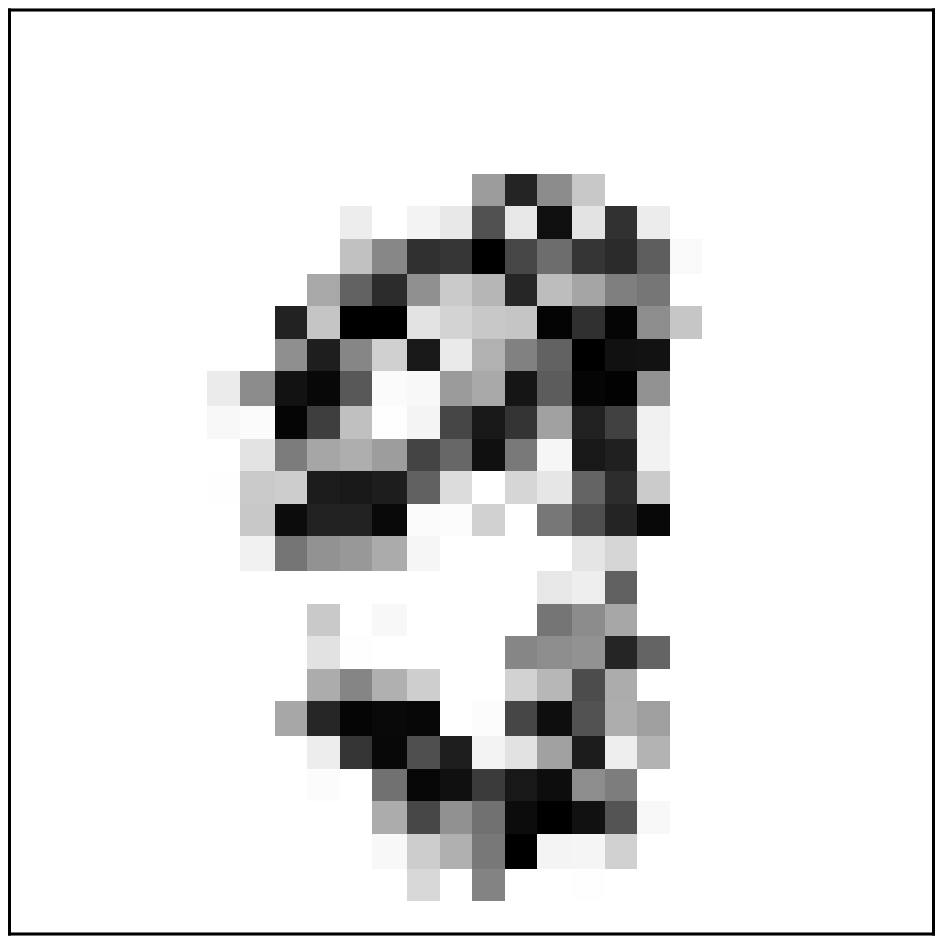}
            \hspace{-0.2cm}
            \includegraphics[width=0.5\textwidth]{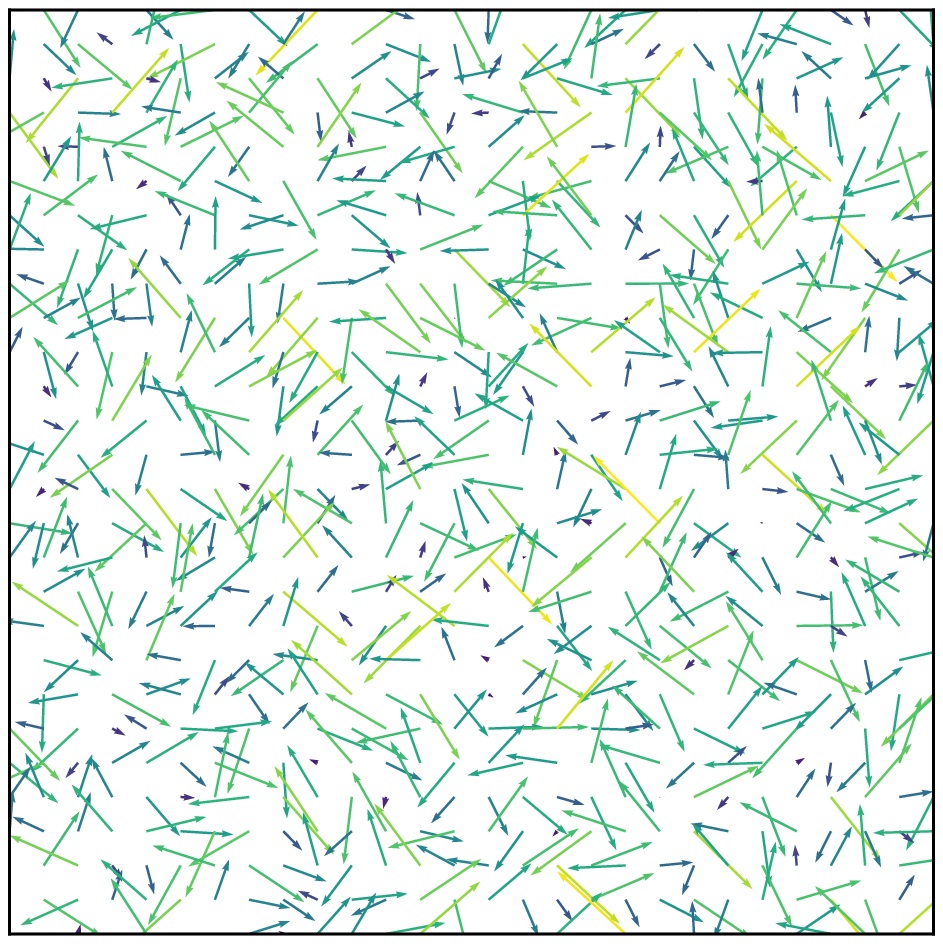}
            \subcaption{$\delta = 1$, $\gamma = \infty$}
            \label{fig:intro:non-smooth}
        \end{subfigure}
        \begin{subfigure}[t]{.38\linewidth}
            \centering
            \includegraphics[width=0.5\textwidth]{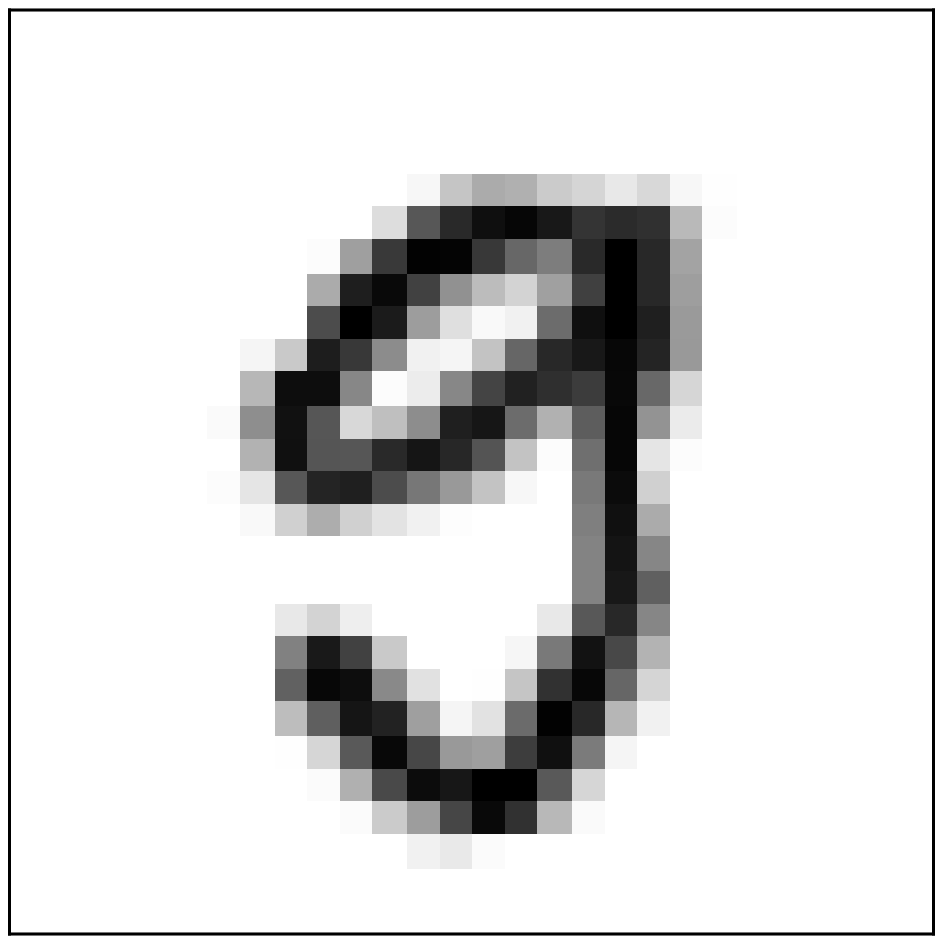}
            \hspace{-0.2cm}
            \includegraphics[width=0.5\textwidth]{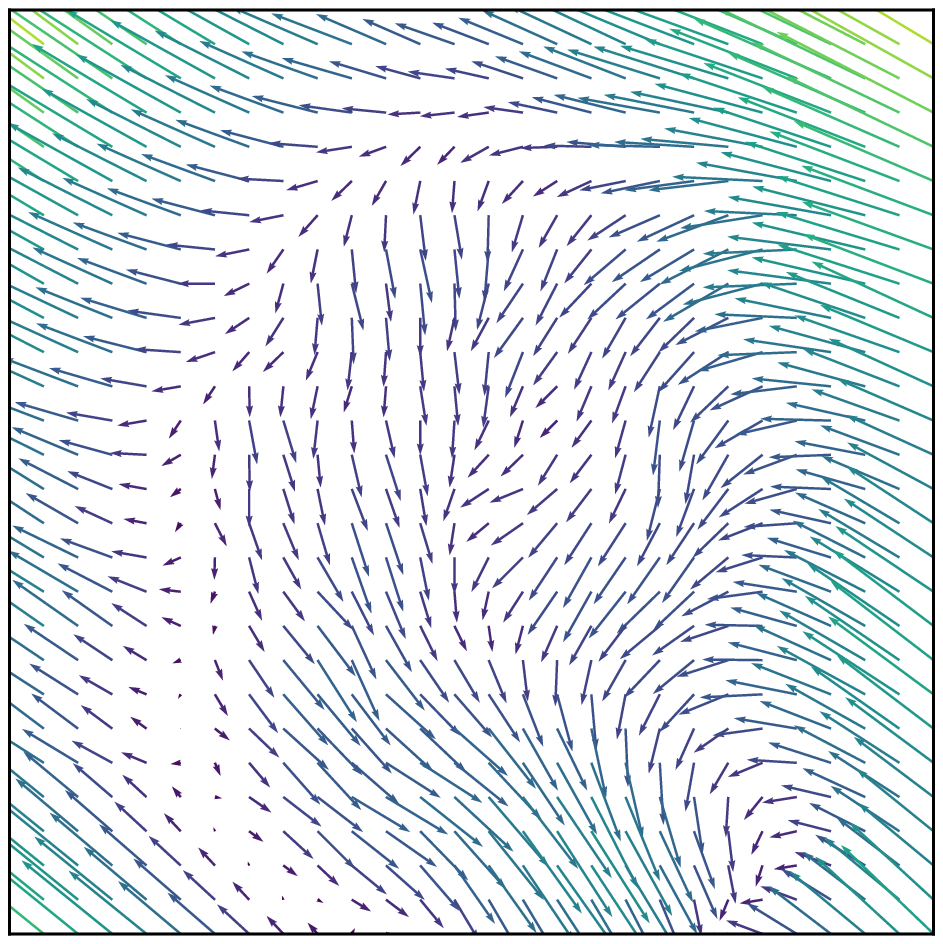}
            \subcaption{$\delta = 3$, $\gamma = 0.25$}
            \label{fig:intro:smooth}
        \end{subfigure}
    \end{center}
    \caption{
        Image instances and corresponding deforming vector fields
        ($\gamma$ and $\delta$ explained below): (a)~original,
        (b)~adversarially deformed (label 5) by non-smooth vector field, and
        (c)~adversarially deformed (label 3) by smooth vector field.
    }
    \label{fig:intro}
\end{figure}

To address this issue, our goal is to provide a provable \emph{certificate} that
a neural network is robust against all possible deforming vector fields within
an attack budget for a given dataset of images.
While various certification methods exist, they are limited to noise-based
perturbations~\cite{katz2017reluplex, gehr2018ai2, wong2018provable, singh2018fast,
zhang2018efficient} or compositions of common geometric transformations (\eg
rotations and translations)~\cite{pei2017towards, singh2019abstract,
balunovic2019geometric, mohapatra2020towards} and thus cannot be applied to our
setting.

A common approach in robustness certification is to compute, for a given image,
pixel bounds containing all possible perturbed images within some attack
budget, and then propagate these bounds through the network to obtain bounds
on the output neurons.
Then, if all images within the output bounds classify to the correct label, the
network is provably robust against all attacks limited to the same attack
budget.
In our work, we intuitively parametrize the attack budget by the magnitude of
pixel displacement, denoted by $\delta$, and the smoothness of the vector
field, denoted by $\gamma$.
This allows us to efficiently compute the tightest-possible pixel interval
bounds on vector field deformations limited to a given displacement magnitude
$\delta$ by using a mathematical analysis of the transformation.
However, even small but non-smooth vector fields (\ie small $\delta$ but large
$\gamma$) can generate large pixel differences, resulting in recognizably
perturbed images (\cref{fig:intro:non-smooth}) and leading to large pixel
bounds, which limit certification performance.
Thus, a key challenge is to define smoothness constraints that can be
efficiently incorporated with neural network verifiers to enable certification
of smooth vector fields with large displacement magnitude
(\cref{fig:intro:smooth}).

Hence, we tighten our convex relaxation for smooth vector fields by introducing
smoothness constraints that can be efficiently incorporated into
state-of-the-art verifiers~\cite{tjeng2019evaluating, singh2019beyond,
singh2019boosting}.
To that end, we leverage the idea of computing linear constraints on the pixel
values in terms of the transformation parameters~\cite{balunovic2019geometric,
mohapatra2020towards}.
We show that our mathematical analysis of the transformation induces an
optimization problem for computing the linear constraints, which can be
efficiently solved by linear programming.
Finally, we show that the idea by \citet{balunovic2019geometric,
mohapatra2020towards} alone is insufficient for the setting of smooth vector
fields, and only the combination with our smoothness constraints yields superior
certification performance.
We implement our method in an open-source system and show that our convex
relaxations can be leveraged to, for the first time, certify neural network
robustness against vector field attacks.

\paragraph{Key contributions} We make the following contributions:

\begin{itemize}
    \item A novel method to compute the tight interval bounds for
        norm-constrained vector field attacks, enabling the first certification
        of neural networks against vector field attacks.
    \item A tightening of our relaxation for smooth vector fields and
        integration with state-of-the-art robustness certifiers.
    \item An open-source implementation together with extensive
        experimental evaluation on the MNIST and CIFAR-10 datasets,
        with convolutional and large residual networks.
        We make our code publicly available as part of the ERAN framework for
        neural network verification (available at
        \url{https://github.com/eth-sri/eran}).
\end{itemize}

    \section{Related Work}
\label{sec:related-work}

Here we discuss the most relevant related work on spatial robustness and certification
of neural networks.

\paragraph{Empirical spatial robustness}

In addition to previously known adversarial examples based on $\ell_p$-norm
perturbations, it has recently been demonstrated that adversarial examples can
also be constructed via geometric transformations~\cite{kanbak2018geometric},
rotations and translations~\cite{engstrom2019exploring}, Wasserstein
distance~\cite{wong2019wasserstein, levine2020wasserstein, hu2020improved}, and
vector field deformations~\cite{alaifari2018adef, xiao2018spatially}.
Here, we use a threat model based on vector field deformations for which both
prior works have proposed attacks: \citet{alaifari2018adef} perform first-order
approximations to find minimum-norm adversarial vector fields, and
\citet{xiao2018spatially} relax vector field smoothness constraints with a
continuous loss to find perceptually realistic adversarial examples.

\paragraph{Robustness certification}

There is a long line of work on certifying the robustness of neural networks to
noise-based perturbations.
These approaches employ SMT solvers~\cite{katz2017reluplex}, mixed-integer
linear programming~\cite{tjeng2019evaluating, singh2019boosting},
semidefinite programming~\cite{raghunathan2018semidefinite}, and linear
relaxations~\cite{gehr2018ai2, wong2018provable, singh2018fast, zhang2018efficient,
wang2018efficient, weng2018towards, singh2019abstract, salman2019convex,
lin2019robustness}.
Another line of work also considers certification via randomized
smoothing~\cite{lecuyer2019certified, cohen2019smoothing, salman2019provably},
which, however, only provides probabilistic robustness guarantees for smoothed
models whose predictions cannot be evaluated exactly, only approximated to
arbitrarily high confidence.

\paragraph{Certified spatial robustness}

Prior work introduces certification methods for special cases of spatial
transformations: a finite number of transformations~\cite{pei2017towards},
rotations~\cite{singh2019abstract}, and compositions of common geometric
transformations~\cite{balunovic2019geometric, mohapatra2020towards}.
Some randomized smoothing approaches exist, but they only handle single
parameter transformations~\cite{li2020provable} or transformations without
compositions~\cite{fischer2020certified}.
In a different setting, \citet{wu2020robustness} compute the maximum safe
radius on optical flow video perturbations but only for a finite set of
neighboring grid points.
Overall, previous approaches are limited because they only certify
transformations in specific templates that can be characterized as special cases
of smooth vector field deformations.
Certifying these vector field deformations is precisely the goal of our work.

While rotations are special cases of smooth deformations, even small rotations
can cause not only large pixel displacements (\ie large $\delta$) far from the
center of rotation but also large smoothness constraints (\ie large $\gamma$)
due to the singularity at the center of rotation.
Since such large $\gamma$ and $\delta$ values are currently out of reach for our
method, we advise using specialized certification
methods~\cite{balunovic2019geometric, mohapatra2020towards} when considering
threat models consisting only of rotations.
However, we could combine our approach with these specialized certification
methods, \eg by instantiating DeepG~\cite{balunovic2019geometric} with our
interval bounds, which would correspond to first deforming an image by a vector
field and then rotating the deformed image.

    \section{Background}
\label{sec:background}

Here we introduce our notation, define the similarity metric for spatially
transformed images, and provide the necessary background for neural network
certification.

\paragraph{Vector field deformations}

We represent an image as a function
$I \colon P \subseteq \sR^2 \rightarrow \sR^C$, where $C$ is the number of color
channels, and $P = \{1, 2, \ldots, W\}^2$ corresponds to the set of pixel
coordinates of the image with dimension $W \times W$.

Vector field transformations are parameterized by a vector field
$\tau \colon P \rightarrow \sR^2$ assigning a displacement vector $\tau(i,j)$ to
every pixel $(i,j)$.
Thus, we obtain the deformed pixel coordinates via
$(\mathbb{I} + \tau) \colon P \rightarrow \mathbb{R}^2$ with
$(\mathbb{I} + \tau) (i, j) = (i, j) + \tau(i, j)$, where $\mathbb{I}$ is the
identity operator.
Since these deformed coordinates may not lie on the integer grid, we use
bilinear interpolation $\mathcal{I}_I : \sR^2 \rightarrow \sR^C$ induced by the
image $I$ and evaluated at $(i,j) + \tau(i,j)$ to get the deformed pixel values:
\begin{align}
    \mathcal{I}_I \left( i, j \right) &:=
    \begin{cases}
        \mathcal{I}_I^{mn} \left( i, j \right) & \text{if} \left( i, j \right) \in A_{mn}
    \end{cases}, \nonumber
    \\
    \mathcal{I}^{mn}_I \left( i, j \right) &:=
    \sum_{\substack{
        p \in \left\{ m, m + 1 \right\} \\ q \in \left\{ n, n + 1 \right\}
    }}
    I(p, q)
    \left( 1 - \lvert p - i \rvert \right)
    \left( 1 - \lvert q - j \rvert \right),
    \label{eq:bil_interpolation}
\end{align}
where $A_{mn} := \left[ m, m + 1 \right] \times \left[ n, n + 1 \right]$ is an
\emph{interpolation region}.
Hence, we define the deformed image as
$\mathcal{I}_I \circ \left(\mathbb{I} + \tau\right)$.
Like \citet{alaifari2018adef}, we only study vector fields that do not move
pixels outside the image.

\paragraph{Estimating perceptual similarity}

For noise-based adversarial attacks, an adversarial image $I_{adv}$ is typically
considered perceptually similar to the original image $I$ if the $\ell_p$-norm
of the perturbation $\left\| I - I_{adv} \right\|_p$ is
small~\cite{szegedy2014intriguing, goodfellow2015explaining, carlini2016towards,
madry2018towards}.
However, prior work~\cite{alaifari2018adef, xiao2018spatially} has demonstrated
that this is not necessarily a good measure for spatially transformed images.
For example, translating an image by a small amount will typically produce an
image that looks very similar to the original but results in a large
perturbation with respect to the $\ell_p$-norm.
For this reason, the similarity of spatially transformed images is typically
estimated with a norm on the deforming vector field and not on the pixel value
perturbation.
Here, we consider the $T_p$-norm~\cite{alaifari2018adef}, defined as
\begin{equation*}
    \left\| \tau \right\|_{T_p} :=
    \max_{(i,j) \in P} \left\| \tau \left( i, j \right) \right\|_p.
\end{equation*}
We consider $p \in \{1, 2, \infty\}$ and note that the case $p=2$ corresponds
to the norm used by \citet{alaifari2018adef}.
Intuitively, a vector field with $T_2$-norm at most $1$ will displace any pixel
by at most one grid length on the image grid.
In general, for a vector field $\tau$ with
$\left\| \tau \right\|_{T_p} \leq \delta$, the set of reachable coordinates
from a single pixel $\left( i, j \right) \in P$ is
\begin{equation*}
    B_\delta^p \left( i,j \right) :=
    \left\{
    x \in \mathbb{R}^2 \mid
    \left\| \left( i, j \right)  - x \right\|_p \leq \delta
    \right\}.
\end{equation*}

However, there may be vector fields with small $T_p$-norm that produce
unrealistic images.
For example, moving every pixel independently by at most one grid length can
already result in very pixelated images that can be easily recognized as
unnatural when comparing with the original (\cref{fig:intro:non-smooth}).
To address this, \citet{xiao2018spatially} introduce a flow loss that
penalizes the vector field's lack of smoothness.
Following this approach, we say that vector field $\tau$ has flow $\gamma$ if
it satisfies, for each pixel $(i, j)$, the \textit{flow-constraints}
\begin{equation}
    \label{eq:flow}
    ||\tau(i, j) - \tau(i', j')||_\infty \leq \gamma,
    \forall \left(i', j'\right) \in \mathcal{N}(i, j),
\end{equation}
where $\mathcal{N}(i, j) \subseteq P$ represents the set of neighboring pixels
in the 4 cardinal directions of pixel $(i, j)$.
For instance, translation is parametrized by a vector field that has flow $0$
(each pixel has the same displacement vector).
These constraints enforce smoothness of the vector field $\tau$, which in turn
ensures that transformed images look realistic and better preserve image
semantics -- even for large values of $\delta$ (\cref{fig:intro:smooth}).
We provide a more thorough visual investigation of the norms and constraints
considered in~\cref{sec:visual-investigation}.

\begin{figure*}
    \begin{center}
    \resizebox{0.75\linewidth}{!}{ \begin{tikzpicture}

    \pgfmathsetmacro{\dx}{3.25}
    \pgfmathsetmacro{\dy}{2}

    \draw (-1.25,-5.5) rectangle (11,3.5);
    \draw (-6.5,-5.75) rectangle (11.25,3.75);
    \node (L1) at (-3.5, 3.25) {Pixel Interval Bounds};
    \node (LP1) at (-1.25, 3.25) {};
    \node (L2) at (-4.5, -5.25) {Spatial Constraints};
    \node (LP2) at (-6.5, -5.25) {};
    \path [->] (L1) edge node[left,below] {} (LP1);
    \path [<-] (LP2) edge node[left,below] {} (L2);

    \node[shape=circle,draw=gray] (Vx) at (-0.75 * \dx,\dy) {$v_x$};
    \node[shape=circle,draw=gray] (Vy) at (-0.75 * \dx,0) {$v_y$};
    \node[shape=circle,draw=gray] (Wx) at (-0.75 * \dx,-\dy) {$w_x$};
    \node[shape=circle,draw=gray] (Wy) at (-0.75 * \dx,-2*\dy) {$w_y$};

    \node[shape=circle,draw=myblue,line width=2] (A) at (0,0) {$x_{0}$};
    \node[shape=circle,draw=mygreen,line width=2] (B) at (0,-\dy) {$x_{1}$};
    \node[shape=circle,draw=gray] (C) at (\dx,0) {$x_{2}$};
    \node[shape=circle,draw=gray] (D) at (\dx,-\dy) {$x_{3}$};
    \node[shape=circle,draw=gray] (E) at (2*\dx,0) {$x_{4}$};
    \node[shape=circle,draw=gray] (F) at (2*\dx,-\dy) {$x_{5}$};
    \node[shape=circle,draw=gray] (G) at (3*\dx,0) {$x_{6}$} ;
    \node[shape=circle,draw=gray] (H) at (3*\dx,-\dy) {$x_{7}$} ;

    \path[->,dash pattern=on 4pt off 4pt] (Vx) edge node[left,above] {} (A);
    \path[->,dash pattern=on 4pt off 4pt] (Vy) edge node[left,above] {} (A);
    \path[->,dash pattern=on 4pt off 4pt] (Wx) edge node[left,above] {} (B);
    \path[->,dash pattern=on 4pt off 4pt] (Wy) edge node[left,above] {} (B);

    \path [->](A) edge node[left,above] {2} (C);
    \path [->](B) edge node[left,below] {1} (D);
    \path [->](A) edge node[pos=0.3,left,above] {-1} (D);
    \path [->](B) edge node[pos=0.3,left,below] {-1} (C);

    \path [->](C) edge node[left,above] {$\max(0,x_{2})$} (E);
    \path [->](D) edge node[left,below] {$\max(0,x_{3})$} (F);

    \path [->](E) edge node[left,above] {-2} (G);
    \path [->](F) edge node[left,below] {1} (H);
    \path [->](E) edge node[pos=0.3,left,above] {-1} (H);
    \path [->](F) edge node[pos=0.3,left,below] {0} (G);

    \coordinate (vx0) at (Vx);
    \node at (vx0) [left=0.2*\dx of vx0] { $ -0.5 \leq v_x \leq 0.5 $ };
    \coordinate (vy0) at (Vy);
    \node at (vy0) [left=0.2*\dx of vy0] { $ -0.5 \leq v_y \leq 0.5 $ };

    \coordinate (v0) at ($(vx0)!0.5!(vy0)$);
    \node at (v0) [left=0.2*\dx of v0] {
        $\begin{aligned}
            x_0 \geq 0.5 v_x, \\
            x_0 \leq 0.125 + 0.25 v_x
        \end{aligned}$
    };

    \coordinate (wx0) at (Wx);
    \node at (wx0) [left=0.2*\dx of wx0] { $ -0.5 \leq w_x \leq 0.5 $ };
    \coordinate (wy0) at (Wy);
    \node at (wy0) [left=0.2*\dx of wy0] { $ -0.5 \leq w_y \leq 0.5 $ };

    \coordinate (w0) at ($(wx0)!0.5!(wy0)$);
    \node at (w0) [left=0.2*\dx of w0] {
        $\begin{aligned}
            0.5 + 0.5 w_x \leq x_1, \\
            x_1 \leq 0.5 + 0.5 w_x
        \end{aligned}$
    };

    \coordinate (x0) at (A);
    \node at (x0) [above=0.5*\dy of x0] {
        $\begin{aligned}
            & x_0 \ge 0,\\
            & x_0 \le 0.25,\\
            & l_0 = 0,\\
            & u_0 = 0.25
        \end{aligned}$
    };

    \coordinate (x1) at (B);
    \node at (x1) [below=0.5*\dy of x1] {
        $\begin{aligned}
            & x_1 \ge 0.25,\\
            & x_1 \le 0.75,\\
            & l_1 = 0.25,\\
            & u_1 = 0.75
        \end{aligned}$
    };

    \coordinate (x2) at (C);
    \node at (x2) [above=0.5*\dy of x2] {
        $\begin{aligned}
            & x_2 \ge 2 x_0 - x_1 + 0.25,\\
            & x_2 \le 2 x_0 - x_1 + 0.25,\\
            & l_2 = -0.5,\\
            & u_2 = 0.5
        \end{aligned}$
    };

    \coordinate (x2) at (C);
    \node at (x2) [above=0.2*\dy of x2] {0.25};

    \coordinate (x3) at (D);
    \node at (x3) [below=0.5*\dy of x3] {
        $\begin{aligned}
            & x_3 \ge x_1 - x_0 + 0.125,\\
            & x_3 \le x_1 - x_0 + 0.125,\\
            & l_3 = 0.125,\\
            & u_3 = 0.875
        \end{aligned}$
    };

    \coordinate (x3) at (D);
    \node at (x3) [below=0.2*\dy of x3] {0.125};

    \coordinate (x4) at (E);
    \node at (x4) [above=0.5*\dy of x4] {
        $\begin{aligned}
            & x_4 \ge 0,\\
            & x_4 \le 0.5 x_2 + 0.25,\\
            & l_4 = 0,\\
            & u_4 = 0.5
    \end{aligned}$ };

    \coordinate (x5) at (F);
    \node at (x5) [below=0.5*\dy of x5] {
        $\begin{aligned}
            & x_5 \ge x_3,\\
            & x_5 \le x_3,\\
            & l_5=0.125,\\
            & u_5=0.875
        \end{aligned}$
    };

    \coordinate (x6) at (G);
    \node at (x6) [above=0.5*\dy of x6] {
        $\begin{aligned}
            & x_6 \ge -2 x_4,\\
            & x_6 \le -2 x_4,\\
            & l_6 = -1,\\
            & u_6 = 0
        \end{aligned}$
    };

    \coordinate (x6) at (G);
    \node at (x6) [above=0.2*\dy of x6] {0};

    \coordinate (x7) at (H);
    \node at (x7) [below=0.5*\dy of x7] {
        $\begin{aligned}
            & x_7 \ge x_5 -x_4,\\
            & x_7 \le x_5 -x_4,\\
            & l_7 = -0.375,\\
            & u_7 = 0.875
        \end{aligned}$
    };

    \coordinate (x7) at (H);
    \node at (x7) [below=0.2*\dy of x7] {0};

\end{tikzpicture} }
    \end{center}
    \caption{
        Convex relaxation of a sample neural network with inputs $x_0$ and
        $x_1$, and vector field components $v_x$, $v_y$, $w_x$, and $w_y$.
    }
    \label{fig:toy-network}
\end{figure*}
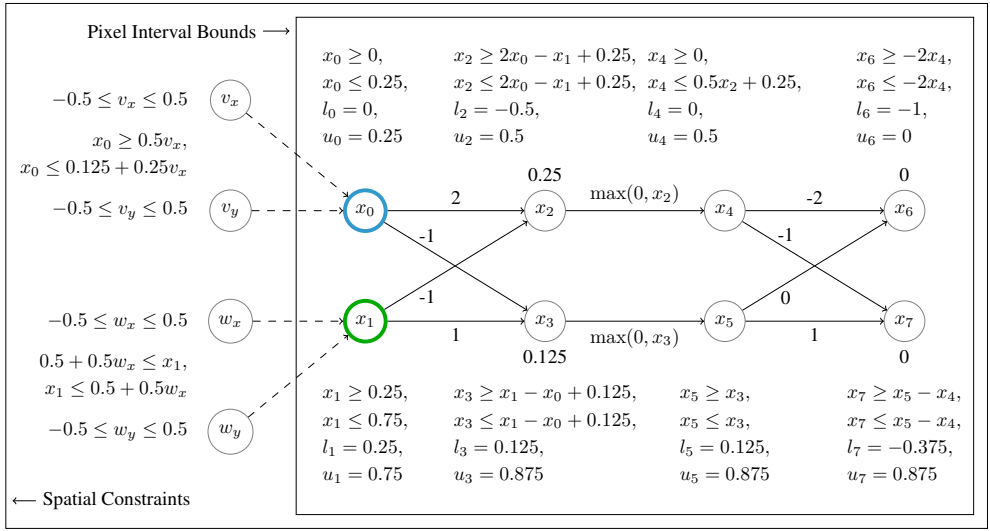

\paragraph{Robustness certification}

Robustness of neural networks is typically certified by (i) computing a convex
shape around the input we want to certify (an over-approximation) and
(ii)~propagating this shape through all operations in the network to obtain a
final output shape.
Robustness is then proven if all concrete outputs inside the output shape
classify to the correct class.
For smaller networks, an input shape can be propagated exactly using
mixed-integer linear programming~\cite{tjeng2019evaluating}.
To scale to larger networks, standard approaches over-approximate the shape
using various convex relaxations: intervals~\cite{gowal2019scalable},
zonotopes~\cite{gehr2018ai2, singh2018fast, weng2018towards}, and restricted
polyhedra~\cite{zhang2018efficient, singh2019abstract}, just to name a few.
In this work, we build on the convex relaxation DeepPoly, an instance of
restricted polyhedra, introduced by \citet{singh2019abstract}, as it provides
a good trade-off between scalability to larger networks and precision.
For every pixel, DeepPoly receives a lower and upper bound on the pixel value,
\ie an input shape in the form of a box.
This shape is then propagated by maintaining one lower and upper linear
constraint for each neuron.
Here, we first show how to construct a tight box around all spatially
transformed images.
However, since this box does not capture the relationship between neighboring
pixels induced by flow-constraints, it contains spurious images that cannot be
produced by smooth vector fields.
To address this, we tighten the convex relaxation for smooth vector fields by
incorporating flow-constraints.

\paragraph{Problem statement}

Our goal is to prove local robustness against vector field attacks constrained
by maximal pixel displacement, \ie all vector fields with $T$-norm smaller than
$\delta$.
That is, for every image from the test set, we try to compute a certificate that
guarantees that no vector field with a displacement magnitude smaller than
$\delta$ can change the predicted label.
Furthermore, since smooth deformations are more realistic, we also consider the
case where the vector fields additionally need to satisfy our flow-constraints,
\ie that neighboring deformation vectors can differ by at most $\gamma$ in
$\ell_\infty$-norm.

    \section{Overview}
\label{sec:overview}

We now provide an end-to-end example of how to compute our convex relaxation of
vector field deformations and use it to certify the robustness of the toy
network in~\cref{fig:toy-network}.
This network propagates inputs $x_0$ and $x_1$ according to the weights
annotated on the edges.
Neurons $x_2$, $x_3$ and $x_4$, $x_5$ denote the pre- and
post-activation values, and $x_6$, $x_7$ are the network logits.
We augment this network by vector field components $v_x$, $v_y$, $w_x$, and
$w_y$, to introduce the flow-constraints.

The concrete inputs to the neural network are the pixels marked with blue
($x_0$) and green ($x_1$), shown in~\cref{fig:sample}.
The image is perturbed by a vector field of $T_\infty$-norm $\delta=0.5$ and
flow $\gamma = 0.25$.
Thus the blue and green pixels are allowed to move in the respective rectangles
shown in~\cref{fig:sample}.
However, to satisfy the flow-constraints, their deformation vectors can differ
by at most $\gamma=0.25$ in each coordinate.

Our objective is to certify that the neural network classifies the input to the
correct label regardless of its deformed pixel positions.
A simple forward pass of the pixel values $x_0 = 0$ and $x_1 = 0.5$ yields
logit values $x_6 = 0$ and $x_7 = 0.625$.
Assuming that the image is correctly classified, we thus need to prove that the
value of neuron $x_7$ is greater than the value of neuron $x_6$ for all
admissible smooth vector field transformations.
To that end, we will first compute interval bounds for the pixels without
using the relationship between vector field components $v_x$, $v_y$, $w_x$, and
$w_y$, and then, in a second step, we tighten the relaxation for smooth vector
fields by introducing linear constraints on $x_0$ and $x_1$ in terms of $v_x$,
$v_y$, $w_x$, and $w_y$ to exploit the flow-constraint relationship.
Propagation of our convex relaxation through the network closely follows
\citet{singh2019abstract}, and we provide a full formalization of our methods
in~\cref{sec:convex_relaxation}.

\paragraph{Calculating interval bounds}

The first part of our convex relaxation is computing upper and lower interval
bounds for the values that the blue and green pixel can attain on their
$\ell_\infty$-neighborhood of radius $\delta = 0.5$.
For both pixels, the minimum and maximum are attained on the left and right
border of the $\ell_\infty$-ball, respectively.
Using bilinear interpolation from \cref{eq:bil_interpolation}, we thus obtain
the interval bounds $[l_0, u_0] = [0, 0.25]$ for $x_0$ and
$[l_1, u_1] = [0.25, 0.75]$ for $x_1$.

\begin{figure*}
    \begin{center}
        \hspace{0.025\linewidth}
        \begin{subfigure}{0.2\linewidth}
            \includegraphics[width=\linewidth]{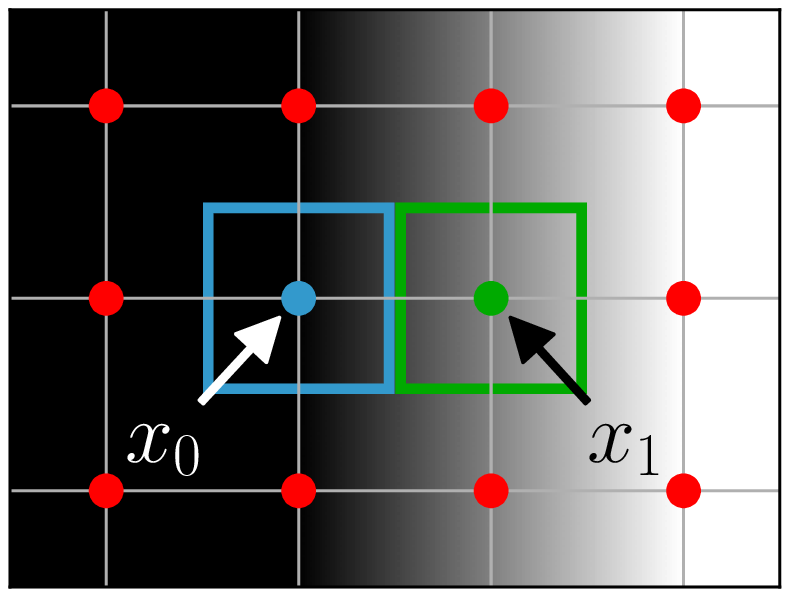}
            \caption{Sample image}
            \label{fig:sample}
        \end{subfigure}
        \hspace{0.01\linewidth}
        \begin{subfigure}{0.75\linewidth}
            \begin{subfigure}{0.45\linewidth}
                \includegraphics[width=\linewidth]{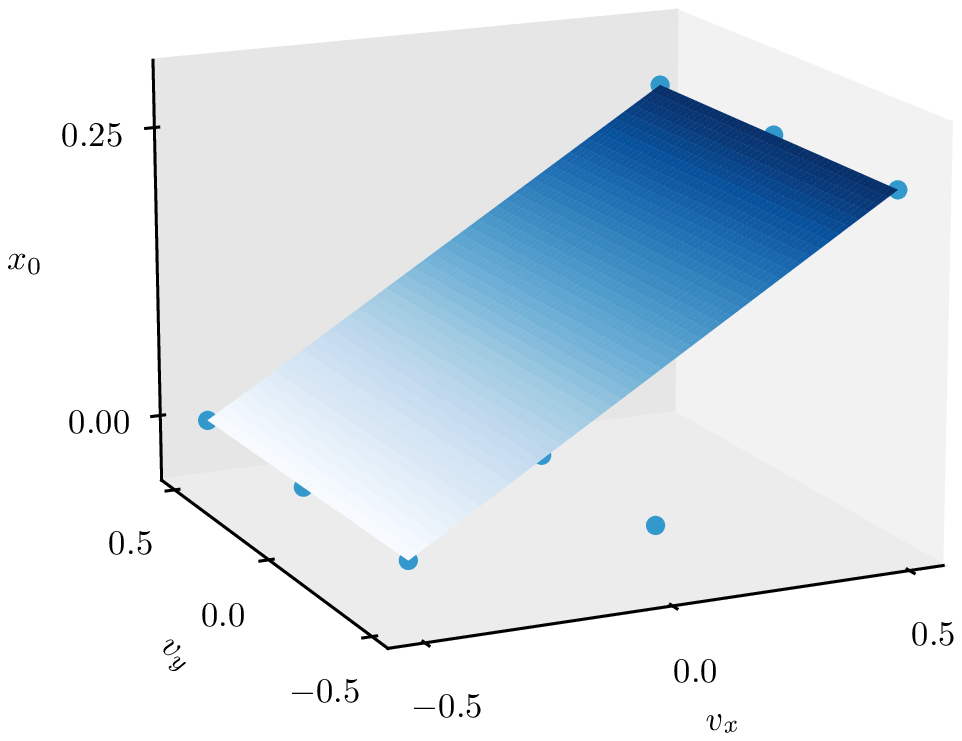}
                \caption{Upper-bounding plane for $x_0$}
                \label{fig:upper_plane}
            \end{subfigure}
            \begin{subfigure}{0.45\linewidth}
                \includegraphics[width=\linewidth]{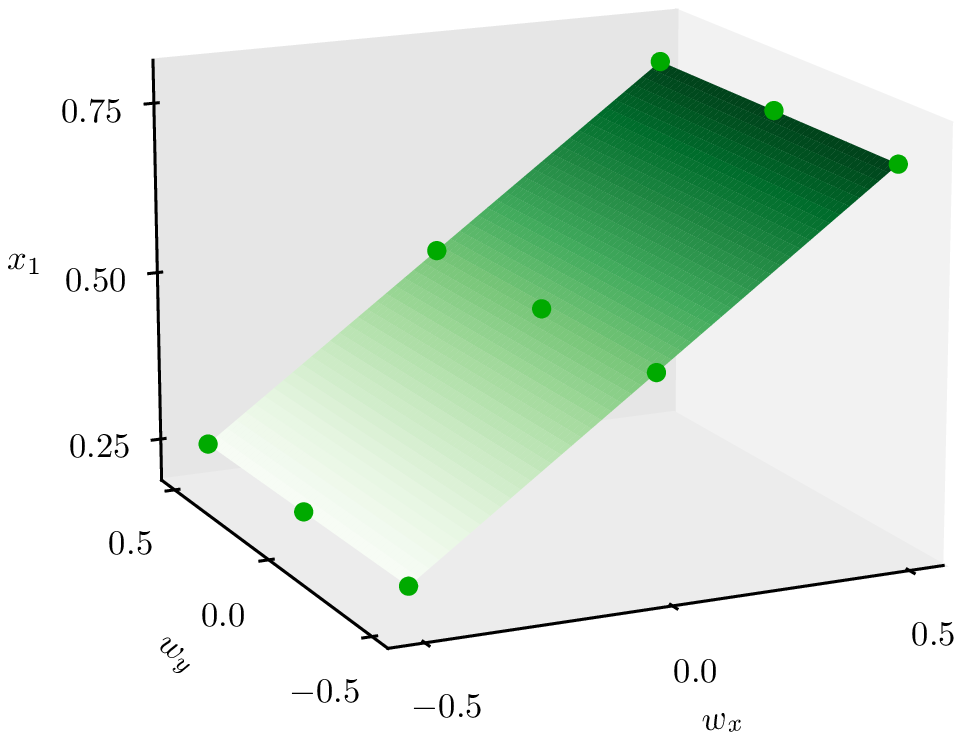}
                \caption{Lower-bounding plane for $x_1$}
                \label{fig:lower_plane}
            \end{subfigure}
        \end{subfigure}
    \end{center}
    \caption{
        Bilinear interpolation for a sample image with dots indicating pixel
        positions and squares denoting the set of reachable coordinates
        $B_{0.5}^\infty$ for pixels $x_0$ and $x_1$.
        \Cref{fig:upper_plane,fig:lower_plane} display the interval bound
        candidates, \ie the potential extrema of bilinear interpolation
        (see~\cref{sec:convex_relaxation}), as dots and the linear bounding
        planes used to enforce flow-constraints.
        Note that the candidates are computed separately for every intersection
        of an interpolation region with the set of reachable coordinates
        $B_{0.5}^\infty$.
    }
\end{figure*}

\paragraph{Interval bound propagation}

The intervals $[l_0, u_0]$ and $[l_1, u_1]$ can be utilized directly for
verification using standard interval propagation to estimate the output of the
network~\cite{gehr2018ai2, gowal2019scalable, mirman2018differentiable}.
While this method is fast, it is also imprecise.
The interval bounds for $x_2$ and $x_4$ are
\begin{align*}
    [l_2, u_2] &= 2 \cdot [l_0, u_0] - [l_1, u_1] + 0.25 \\
    &= [2l_0 - u_1 + 0.25, 2u_0 - l_1 + 0.25] = [-0.5, 0.5], \\
    [l_4, u_4] &= [\max(l_2, 0), \max(u_2, 0)] = [0, 0.5].
\end{align*}
All lower and upper interval bounds are given in \cref{fig:toy-network}.
The output for $x_6$ is thus between $l_6 = -1$ and $u_6 = 0$, while the one
for $x_7$ is between $l_7=-0.375$ and $u_7 = 0.875$.
As this is insufficient to prove $x_7 > x_6$, certification fails.

\paragraph{Backsubstitution}

To gain precision, one can keep track of the relationship between the neurons by
storing linear constraints~\cite{zhang2018efficient,singh2019abstract}.
In addition to $[l_2, u_2]$, we store upper- and lower-bounding linear
constraints
\begin{equation*}
    2 x_0 - x_1 + 0.25 \leq x_2 \leq 2 x_0 - x_1 + 0.25.
\end{equation*}
Similarly, for $x_4$, we store, in addition to $[l_4, u_4]$,
\begin{equation*}
    0 \leq x_4 \leq 0.5 x_2 + 0.25,
\end{equation*}
where we use the rules given in \citet{singh2019abstract} to calculate the
upper and lower linear constraints.
All linear constraints are shown in~\cref{fig:toy-network}, next to the
corresponding neurons.
Certification succeeds, if we can show that $x_7 - x_6 > 0$.
Using the linear constraints, we thus obtain
\begin{align*}
    x_7 - x_6 &\geq (-x_4 + x_5) - (-2x_4) = x_4 + x_5\\
    &\geq x_3 \geq x_1 - x_0 + 0.125 \geq 0.125.
\end{align*}
This proves that $x_7 - x_6 > 0$, implying that the network classifies to the
correct class under all considered deformations.

\paragraph{Spatial constraints}

Although the above method can certify robustness, certification fails for
a more challenging network where the bias of $x_3$ is equal to $-0.125$ instead
of $0.125$.
With the previous approach, we can only prove
$x_7 - x_6 \geq x_1 - x_0 - 0.125 \geq -0.125$, which is insufficient for
certification.
However, we can leverage our vector field smoothness condition
(\cref{eq:flow}), namely that the deformation vectors of $x_0$ and $x_1$ can
differ by at most $\gamma = 0.25$ in the $\ell_\infty$-norm.
Unfortunately, these constraints cannot be directly applied since they are
defined on the vector field components and not on the pixel values.
To amend this, we build on the idea from \citet{balunovic2019geometric,
mohapatra2020towards} and add upper- and lower-bounding linear constraints on
the pixel values $x_0$ and $x_1$.
That is, we compute upper and lower planes in terms of vector field
components $v_x$ and $v_y$ for $x_0$ and $w_x$ and $w_y$ for $x_1$, as shown
in~\cref{fig:upper_plane,fig:lower_plane}.
The plane equations are shown in~\cref{fig:toy-network}, and details on the
computation are provided in~\cref{sec:convex_relaxation}.
By substituting these plane equations into our expression and considering
that all vector field components $v_x$, $v_y$, $w_x$, and $w_y$ are bounded
from above and below by $[-\delta, \delta] = [-0.5, 0.5]$, we thus obtain
\begin{align*}
    x_7 - x_6 &\geq x_1 - x_0 - 0.125 \\
    &\geq (0.5 + 0.5 w_x) - (0.125 + 0.25 v_x) - 0.125 \\
    &= 0.25 + 0.5 w_x - 0.25 v_x \geq -0.125,
\end{align*}
showing that a simple instantiation of the idea by
\citet{balunovic2019geometric, mohapatra2020towards} is insufficient for
certification in our setting.
Only with our flow-constraints $-\gamma \leq v_x - w_x \leq \gamma$ with
$\gamma = 0.25$ can we finally certify:
\begin{align*}
    x_7 - x_6 &\geq x_1 - x_0 - 0.125 \\
    &\geq 0.25 + 0.5 w_x - 0.25 v_x \\
    &= 0.25 + 0.25 w_x + 0.25 (w_x - v_x) \geq 0.0625.
\end{align*}
In practice, the resulting expression may have more than two pixels, and we use
a linear program (described in~\cref{sec:convex_relaxation}) to perform the
substitution of flow-constraints.

    \section{Convex Relaxation}
\label{sec:convex_relaxation}

Here, we introduce a novel convex relaxation tailored to vector field
deformations.
First, we compute the \emph{tightest} interval bounds for each pixel in the
transformed image.
As intervals do not capture dependencies between variables, we then propose a
method for introducing linear constraints on the pixel values in terms of
vector field components and show how to use flow-constraints to further tighten
our convex relaxation.

\subsection{Computing Tight Interval Constraints}
\label{sec:convex_relaxation:tight-intervals}

Consider an image $I$ and a maximum pixel displacement $\delta$.
Our goal is to compute, for pixel $(i, j)$, interval bounds $l_{i, j}$ and
$u_{i, j}$ such that
$l_{i, j} \leq \mathcal{I}_I \circ (\mathbb{I} + \tau)(i, j) \leq u_{i, j}$, for
any vector field $\tau$ of $T_p$-norm at most $\delta$.
We now show how to compute \emph{tight} interval pixel bounds
$[l_{i, j}, u_{i, j}]$.

Within a given interpolation region $A_{mn}$, the pixel $(i, j)$ can move to
positions in $B_\delta^p(i,j) \cap A_{mn}$.
Thus, for every pixel, we construct a set of candidates containing the possible
maxima and minima of that pixel in $B_\delta^p(i,j) \cap A_{mn}$.
However, this could potentially yield an infinite set of candidate points, and
we thus make the key observation that the minimum and maximum pixel values of
$\mathcal{I}^{mn}_I$ in $B_\delta^p(i,j) \cap A_{mn}$ are always obtained at
the boundary (see~\cref{lem:boundary} below, with proof provided
in~\cref{sec:candidate-calculation}).
Hence, for any reachable interpolation region $A_{mn}$, it suffices to consider
the boundary of $B_\delta^p(i,j) \cap A_{mn}$ to derive the candidate points
analytically.
Finally, we set the lower and upper bound of the pixel value to the minimum
and maximum of the candidate set, respectively.

\begin{restatable}{lem}{boundary}
    The minimum and maximum pixel values of $\mathcal{I}^{mn}_I$ in
    $B_\delta^p(i,j) \cap A_{mn}$ are always obtained at the boundary.
    \label{lem:boundary}
\end{restatable}

We note that it is essential to calculate the candidates analytically to
guarantee the soundness of our interval bounds.
Furthermore, for a single-channel image, our interval bounds are \emph{exact},
\ie for every deformed image within our pixel bounds there exists a
corresponding vector field $\tau$ with $||\tau||_p \leq \delta$.
For the derivation of candidates, we make use of the following auxiliary lemma
(with proof in~\cref{sec:candidate-calculation}):

\begin{restatable}{lem}{bilinear}
    The bilinear interpolation $\mathcal{I}^{mn}_I(v, w)$ on the region
    $A_{mn}$ can be rewritten as
    \begin{equation}
        \mathcal{I}^{mn}_I(v,w) = A + B v + C w + D vw.
        \label{eq:bilinear-interpolation}
    \end{equation}
    \label{lem:bilinear}
\end{restatable}

\paragraph{Candidates for $T_\infty$}

The boundary of $B_\delta^{\infty}(i,j) \cap A_{mn}$ consists of line segments
parallel to the coordinate axes.
This means that for all pixels $(v, w)$ on such a line segment either $v$ or
$w$ is constant.
Thus, according to~\cref{lem:bilinear}, if $v$ is fixed, then
$\mathcal{I}^{mn}_I$ is linear in $w$ and vice-versa.
Hence, we only need to consider the line segment endpoints to obtain the
candidate set of all possible extremal values in
$B_\delta^{\infty}(i,j) \cap A_{mn}$.

\paragraph{Candidates for $T_1$}

The boundary of $B_\delta^1(i,j) \cap A_{mn}$ consists of line segments
parallel to the coordinate axes or lines defined by $w = \pm v + a$.
In the first case, we add the pixel values of the line segment endpoints
to the set of candidates.
In the second case, the bilinear interpolation $\mathcal{I}^{mn}_I$ restricted
to $w = \pm v + a$ is a polynomial of degree $2$ given by
\begin{align*}
    \mathcal{I}^{mn}_I(v, \pm v + a)
    &= A + B v + C (\pm v + a) + D v (\pm v + a) \\
    &= (A + C a) + (B \pm C) v \pm D v^2.
\end{align*}
Thus, the extremum can be attained on the interior of that line, unlike the
$T_{\infty}$-case.
Hence, we add both the endpoint values and the polynomial's extremum to the
candidate set if the corresponding extremal point lies on the line segment.

\paragraph{Candidates for $T_2$}

The boundary of $B_\delta^2(i,j) \cap A_{mn}$ consists of line segments
parallel to the coordinate axes or an arc with circle center $(i, j)$.
We handle the line segments by adding the pixel values at the endpoints to the
candidate set.
However, the interpolation can also have minima and maxima on the interior of
the arc.
To find those, we extend the interpolation $\mathcal{I}^{mn}_I$ from $A_{mn}$
to $\mathbb{R}^2$ and use Lagrange multipliers to find the extrema on the
circle $v^2 + w^2 = \delta^2$ (assuming $i = j = 0$ for notational
convenience).
The Lagrangian is
\begin{equation*}
    \Lambda (v, w, \lambda)
    := \mathcal{I}^{mn}_I \left( v, w \right) - \lambda (v^2 + w^2 - \delta^2),
\end{equation*}
which yields
\begin{equation*}
    \nabla_{v, w, \lambda} \Lambda (v, w, \lambda) = \begin{pmatrix}
                                                         B + D w - 2 \lambda v \\
                                                         C + D v - 2 \lambda w \\
                                                         \delta^2-v^2-w^2 \\
    \end{pmatrix} \stackrel{!}{=} 0.
\end{equation*}
Solving the first two equations for $\lambda$, we obtain
\begin{equation*}
    \lambda = \frac{B + D w}{2v} \quad \text{and} \quad
    \lambda = \frac{C + D v}{2w},
\end{equation*}
assuming $v \neq 0 \neq w$ (else $\delta = 0$).
Eliminating $\lambda$, we have
\begin{equation*}
    w (B + D w) - v (C + D v) = 0.
\end{equation*}
We solve this quadratic equation and substitute the solutions
\begin{equation*}
    w = \frac{- B \pm \sqrt{B^2 + 4 D v (C + Dv)}}{2D}.
\end{equation*}
into $\delta^2 - v^2 - w^2 = 0$ to obtain
\begin{equation*}
    \delta^2 = v^2 +
    \left(\frac{- B \pm \sqrt{B^2 + 4 D v (C + Dv)}}{2D} \right)^2.
\end{equation*}
Setting $E := \frac{-B}{2D}$, $F := \frac{B^2}{4D^2}$, and $G := \frac{C}{D}$
we have
\begin{align*}
    \delta^2 &= v^2 + \left( E \pm \sqrt{F + G v + v^2} \right)^2 \\
    &= v^2 + E^2 \pm 2 E \sqrt{F + G v + v^2} + F + G v + v^2 \\
    &= 2 v^2 + G v \pm 2 E \sqrt{F + G v + v^2} + H,
\end{align*}
for $H := (E^2 + F)$.
Solving for $v$ requires squaring both sides to resolve the square root,
yielding
\begin{align*}
    \left(\delta^2 - H - G v - 2 v^2 \right)^2
    &= \left(\pm 2 E \sqrt{F + G v + v^2}\right)^2 \\
    &= 4 E^2 (F + G v + v^2).
\end{align*}
Thus, we are interested in finding the roots of
\begin{equation}
    J + K v + L v^2 + M v^3 + N v^4
    \label{eq:T2Polynomial}
\end{equation}
with
\begin{align*}
    J &:= \left( \delta^2 - H \right)^2 - 4 F E^2, \\
    K &:= - 2 G \left( \left( \delta^2 - H \right) + 2 E^2 \right), \\
    L &:= G^2 - 4 \left( \left( \delta^2 - H \right) + E^2 \right) \\
    M &:= 4 G \\
    N &:= 4,
\end{align*}
which is a polynomial of degree $4$.
The roots of a polynomial of degree $4$ are known in closed
form~\cite{shmakov2011universal}, and we use Durand-Kerner's root finding
method~\cite{kerner1966gesamtschrittverfahren} to compute them analytically.
We recall that computing the roots approximately, \eg via gradient descent or
Newton's method, does not guarantee the soundness of our interval bounds.
If the coordinates obtained from the roots of~\cref{eq:T2Polynomial} lie
within $A_{mn}$, we add the corresponding pixel values to the set of
candidates.
Finally, we add the pixel values of the endpoints of the arc to the set of
candidates.

\subsection{Computing Spatial Constraints}
\label{sec:convex_relaxation:spatial-constraints}

While our interval bounds are tight, they can contain spurious images,
which cannot be produced by smooth vector fields of flow $\gamma$.
To address this, we build upon the idea of \citet{balunovic2019geometric,
mohapatra2020towards} and introduce linear constraints on the pixel values in
terms of the vector field, which can then be paired with flow-constraints to
yield a tighter convex relaxation.
We demonstrate how our method can be applied to tighten the DeepPoly
relaxation~\cite{singh2019abstract} and robustness certification via
mixed-integer linear programming (MILP)~\cite{tjeng2019evaluating}.

\begin{table*}
    \begin{center}
        \begin{small}
            \begin{sc}
                \resizebox{\linewidth}{!}{ \begin{tabular}{@{}llrrrrrr|rrrr@{}}
    \toprule
                            &           & \multicolumn{6}{c}{MNIST}                                                                                         & \multicolumn{4}{c}{CIFAR-10} \\
                            &           & \multicolumn{2}{c}{ConvSmall PGD} & \multicolumn{2}{c}{ConvSmall DiffAI}  & \multicolumn{2}{c}{ConvBig DiffAI}    & \multicolumn{2}{c}{ConvSmall DiffAI}  & \multicolumn{2}{c}{ConvMed DiffAI} \\
    $\delta$                & $\gamma$  & DeepPoly  & MILP                  & DeepPoly  & MILP                      & DeepPoly  & MILP                      & DeepPoly  & MILP                      & DeepPoly  & MILP \\
    \midrule
    \multirow{4}{*}{0.3}    & $\infty$  & 51        & 97                    & 90        & 91                        & 90        & 91                        & 40        & 47                        & 51        & 56 \\
                            & 0.1       & 78        & 98                    & 90        & 92                        & 90        & 94                        & 45        & 56                        & 53        & 60 \\
                            & 0.01      & 91        & 99                    & 91        & 95                        & 91        & 95                        & 50        & 70                        & 57        & 70 \\
                            & 0.001     & 92        & 98                    & 92        & 95                        & 91        & 95                        & 53        & 71                        & 58        & 73 \\
    \midrule
    \multirow{4}{*}{0.4}    & $\infty$  & 6         & 89                    & 74        & 76                        & 78        & 84                        & 31        & 38                        & 36        & 42 \\
                            & 0.1       & 40        & 91                    & 75        & 84                        & 78        & 90                        & 32        & 46                        & 37        & 51 \\
                            & 0.01      & 75        & 98                    & 77        & 92                        & 78        & 90                        & 35        & 67                        & 43        & 57 \\
                            & 0.001     & 77        & 99                    & 77        & 92                        & 78        & 92                        & 37        & 69                        & 43        & 58 \\
    \midrule
    \multirow{4}{*}{0.5}    & $\infty$  & 0         & 76                    & 40        & 50                        & 36        & 62                        & 20        & 32                        & 29        & 34 \\
                            & 0.1       & 7         & 85                    & 44        & 69                        & 37        & 79                        & 23        & 47                        & 30        & 42 \\
                            & 0.01      & 32        & 88                    & 44        & 91                        & 37        & 89                        & 27        & 53                        & 31        & 47 \\
                            & 0.001     & 35        & 89                    & 44        & 92                        & 37        & 90                        & 27        & 53                        & 33        & 48 \\
    \bottomrule
\end{tabular}
 }
            \end{sc}
        \end{small}
    \end{center}
    \caption{
        $T_\infty$-norm certification rates (\%) for vector fields $\tau$ with
        displacement magnitude $\| \tau \|_{T_\infty} = \delta$ and flow
        $\gamma$.
    }
    \label{tab:certification}
\end{table*}

\paragraph{Upper- and lower-bounding planes}

We seek to compute sound linear constraints on the spatially transformed pixel
value in terms of the deforming vector field $\tau$.
Since every pixel is displaced independently by its corresponding vector, the
linear constraints induce bounding planes of the form
\begin{equation*}
    \lambda_0 + \lambda \cdot \tau(i, j)
    \leq \mathcal{I}_{I}((i, j) + \tau(i, j))
    \leq \upsilon_0 + \upsilon \cdot \tau(i, j),
\end{equation*}
where $\lambda^T = (\lambda_1, \lambda_2)$ and
$\upsilon^T = (\upsilon_1, \upsilon_2)$.
To compute a sound lower-bounding plane, we apply our method
from~\cref{sec:convex_relaxation:tight-intervals} to compute the set of
candidate coordinates $\mathcal{C}$ of potential minima and maxima in
$B_\delta^\infty(i, j)$ and then solve
\begin{gather*}
    \argmin_{\lambda_0, \lambda_1, \lambda_2} \sum_{(p, q) \in \mathcal{C}}
    \mathcal{I}_{I}(p, q) - (\lambda_0 + \lambda_1 (p - i) + \lambda_2 (q - j)) \\
    \mathcal{I}_{I}(p, q) \geq (\lambda_0 + \lambda_1 (p - i) + \lambda_2 (q - j)), \quad \forall (p,q) \in \mathcal{C}.
\end{gather*}
Since both the objective and the constraints are linear, we can compute this
plane in polynomial time using linear programming.
The upper-bounding plane is obtained analogously.

Given these linear bounding planes, one could be tempted to simply instantiate
the framework from \citet{balunovic2019geometric, mohapatra2020towards}.
Unfortunately, their approach only works in the setting where multiple pixels
are transformed by the same spatial parameters (\eg rotation angle).
However, we make the key observation that these linear constraints can be
leveraged to enforce the flow-constraints, thus tightening our convex
relaxation.
We now describe how this can be achieved for the DeepPoly relaxation and MILP.

\paragraph{Tightening DeepPoly relaxation}

To compute precise bounds, DeepPoly performs backsubstitution for each neuron
in the network (recall the example in~\cref{sec:overview}).
That is, every backsubstitution step computes a linear expression
$e = a_1x_1 + \ldots + a_nx_n$ in terms of the input pixels.
A naive way to obtain the upper and lower bounds of $e$ is to substitute the
interval bounds for each pixel $x_i$, which is equivalent to
\begin{gather*}
    \min a_1x_1 + \ldots + a_nx_n \\
    l_i \leq x_i \leq u_i.
\end{gather*}
However, this can be imprecise as intervals do not capture flow-constraints.
Thus, we extend the above linear program with variables $v_x^{(i)}$ and
$v_y^{(i)}$ denoting the vector field components of every pixel
$x_i$, thus allowing us to add the constraints
\begin{gather}
    \lambda_0^{(i)} + (v_x, v_y)^{(i)} \lambda^{(i)} \leq x_i \leq
    \upsilon_0^{(i)} + (v_x, v_y)^{(i)} \upsilon^{(i)},
    \label{eq:plane-constraints} \\
    -\gamma \leq v_x^{(i)} - v_x^{(j)} \leq \gamma, \text{ and }
    -\gamma \leq v_y^{(i)} - v_y^{(j)} \leq \gamma, \label{eq:flow-constraints}
\end{gather}
where $\lambda^T = (\lambda_1^{(i)}, \lambda_2^{(i)})$ and
$\upsilon^T = (\upsilon_1^{(i)}, \upsilon_2^{(i)})$.
Here, \cref{eq:plane-constraints} corresponds to the upper- and lower-bounding
planes of pixel $x_i$, and~\cref{eq:flow-constraints} enforces the
flow-constraints for neighboring pixels $i$ and $j$.
Minimization of this linear program then directly yields the tightest lower
bound on the expression and can be performed in polynomial time.
The upper-bounding plane can be obtained analogously.

\paragraph{Tightening MILP certification}

To encode a neural network as MILP, we employ the method
from~\citet{tjeng2019evaluating}, which is exact for models with ReLU
activations.
Our approach of leveraging linear planes on pixel values to enforce
flow-constraints can then be directly applied to the resulting MILP by adding
the same variables and linear constraints
(\cref{eq:plane-constraints,eq:flow-constraints}) as in the DeepPoly case.

    \section{Experiments}
\label{sec:experiments}

We now investigate the precision and scalability of our certification method by
evaluating it on a rich combination of datasets and network architectures.
We make all our networks and code publicly available as part of the ERAN
framework for neural network verification (available at
\url{https://github.com/eth-sri/eran}) to ensure reproducibility.

\paragraph{Experiment setup}

We select a random subset of 100 images from the MNIST~\cite{lecun2010mnist}
and CIFAR-10~\cite{krizhevsky2009learning} test datasets on which we run all
experiments.
We consider adversarially trained variants of the \textsc{ConvSmall},
\textsc{ConvMed}, and \textsc{ConvBig} architectures proposed by
\citet{mirman2018differentiable}, using PGD~\cite{madry2018towards} and
DiffAI~\cite{mirman2018differentiable} for adversarial training.
For CIFAR-10, we also consider a ResNet~\cite{he2016deep}, with 4 residual
blocks of 16, 16, 32, and 64 filters each, trained with the provable defense
from \citet{wong2018scaling}.
We present the model accuracies and training hyperparameters
in~\cref{sec:models}.
While we only consider networks with ReLU activations for our experiments, our
relaxations seamlessly integrate with different verifiers, including DeepPoly,
k-ReLU, and MILP, which allows us to certify all networks employing the
activation functions supported by these frameworks.
For example, DeepPoly handles ReLU, sigmoid, tanh, quadratic, and logarithmic
activations, while MILP can only exactly encode piecewise linear activation
functions such as ReLU or LeakyReLU.
We use a desktop PC with a single GeForce RTX 2080 Ti GPU and a 16-core Intel(R)
Core(TM) i9-9900K CPU @ 3.60GHz, and we report all certification running times
in~\cref{sec:running-times}.

\paragraph{Robustness certification}

We demonstrate the precision of our convex relaxations via robustness
certification against vector field transformations.
To that end, we run DeepPoly and MILP with our interval and spatial constraints
and compute the percentage of certified MNIST and CIFAR-10 images for different
networks and values of $\delta$ and $\gamma$.
Note that $\gamma = \infty$ corresponds
to~\cref{sec:convex_relaxation:tight-intervals}, and $\gamma < \infty$
corresponds to~\cref{sec:convex_relaxation:spatial-constraints}.
We limit MILP to 5 minutes and display the results in~\cref{tab:certification}
(showing only $T_\infty$-norm results for brevity).
We observe that our interval bounds successfully enable certification of vector
field attacks and that our tightened convex relaxation for smooth vector fields
offers (at times substantial) improvements across all datasets, verifiers, and
networks.
For example, for \textsc{ConvSmall PGD} on MNIST, DeepPoly certification
increases from $6\%$ ($\gamma = \infty$) to $77\%$ ($\gamma = 0.001$) for
$\delta = 0.4$.
Similarly, for \textsc{ConvSmall DiffAI} on CIFAR-10, MILP certification
increases from $38\%$ ($\gamma = \infty$) to $69\%$ ($\gamma = 0.001$) for
$\delta = 0.4$.
In fact, our convex relaxation can also be applied with the k-ReLU
verifier~\cite{singh2019beyond} where it increases certification from $24\%$
($\gamma = \infty$) to $51\%$ ($\gamma = 0.1$) for $\delta = 0.4$ for
\textsc{ConvSmall PGD} on MNIST.
Note that while our tightened convex relaxation increases certification
rates, it does so only for the more restricted setting of sufficiently smooth
deformations.

We also compare the certification rates for the different $T_p$-norms on a
\textsc{ConvSmall} network trained to be provably robust with DiffAI on MNIST.
For brevity, we only consider the case where $\gamma = \infty$, and we display
the percentage of certified images in~\cref{tab:t-norm-comparison} (with MILP
timeout of 5 minutes).

\paragraph{Scaling to larger networks}

We evaluate the scalability of our convex relaxation by certifying spatial
robustness of a large CIFAR-10 ResNet with 108k neurons trained with the
provable defense from \citet{wong2018scaling}.
To account for the large network size, we increase the MILP timeout to 10
minutes.
For $\delta = 0.4$, DeepPoly increases certification from $72\%$ ($\gamma =
\infty$) to $76\%$ ($\gamma = 0.1$), whereas MILP increases certification from
$87\%$ ($\gamma = \infty$) to $89\%$ ($\gamma = 0.1$).
The average running times per image are 77 seconds ($\gamma = \infty$) and 394
seconds ($\gamma = 0.1$) for DeepPoly, and 446 seconds ($\gamma = \infty$) and
1168 seconds ($\gamma = 0.1$) for MILP.

\begin{table}
    \begin{center}
        \begin{small}
            \begin{sc}
                \resizebox{\columnwidth}{!}{ \begin{tabular}{@{}lrrrrrr@{}}
    \toprule
                & \multicolumn{2}{c}{$T_1$-norm}    & \multicolumn{2}{c}{$T_2$-norm}    & \multicolumn{2}{c}{$T_\infty$-norm} \\
    $\delta$    & DeepPoly  & MILP                  & DeepPoly  & MILP                  & DeepPoly  & MILP \\
    \midrule
    0.3         & 96        & 97                    & 95        & 95                    & 90        & 91 \\
    0.5         & 75        & 79                    & 70        & 74                    & 40        & 50 \\
    0.7         & 23        & 43                    & 13        & 33                    & 2         & 15 \\
    0.9         & 4         & 21                    & 2         & 13                    & 1         & 5 \\
    \bottomrule
\end{tabular}

 }
            \end{sc}
        \end{small}
    \end{center}
    \caption{
        Certification rates (\%) of \textsc{ConvSmall DiffAI} on MNIST for
        different $T_p$-norms with $\gamma = \infty$.
    }
    \label{tab:t-norm-comparison}
\end{table}

We do not evaluate our convex relaxation on robustness certification of
ImageNet~\cite{deng2009imagenet} models, since normally trained ImageNet
networks are not provably robust and adversarially trained networks have very
low standard and certifiable accuracy~\cite{gowal2019scalable}.
However, our method effortlessly scales to large images since computing our
interval bounds requires at most 0.02 seconds per image for MNIST, CIFAR-10, and
ImageNet, and the average running time for computing our linear bounds is 1.63
seconds for MNIST, 5.03 seconds for CIFAR-10, and 236 seconds for ImageNet (this
could be optimized via parallelization).
Thus, any improvement in robust training on ImageNet would immediately allow us
to certify the robustness of the resulting models against vector field
deformations.
For completeness, we note that randomized smoothing approaches do scale to
ImageNet, but only provide probabilistic guarantees for a smoothed neural
network, as we discussed in~\cref{sec:related-work}.

\paragraph{Approximation error for multi-channel images}

We recall that our interval bounds are only exact for single-channel images.
That is, compared to the single-channel case where the shape of all deformed
images is a box, which we can compute exactly, the output shape for the
multi-channel case is a high-dimensional object, which is not even polyhedral.
Thus, we over-approximate the output shape with the tightest-possible box.
To calculate the approximation error, we would have to compare its volume with
the volume of the exact output shape, but computing the exact volume is
infeasible.
Consequently, we can only estimate the precision by comparing the volume of our
box with the volume of the intervals obtained from sampling vector fields and
computing the corresponding deformed images.
For example, sampling 10’000 vector fields for $\delta = 1$ yields intervals
covering 99.10\% (MNIST), 98.84\% (CIFAR-10), and 98.76\% (ImageNet) of our
bounds.
Based on these results, we conclude that our bounds are reasonably tight, even
for multi-channel images.

    \section{Conclusion}

We introduced a novel convex relaxation for images obtained from vector field
deformations and showed that this relaxation enables, for the first time,
robustness certification against such spatial transformations.
Furthermore, we tightened our convex relaxations for smooth vector fields by
introducing smoothness constraints that can be efficiently incorporated into
state-of-the-art neural network verifiers.
Our evaluation across different datasets and architectures demonstrated the
practical effectiveness of our methods.

    \message{^^JLASTBODYPAGE \thepage^^J}
    \section*{Ethics Statement}

It is well known that neural networks can be successfully employed in settings
with positive (\eg personalized healthcare) and negative (\eg autonomous
weapons systems) social impacts.
Moreover, even for settings with potentially beneficial impacts, such
as personalized healthcare, the concrete implementation of these models remains
challenging, \eg concerning privacy.
In that regard, robustness certification is more removed from the practical
application, as it provides guarantees for a class of models (\eg image
classification networks), irrespective of the particular task at hand.
For example, our method could be applied to certify spatial robustness for
self-driving cars in the same way that it could be employed to prove robustness
for weaponized drones.
Since the vector field deformations considered in our work present a natural way
of describing distortions arising from the fact that cameras map a 3D world to
2D images, our certification method can be used to comply with regulations or
quality assurance criteria for all applications that require robustness against
these types of transformations.

    \section*{Acknowledgments}

We are grateful to Christoph Müller for his help in enabling backsubstitution
with our spatial constraints.
We also thank Christoph Amevor for proofreading an earlier version of this
paper and for his helpful comments.
Finally, we thank the anonymous reviewers for their insightful feedback.

    \bibliography{references}

    \message{^^JLASTREFERENCESPAGE \thepage^^J}


    \ifbool{includeappendix}{%
    \clearpage
    \appendix
    \section{Calculation of Candidates}
\label{sec:candidate-calculation}

Here, we prove~\cref{lem:boundary,lem:bilinear}, which we employed to
efficiently compute the interval bounds $l_{i,j}$ and $u_{i,j}$ such that
$l_{i, j} \leq \mathcal{I}_I \circ (\mathbb{I} + \tau)(i, j) \leq u_{i, j}$, for
a pixel $(i, j)$ of an image $I$ and any vector field $\tau$ of $T_p$-norm of at
most $\delta$.
To ease notation, we view images as a collection of pixels on a regular grid.
In this case, bilinear interpolation is given by
\begin{align*}
    \mathcal{I}_I \left( i, j \right) &:=
    \begin{cases}
        \mathcal{I}_I^{mn} \left( i, j \right) & \text{if} \left( i, j \right) \in A_{mn}
    \end{cases}, \\
    \mathcal{I}^{mn}_I \left( i, j \right) &:=
    \sum_{\substack{
        p \in \left\{ m, m + 1 \right\} \\ q \in \left\{ n, n + 1 \right\}
    }}
    I(p, q)
    \left( 1 - \lvert p - i \rvert \right)
    \left( 1 - \lvert q - j \rvert \right),
\end{align*}
where $A_{mn} := \left[ m, m + 1 \right] \times \left[ n, n + 1 \right]$ is an
interpolation region.

\bilinear*

\begin{proof}
    \begin{align*}
        \mathcal{I}^{mn}_I \left( v, w \right) &=
        \sum_{\substack{
        p \in \left\{ m, m + 1 \right\} \\ q \in \left\{ n, n + 1 \right\}
        }}
        I(p, q) \left( 1 - \lvert p - v \rvert \right) \left( 1 - \lvert q - w \rvert \right) \\
        &= I(m, n) (1 + m - v) (1 + n - w) \\
        &\quad + I(m + 1, n) (v - m) (1 + n - w) \\
        &\quad + I(m, n + 1) (1 + m - v) (w - n) \\
        &\quad + I(m + 1, n + 1) (v - m) (w - n) \\
        &= A + B v + C w + D vw,
    \end{align*}
    where the last equality follows from expanding the parentheses and
    grouping the terms.
\end{proof}

\boundary*

\begin{proof}
    Let $(v, w) \in B_\delta^p(i,j) \cap A_{mn}$ be an interior point such
    that $\mathcal{I}^{mn}_I$ is extremal at $(v, w)$.
    Then, assuming $v$ is fixed, it can be seen
    from~\cref{eq:bilinear-interpolation} that $\mathcal{I}^{mn}_I$ is
    linear in $w$.
    Hence, the pixel value is non-decreasing in one direction and
    non-increasing in the opposite direction of $w$, implying that both the
    minimum and the maximum are obtained at the boundary.
\end{proof}

Hence, for any reachable interpolation region $A_{mn}$ it suffices to consider
the boundary of $B_\delta^p(i,j) \cap A_{mn}$ to construct the set of
candidates containing the possible minima and maxima of the pixel in
$B_\delta^p(i,j) \cap A_{mn}$.
The bounds $l_{i,j}$ and $u_{i,j}$ are then the minimum and maximum of the set
of candidates, respectively.

\section{Models}
\label{sec:models}

Here, we describe the model accuracies and training hyperparameters.
We recall that, for MNIST and CIFAR-10, we consider defended variants of the
\textsc{ConvSmall}, \textsc{ConvMed}, and \textsc{ConvBig} architectures
proposed by \citet{mirman2018differentiable}, using PGD~\cite{madry2018towards}
and DiffAI~\cite{mirman2018differentiable} for adversarial training.
Moreover, for CIFAR-10, we also evaluate our certification method on a
ResNet~\cite{he2016deep}, with 4 residual blocks of 16, 16, 32, and 64 filters
each (108k neurons), trained with the provable defense from
\citet{wong2018scaling}.

To differentiate between the training techniques, we append suffixes to the
model names: \textsc{DiffAI} for DiffAI and \textsc{PGD} for PGD.
All DiffAI networks were pre-trained against $\ell_\infty$-noise perturbations
by \citet{mirman2018differentiable} with $\epsilon = 0.3$ on MNIST and
$\epsilon = 8/255$ on CIFAR-10.
For \textsc{ConvSmall PGD} we used 20 epochs of PGD training against vector
field attacks with $\delta = 0.5$, using cyclic learning rate scheduling from
\num{1e-8} to \num{1e-2}, and $\ell_1$-norm weight decay with trade-off
parameter \num{1e-3}.
We provide the model accuracies on the respective test sets
in~\cref{tab:models}.

\begin{table}
    \begin{center}
        \begin{small}
            \begin{sc}
                \resizebox{\columnwidth}{!}{ \begin{tabular}{@{}lllr@{}}
    \toprule
    Dataset                     & Model             & Training                  & Accuracy \\
    \midrule
    \multirow{2}{*}{MNIST}      & ConvSmall PGD     & Spatial PGD               & 97.53\%  \\
                                & ConvSmall DiffAI  & DiffAI                    & 94.52\%  \\
                                & ConvBig DiffAI    & DiffAI                    & 97.03\%  \\
    \midrule
    \multirow{3}{*}{CIFAR-10}   & ConvSmall DiffAI  & DiffAI                    & 42.60\%  \\
                                & ConvMed DiffAI    & DiffAI                    & 43.57\%  \\
                                & ResNet            & \citet{wong2018scaling}   & 27.70\%  \\
    \bottomrule
\end{tabular}
 }
            \end{sc}
        \end{small}
    \end{center}
    \caption{Model test accuracies and training parameters.}
    \label{tab:models}
\end{table}

\section{Running Times}
\label{sec:running-times}

Here, we provide the running times averaged over our random subsets of 100 test
images.
As mentioned in~\cref{sec:experiments}, we run all experiments on a desktop PC
with a single GeForce RTX 2080 Ti GPU and a 16-core Intel(R) Core(TM) i9-9900K
CPU @ 3.60GHz.

In~\cref{tab:times-certification} we display the average certification times
corresponding to the experiment from~\cref{tab:certification}
in~\cref{sec:experiments}.
Likewise, we display the running times corresponding to the $T_p$-norm
comparison experiment from~\cref{tab:t-norm-comparison}
in~\cref{tab:times-t-norm-comparison}.
We recall that every instance of the MILP is limited to 5 minutes.

The average running times for the k-ReLU verifier~\cite{singh2019beyond}
using our convex relaxation with $k = 15$, an LP timeout of 5 seconds, and a
MILP timeout of 10 seconds are: 176.4s ($\gamma = \infty$) and 285.3s
($\gamma = 0.1$) for $\delta = 0.4$.

\begin{table*}
    \begin{center}
        \begin{small}
            \begin{sc}
                \resizebox{\linewidth}{!}{ \begin{tabular}{@{}llrrrrrr|rrrr@{}}
    \toprule
                            &           & \multicolumn{6}{c}{MNIST}                                                                                         & \multicolumn{4}{c}{CIFAR-10} \\
                            &           & \multicolumn{2}{c}{ConvSmall PGD} & \multicolumn{2}{c}{ConvSmall DiffAI}  & \multicolumn{2}{c}{ConvBig DiffAI}    & \multicolumn{2}{c}{ConvSmall DiffAI}  & \multicolumn{2}{c}{ConvMed DiffAI} \\
    $\delta$                & $\gamma$  & DeepPoly  & MILP                  & DeepPoly  & MILP                      & DeepPoly  & MILP                      & DeepPoly  & MILP                      & DeepPoly  & MILP \\
    \midrule
    \multirow{4}{*}{0.3}    & $\infty$  & 1.0       & 28.0                  & 0.1       & 1.7                       & 18.5      & 269.6                     & 0.6       & 45.8                      & 0.4       & 89.9 \\
                            & 0.1       & 8.1       & 54.8                  & 7.2       & 11.2                      & 65.4      & 240.1                     & 109.8     & 353.3                     & 83.3      & 338.9 \\
                            & 0.01      & 10.7      & 59.6                  & 7.1       & 12.7                      & 68.6      & 288.9                     & 160.5     & 372.0                     & 116.8     & 351.1 \\
                            & 0.001     & 12.2      & 97.0                  & 7.2       & 33.4                      & 86.4      & 327.9                     & 201.0     & 415.0                     & 173.2     & 469.5 \\
    \midrule
    \multirow{4}{*}{0.4}    & $\infty$  & 0.8       & 48.6                  & 0.1       & 1.7                       & 19.8      & 275.5                     & 0.2       & 59.1                      & 0.6       & 153.4 \\
                            & 0.1       & 9.4       & 71.3                  & 3.1       & 15.9                      & 65.1      & 261.6                     & 94.7      & 363.3                     & 86.2      & 333.9 \\
                            & 0.01      & 11.9      & 61.6                  & 3.1       & 8.4                       & 77.1      & 319.8                     & 104.9     & 268.8                     & 107.0     & 368.3 \\
                            & 0.001     & 12.0      & 43.0                  & 3.1       & 11.0                      & 73.3      & 314.2                     & 92.6      & 293.7                     & 104.7     & 379.5 \\
    \midrule
    \multirow{4}{*}{0.5}    & $\infty$  & 0.8       & 97.3                  & 0.1       & 1.9                       & 25.7      & 252.6                     & 0.5       & 128.8                     & 0.7       & 214.6 \\
                            & 0.1       & 9.8       & 76.6                  & 7.2       & 34.6                      & 62.3      & 225.5                     & 103.0     & 320.7                     & 93.8      & 353.3 \\
                            & 0.01      & 11.2      & 87.0                  & 7.3       & 19.5                      & 72.4      & 231.6                     & 159.9     & 404.1                     & 109.4     & 373.7 \\
                            & 0.001     & 85.6      & 12.1                  & 7.1       & 28.6                      & 65.9      & 288.1                     & 139.8     & 370.1                     & 125.3     & 410.4 \\
    \bottomrule
\end{tabular}
 }
            \end{sc}
        \end{small}
    \end{center}
    \caption{
        Average running times (seconds) for certification of vector fields
        $\tau$ with displacement magnitude $\| \tau \|_{T_\infty} = \delta$ and
        flow $\gamma$.
    }
    \label{tab:times-certification}
\end{table*}

\begin{figure*}
    \begin{center}

        \begin{subfigure}{\dimexpr0.9\linewidth+20pt\relax}
            \begin{center}
                \makebox[20pt]{\raisebox{20pt}{\rotatebox[origin=c]{90}{}}}
                \begin{subfigure}[t]{.1\linewidth}
                    \caption*{\textbf{original}}
                \end{subfigure}
                \begin{subfigure}[t]{.2\linewidth}
                    \caption*{$\bm{\gamma = \infty}$}
                \end{subfigure}
                \begin{subfigure}[t]{.2\linewidth}
                    \caption*{$\bm{\gamma = 0.1}$}
                \end{subfigure}
                \begin{subfigure}[t]{.2\linewidth}
                    \caption*{$\bm{\gamma = 0.01}$}
                \end{subfigure}
                \begin{subfigure}[t]{.2\linewidth}
                    \caption*{$\bm{\gamma = 0.001}$}
                \end{subfigure}
            \end{center}
        \end{subfigure}
        \vspace{0.25cm}
        \begin{subfigure}{\dimexpr0.9\linewidth+20pt\relax}
            \begin{center}
                \makebox[20pt]{\raisebox{20pt}{\rotatebox[origin=c]{90}{$\bm{\delta = 0.3}$}}}
                \begin{subfigure}[t]{.1\linewidth}
                    \includegraphics[width=0.98\textwidth]{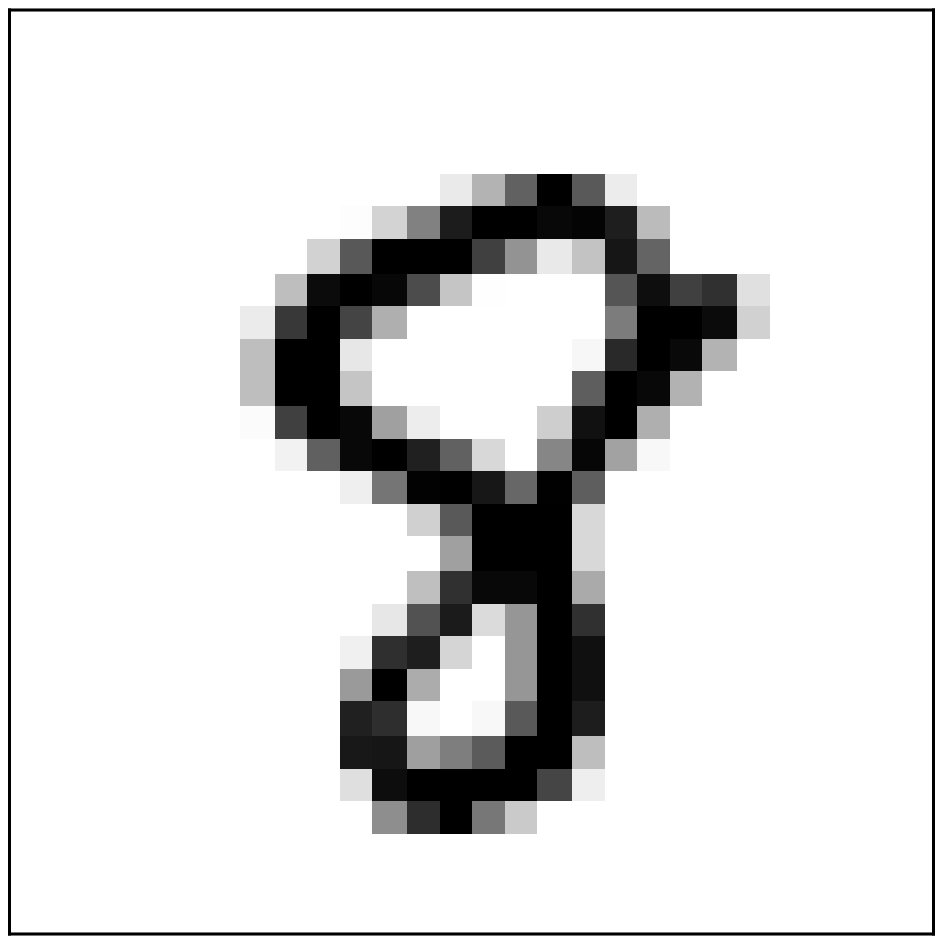}
                \end{subfigure}
                \begin{subfigure}[t]{.2\linewidth}
                    \includegraphics[width=0.49\textwidth]{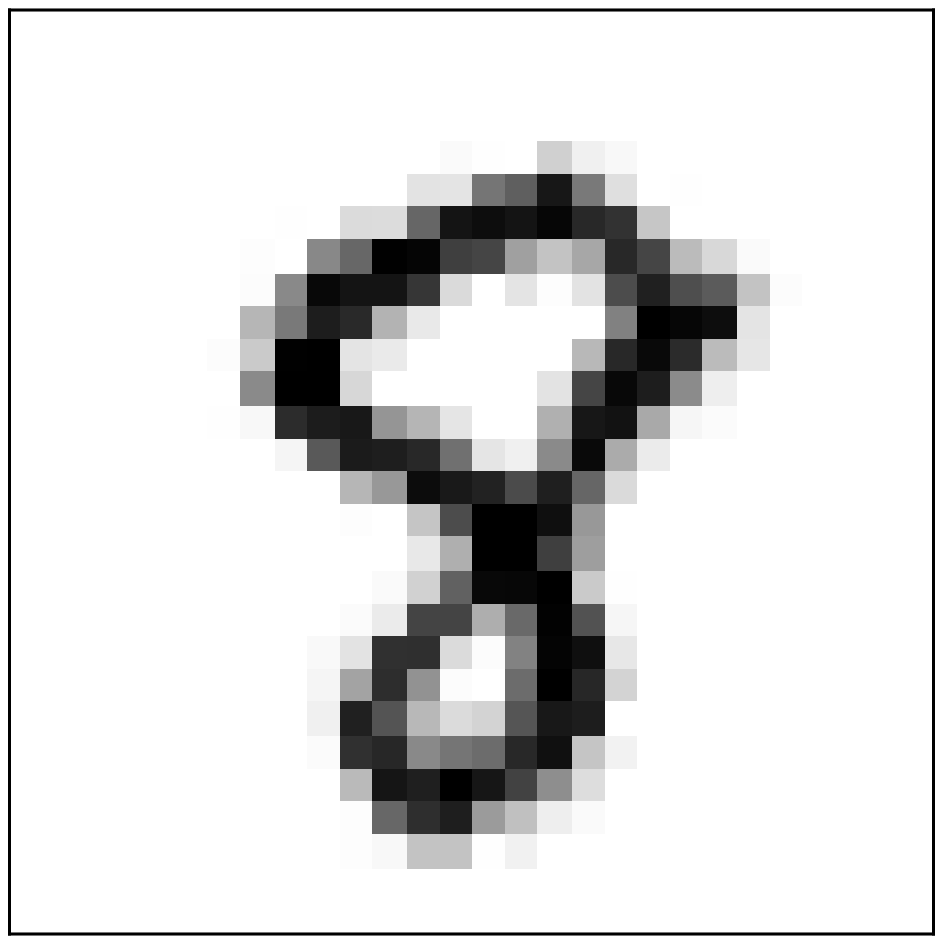}
                    \hspace{-0.2cm}
                    \includegraphics[width=0.49\textwidth]{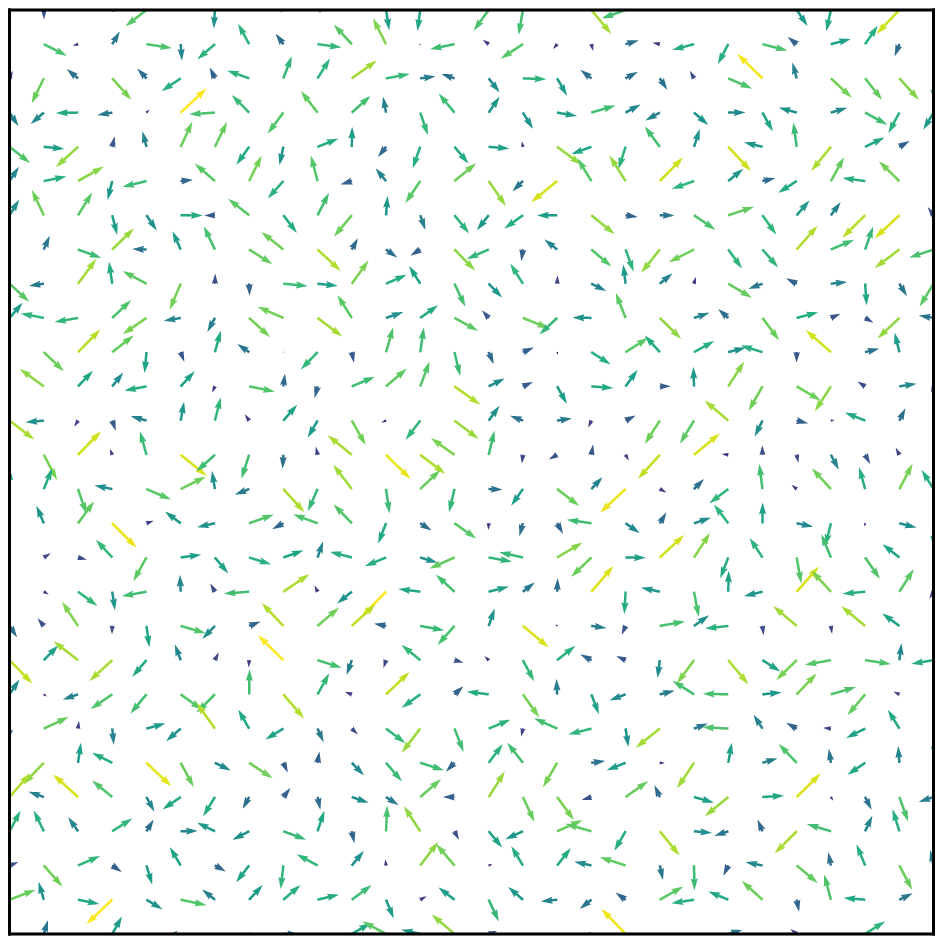}
                    \caption{label 8}
                    \label{fig:visual-comparison:certified-but-not-attacked}
                \end{subfigure}
                \begin{subfigure}[t]{.2\linewidth}
                    \includegraphics[width=0.49\textwidth]{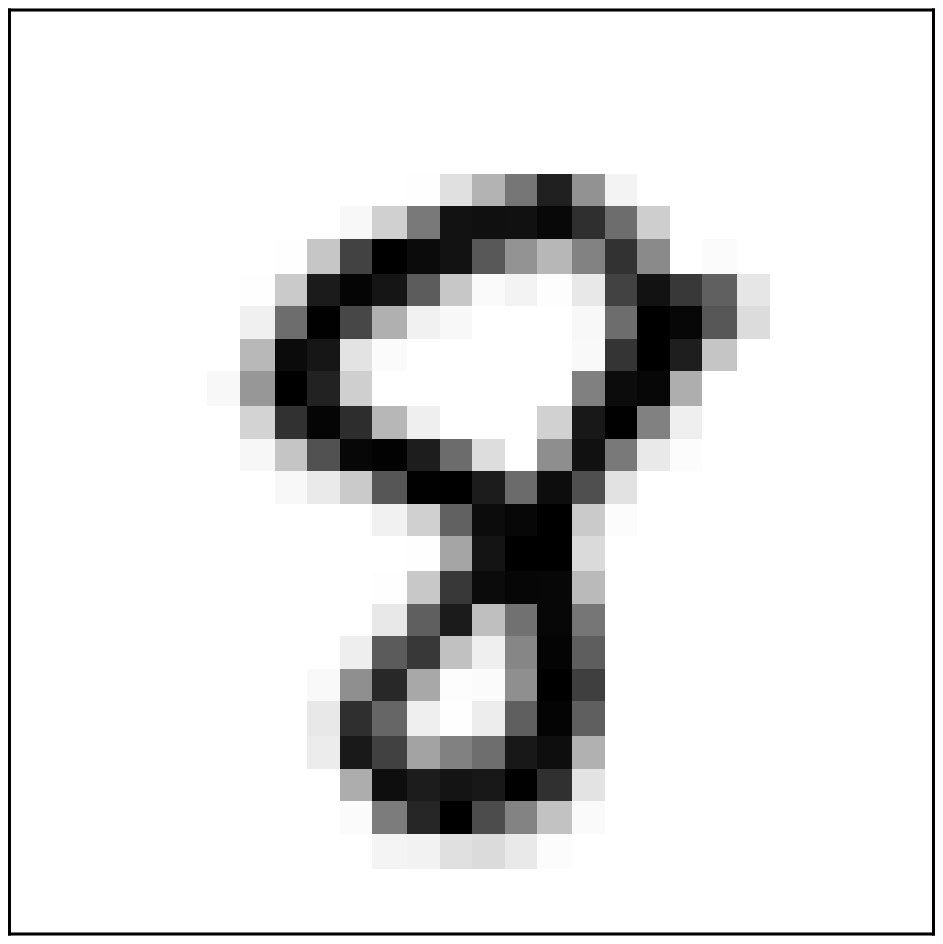}
                    \hspace{-0.2cm}
                    \includegraphics[width=0.49\textwidth]{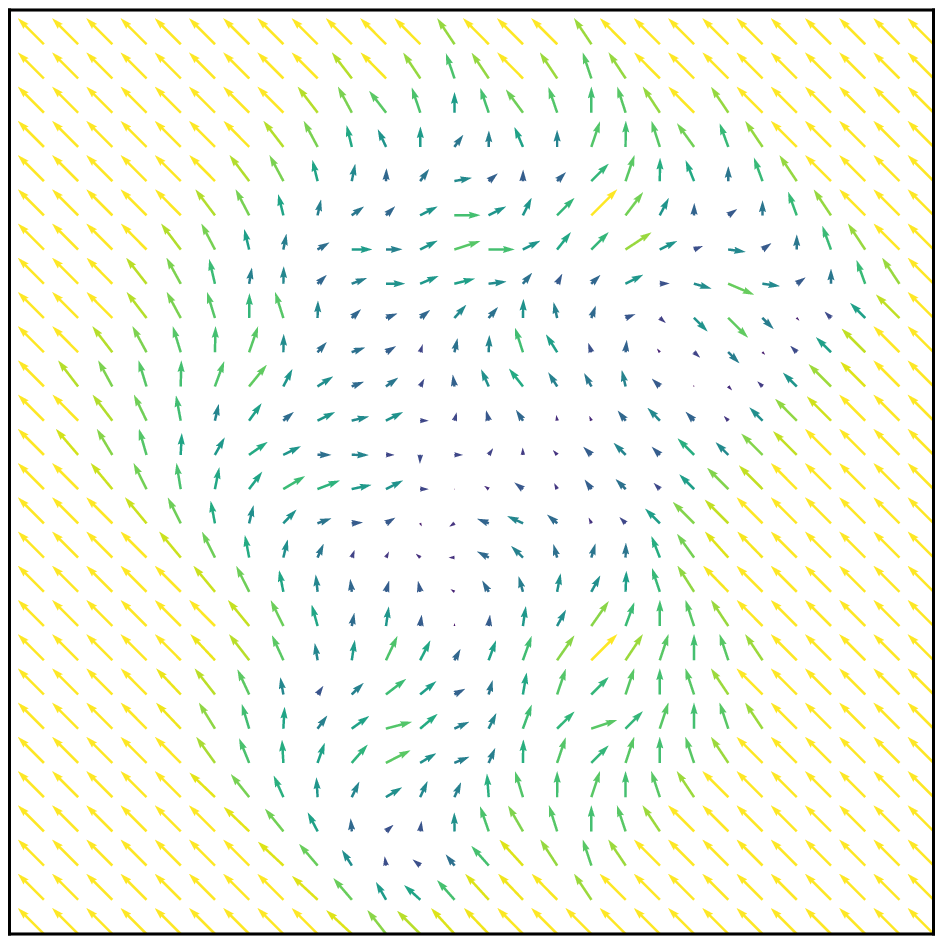}
                    \caption{label 8}
                    \label{fig:visual-comparison:not-certified-and-not-attacked}
                \end{subfigure}
                \begin{subfigure}[t]{.2\linewidth}
                    \includegraphics[width=0.49\textwidth]{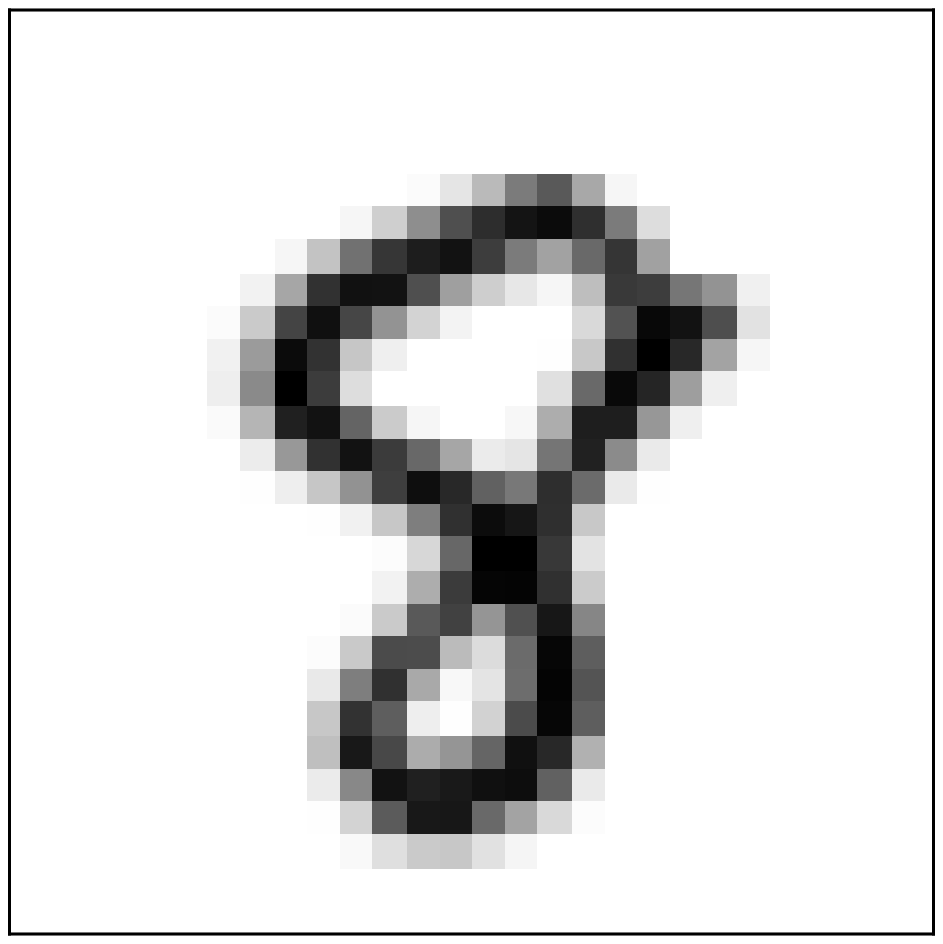}
                    \hspace{-0.2cm}
                    \includegraphics[width=0.49\textwidth]{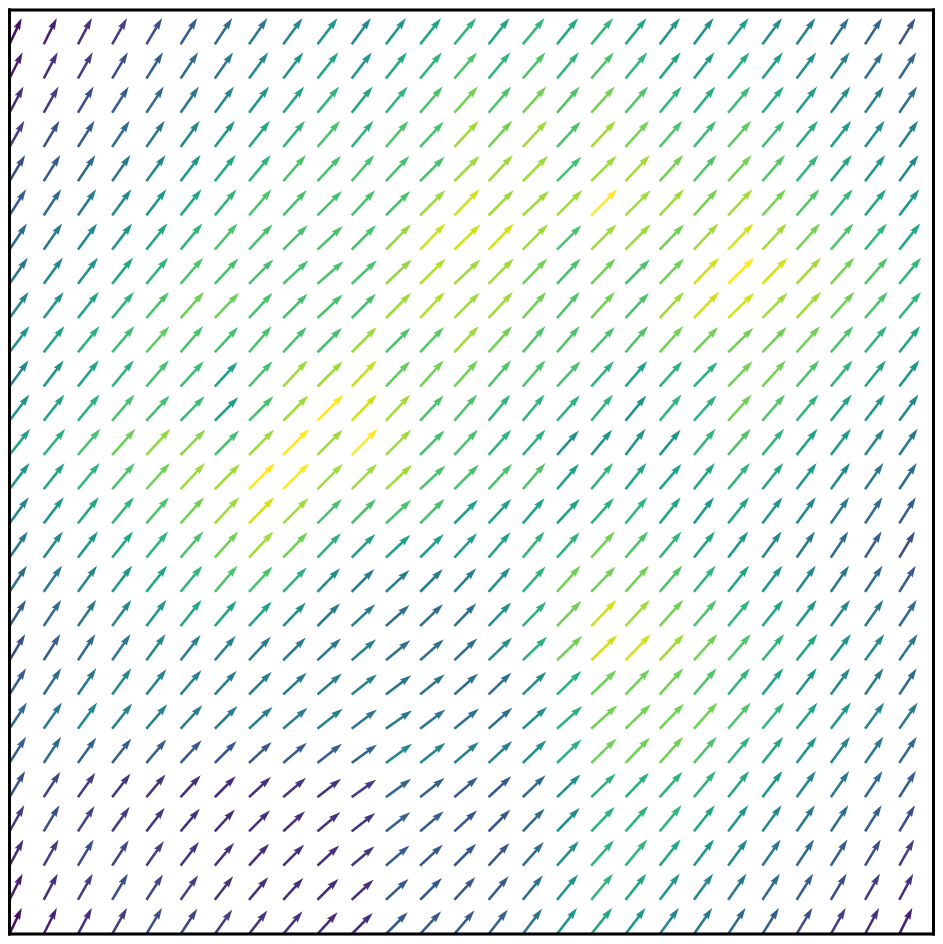}
                    \caption{label 8 (certified)}
                \end{subfigure}
                \begin{subfigure}[t]{.2\linewidth}
                    \includegraphics[width=0.49\textwidth]{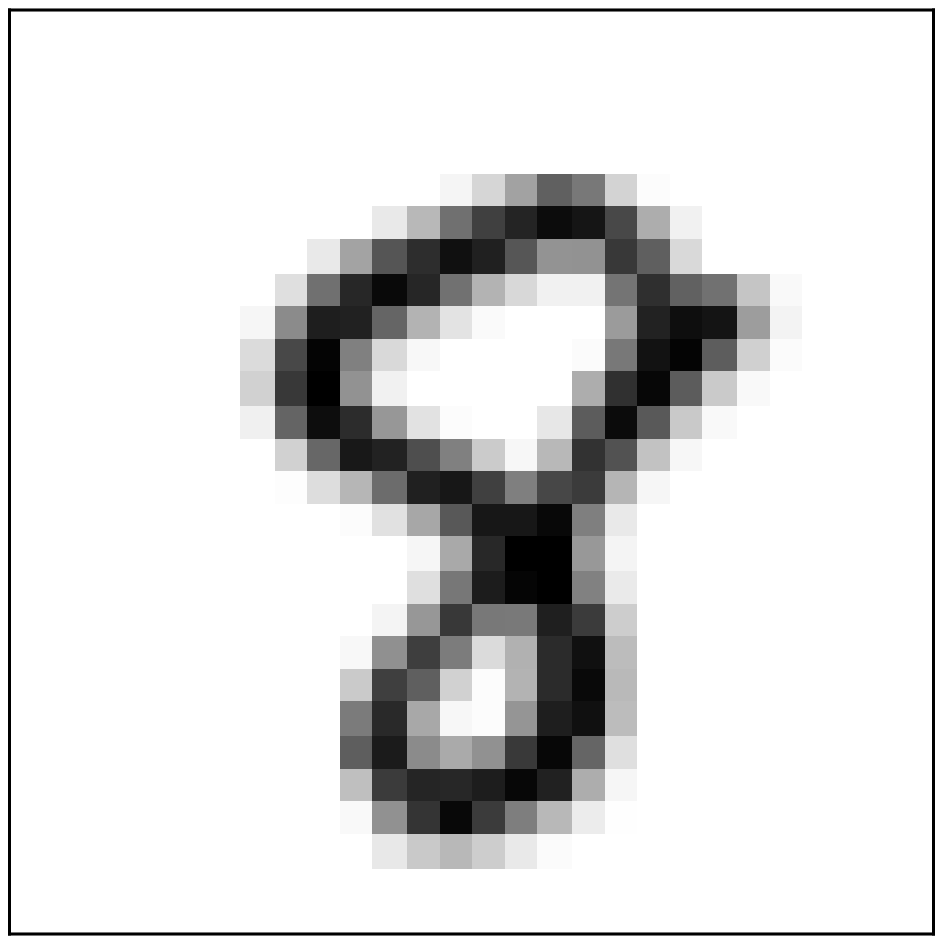}
                    \hspace{-0.2cm}
                    \includegraphics[width=0.49\textwidth]{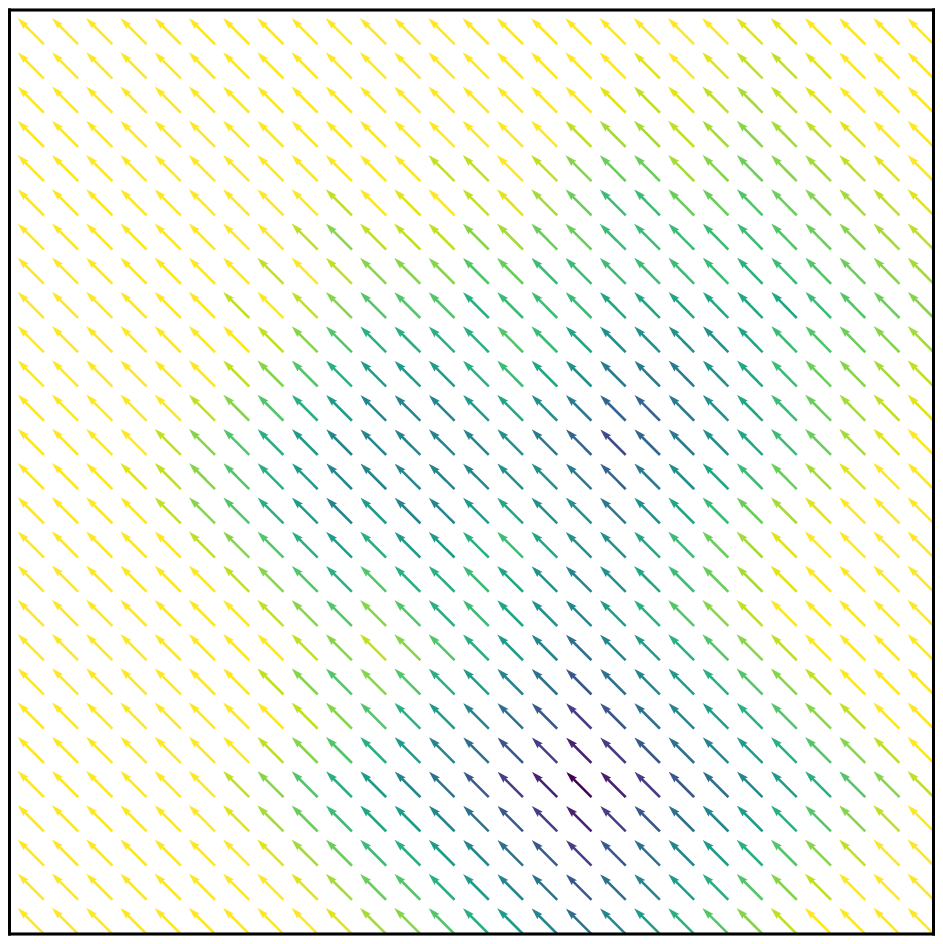}
                    \caption{label 8 (certified)}
                \end{subfigure}
            \end{center}
        \end{subfigure}
        \vspace{0.25cm}
        \begin{subfigure}{\dimexpr0.9\linewidth+20pt\relax}
            \begin{center}
                \makebox[20pt]{\raisebox{20pt}{\rotatebox[origin=c]{90}{$\bm{\delta = 0.4}$}}}
                \begin{subfigure}[t]{.1\linewidth}
                    \includegraphics[width=0.98\textwidth]{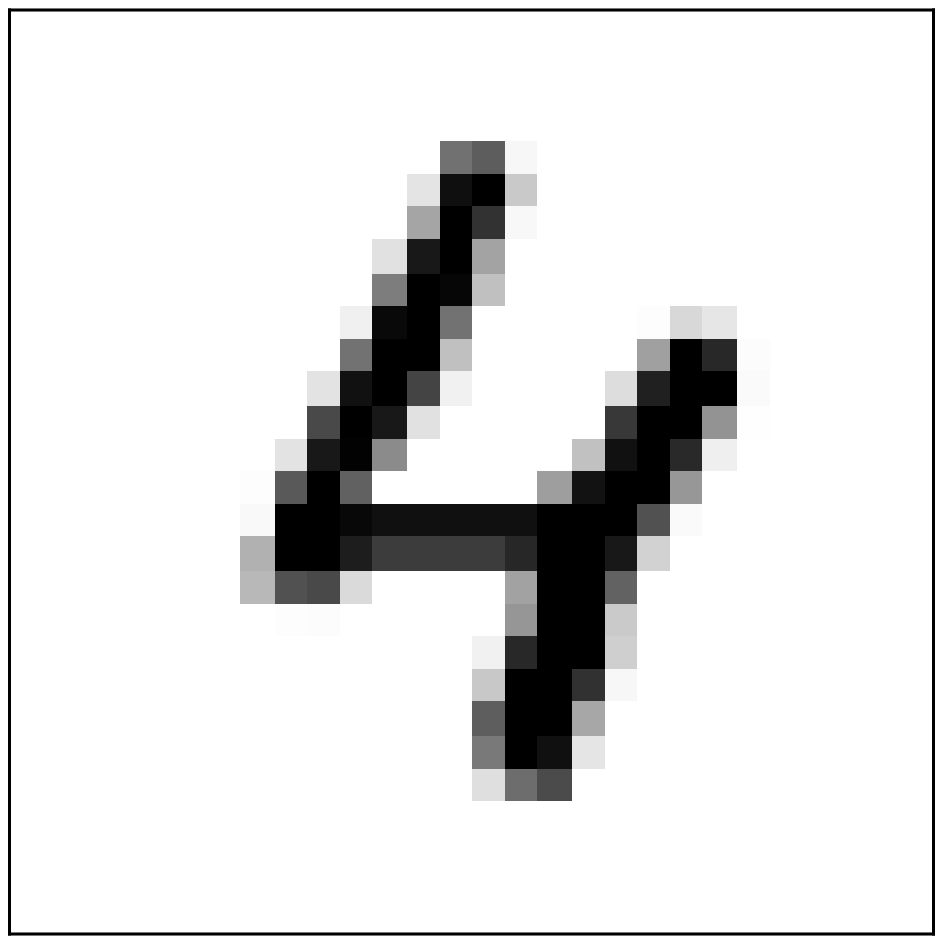}
                \end{subfigure}
                \begin{subfigure}[t]{.2\linewidth}
                    \includegraphics[width=0.49\textwidth]{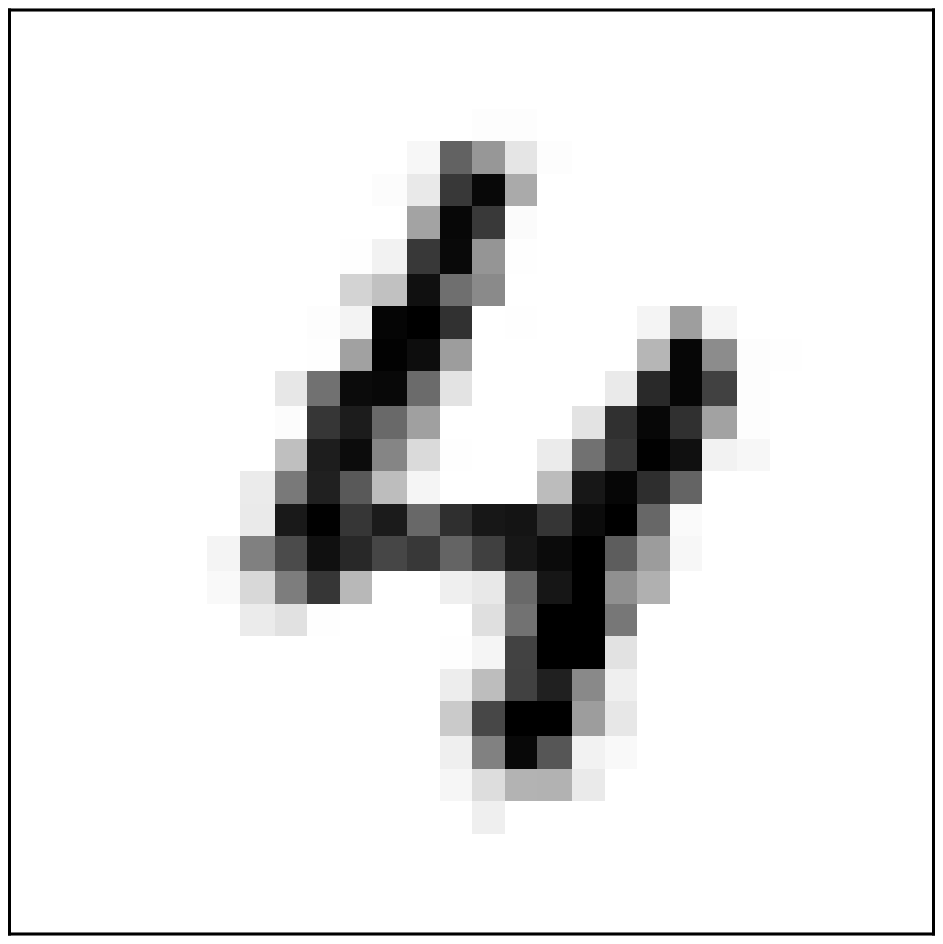}
                    \hspace{-0.2cm}
                    \includegraphics[width=0.49\textwidth]{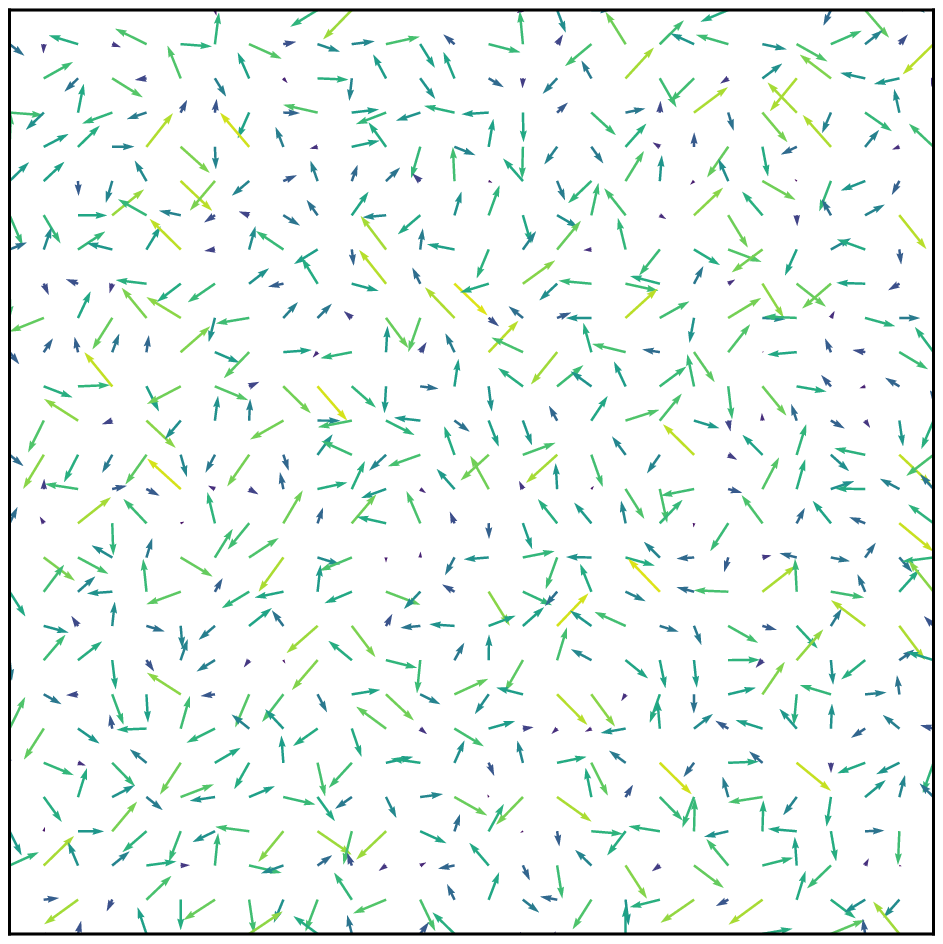}
                    \caption{label 5 (adversarial)}
                \end{subfigure}
                \begin{subfigure}[t]{.2\linewidth}
                    \includegraphics[width=0.49\textwidth]{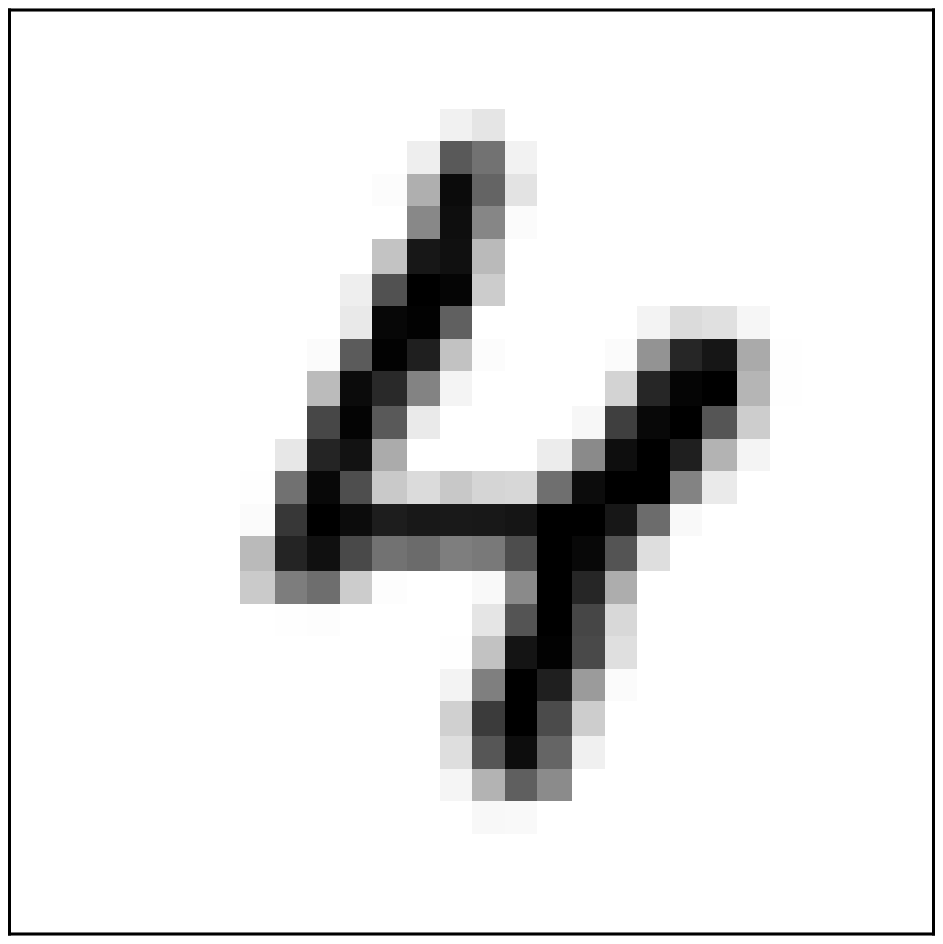}
                    \hspace{-0.2cm}
                    \includegraphics[width=0.49\textwidth]{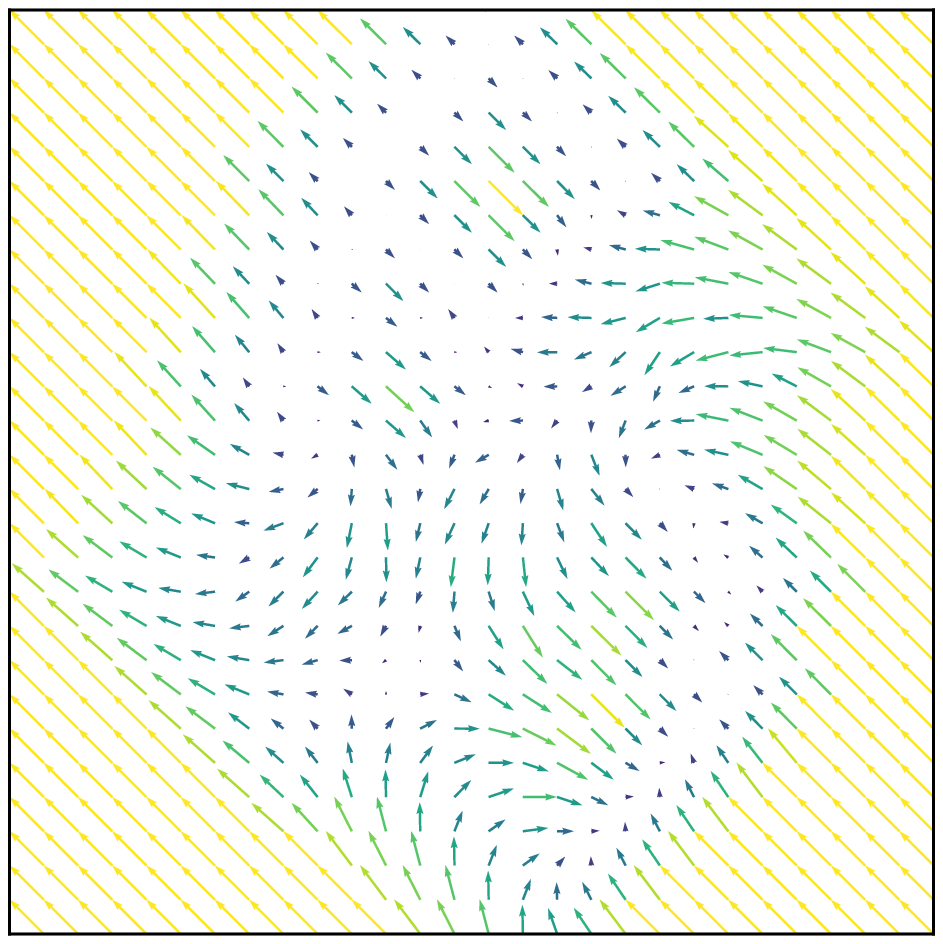}
                    \caption{label 4}
                \end{subfigure}
                \begin{subfigure}[t]{.2\linewidth}
                    \includegraphics[width=0.49\textwidth]{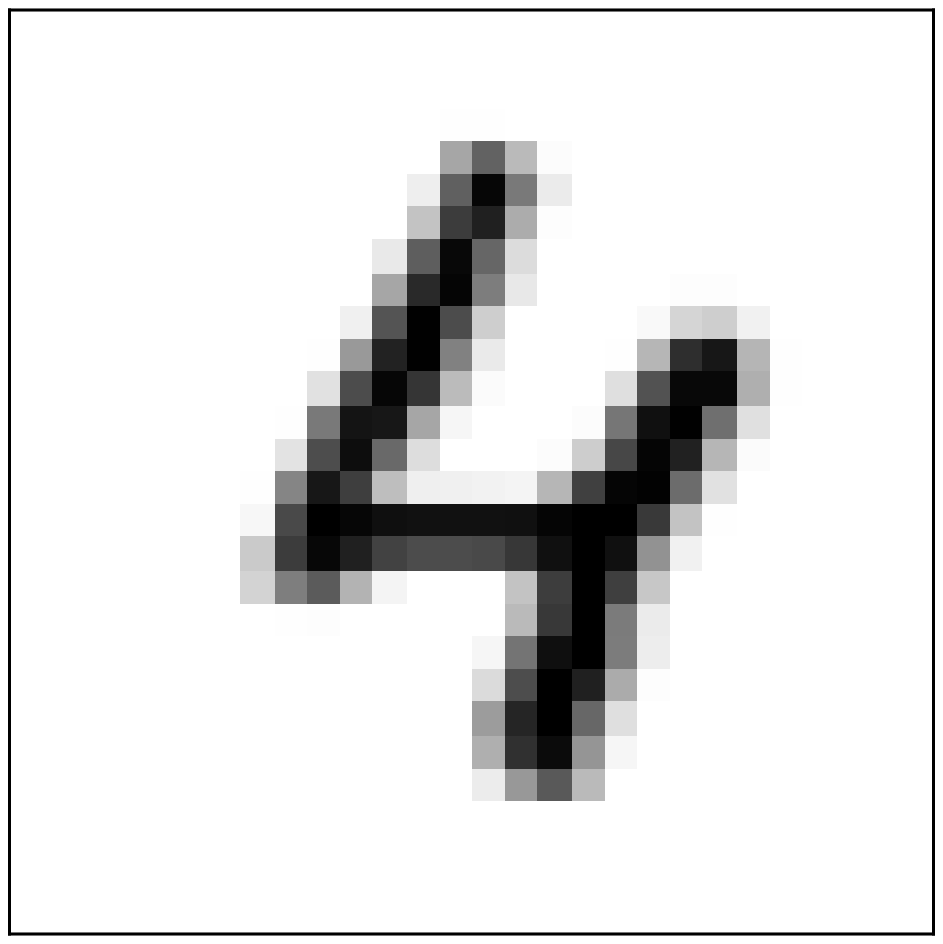}
                    \hspace{-0.2cm}
                    \includegraphics[width=0.49\textwidth]{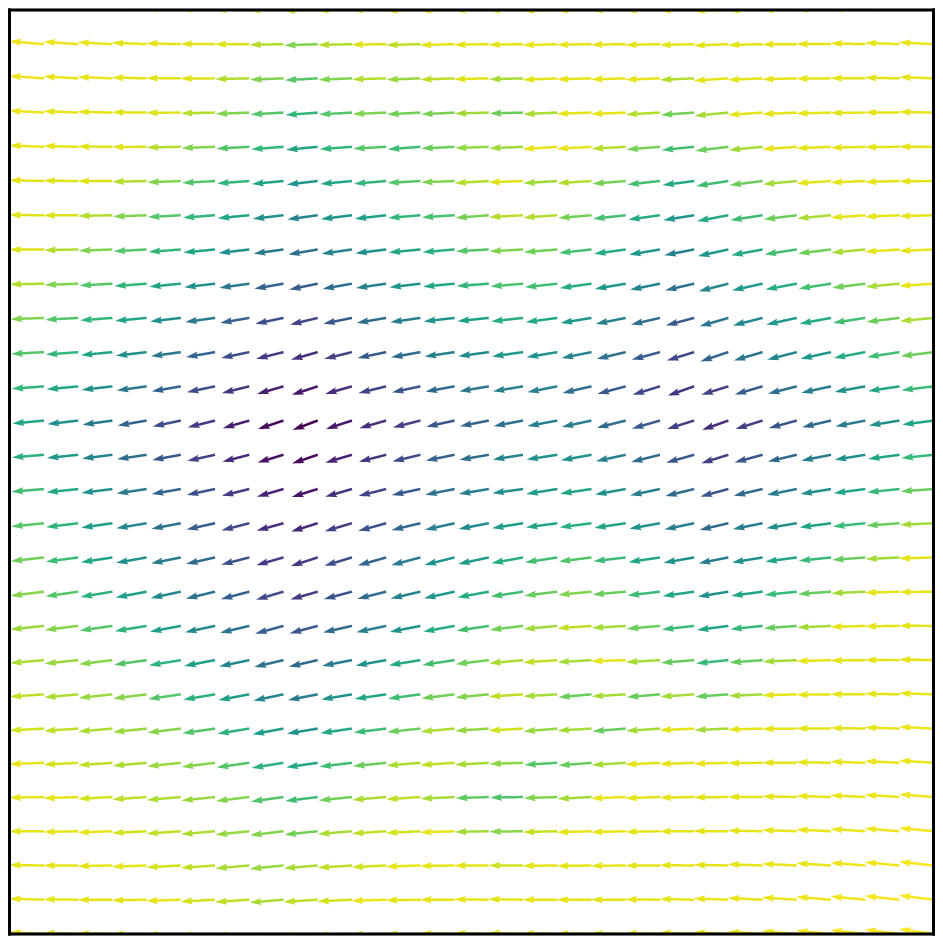}
                    \caption{label 4 (certified)}
                \end{subfigure}
                \begin{subfigure}[t]{.2\linewidth}
                    \includegraphics[width=0.49\textwidth]{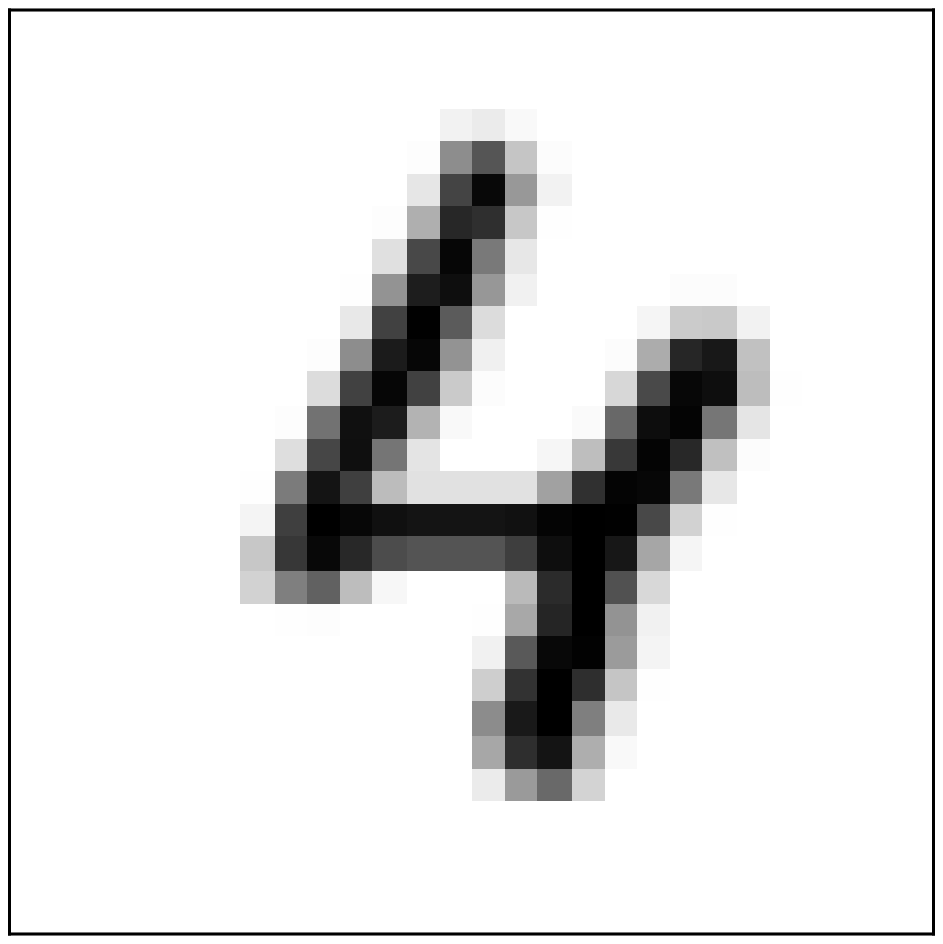}
                    \hspace{-0.2cm}
                    \includegraphics[width=0.49\textwidth]{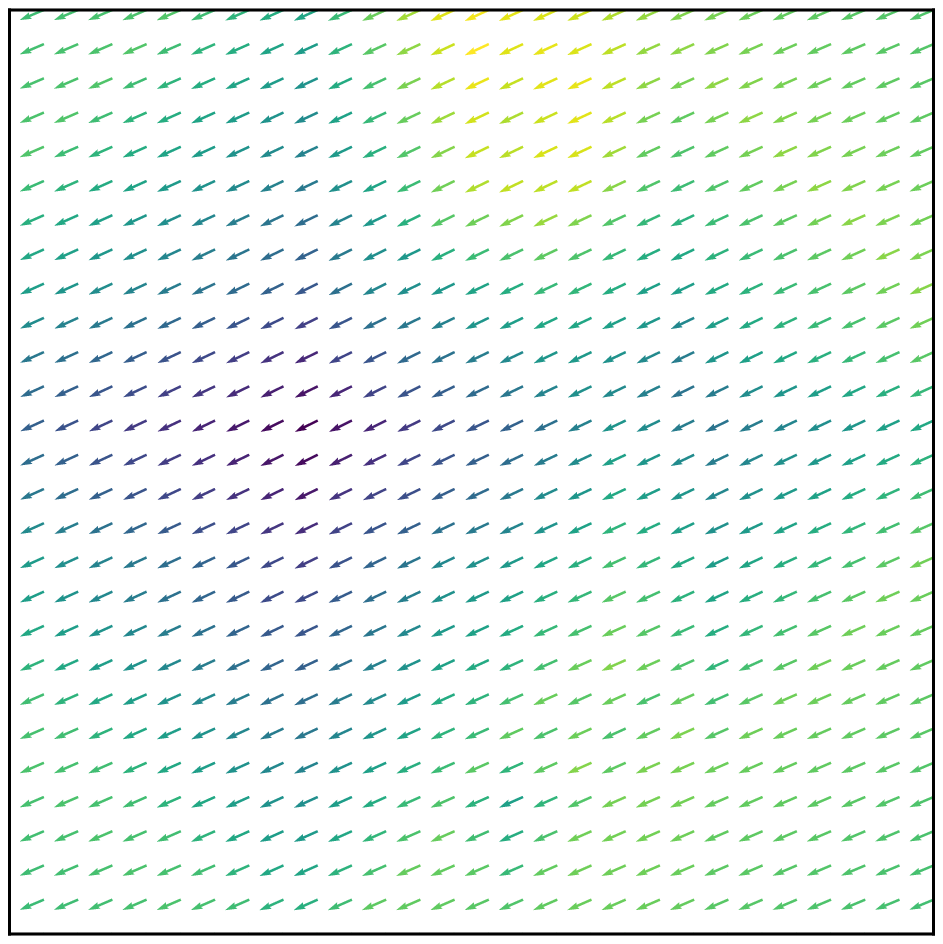}
                    \caption{label 4 (certified)}
                \end{subfigure}
            \end{center}
        \end{subfigure}
        \vspace{0.25cm}
        \begin{subfigure}{\dimexpr0.9\linewidth+20pt\relax}
            \begin{center}
                \makebox[20pt]{\raisebox{20pt}{\rotatebox[origin=c]{90}{$\bm{\delta = 0.5}$}}}
                \begin{subfigure}[t]{.1\linewidth}
                    \includegraphics[width=0.98\textwidth]{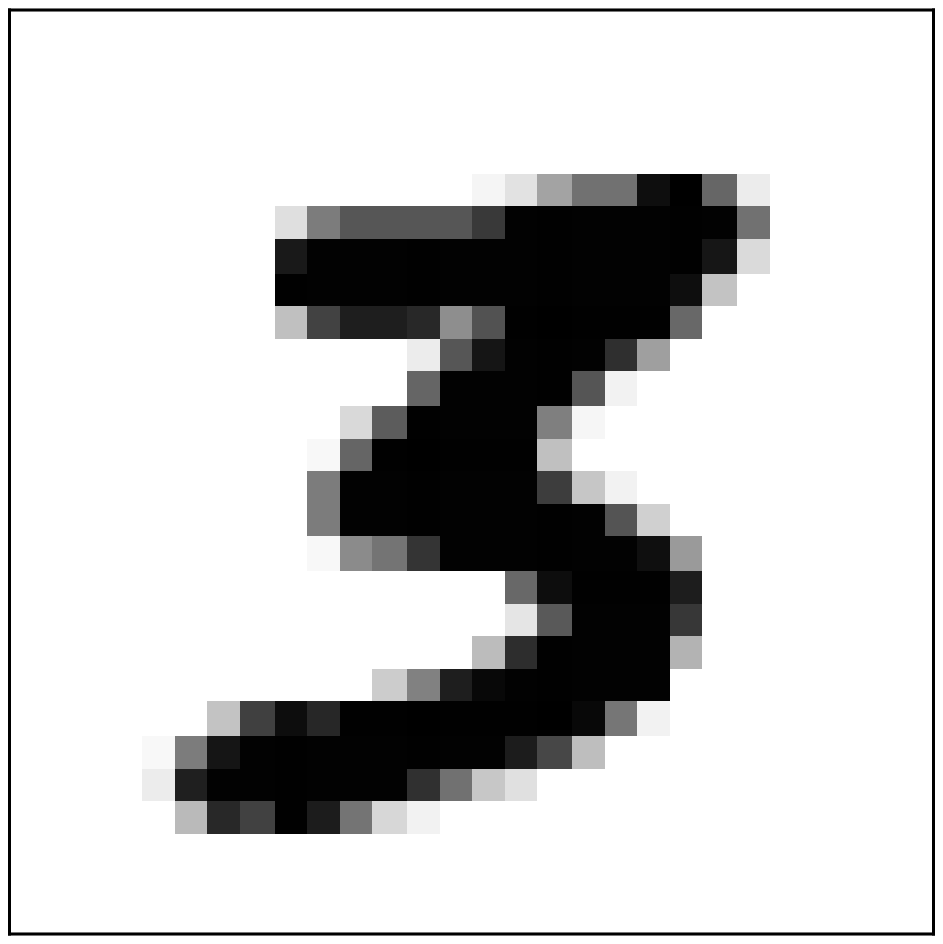}
                \end{subfigure}
                \begin{subfigure}[t]{.2\linewidth}
                    \includegraphics[width=0.49\textwidth]{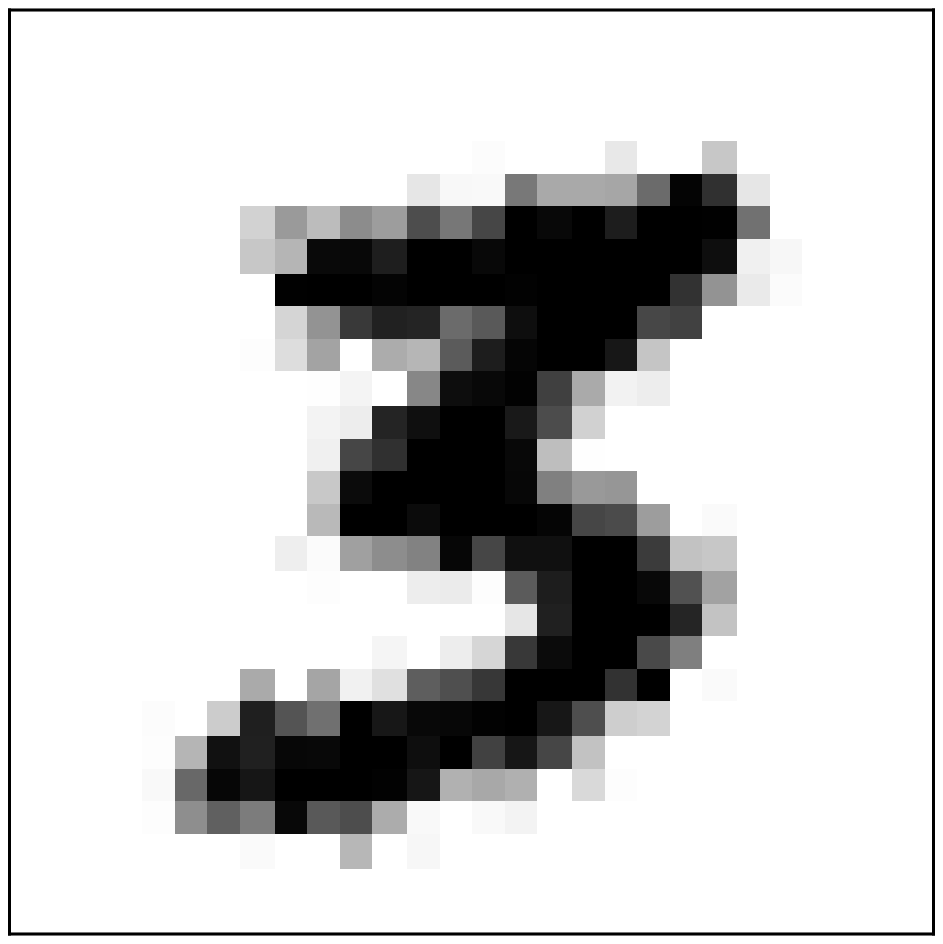}
                    \hspace{-0.2cm}
                    \includegraphics[width=0.49\textwidth]{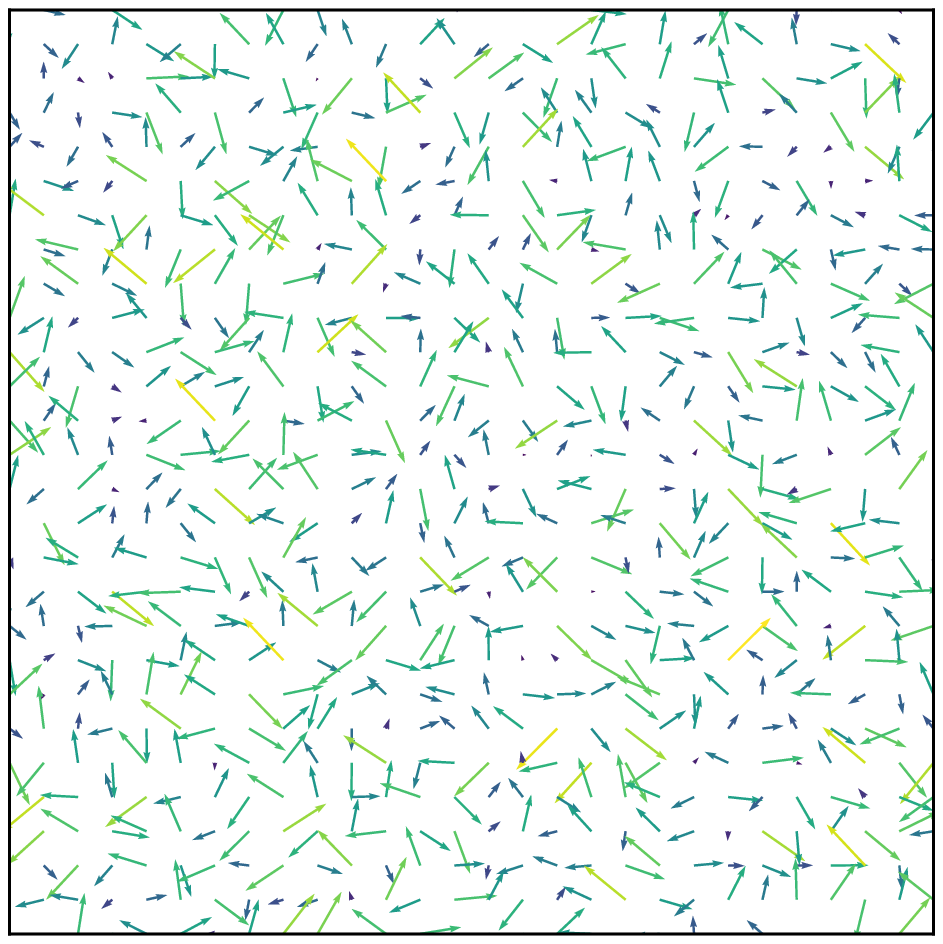}
                    \caption{label 5 (adversarial)}
                \end{subfigure}
                \begin{subfigure}[t]{.2\linewidth}
                    \includegraphics[width=0.49\textwidth]{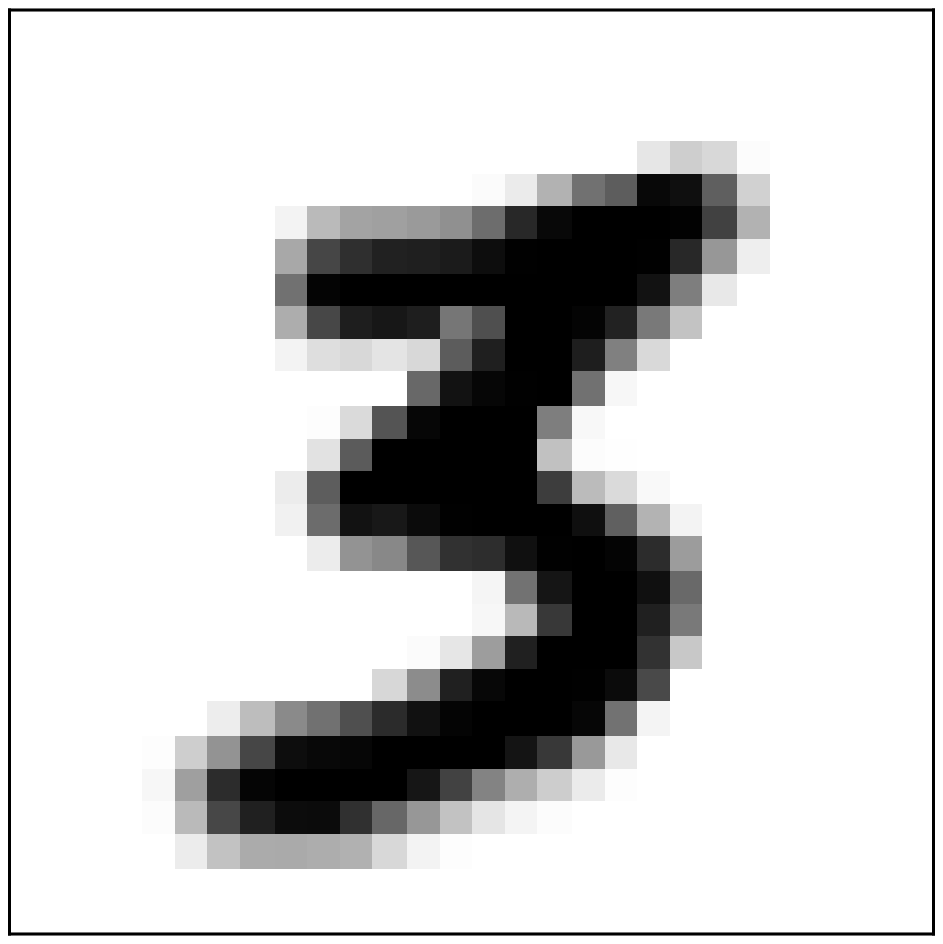}
                    \hspace{-0.2cm}
                    \includegraphics[width=0.49\textwidth]{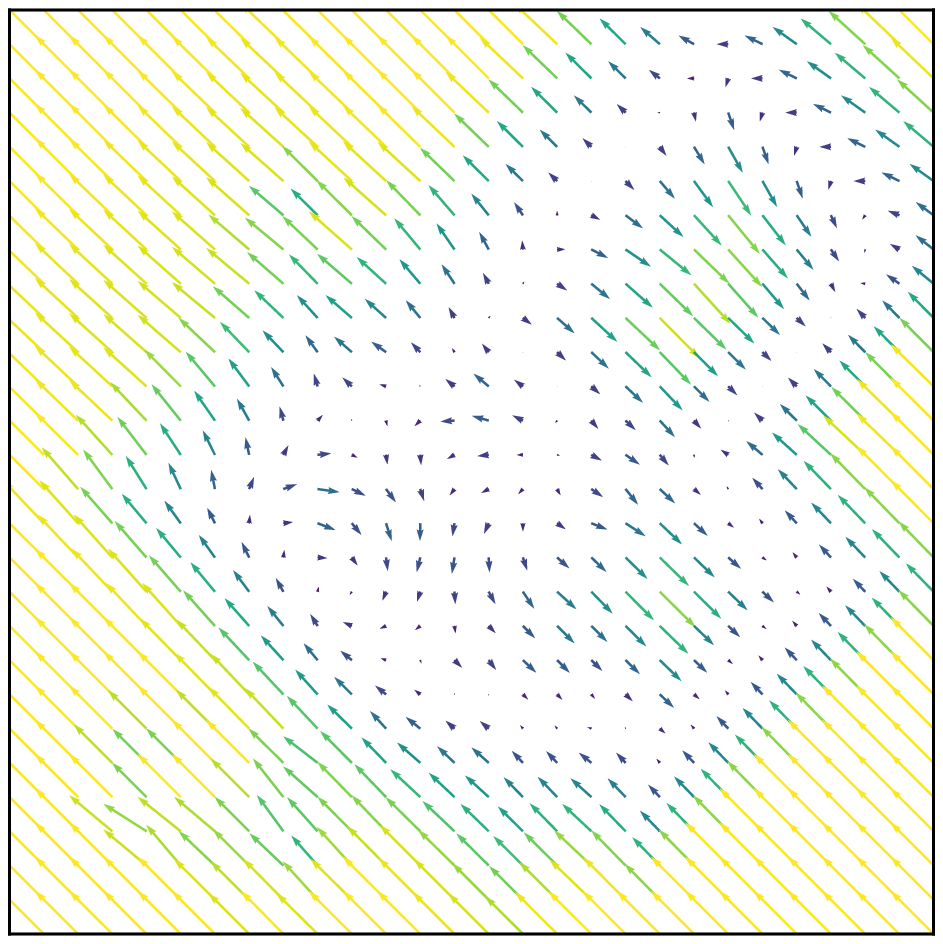}
                    \caption{label 9 (adversarial)}
                \end{subfigure}
                \begin{subfigure}[t]{.2\linewidth}
                    \includegraphics[width=0.49\textwidth]{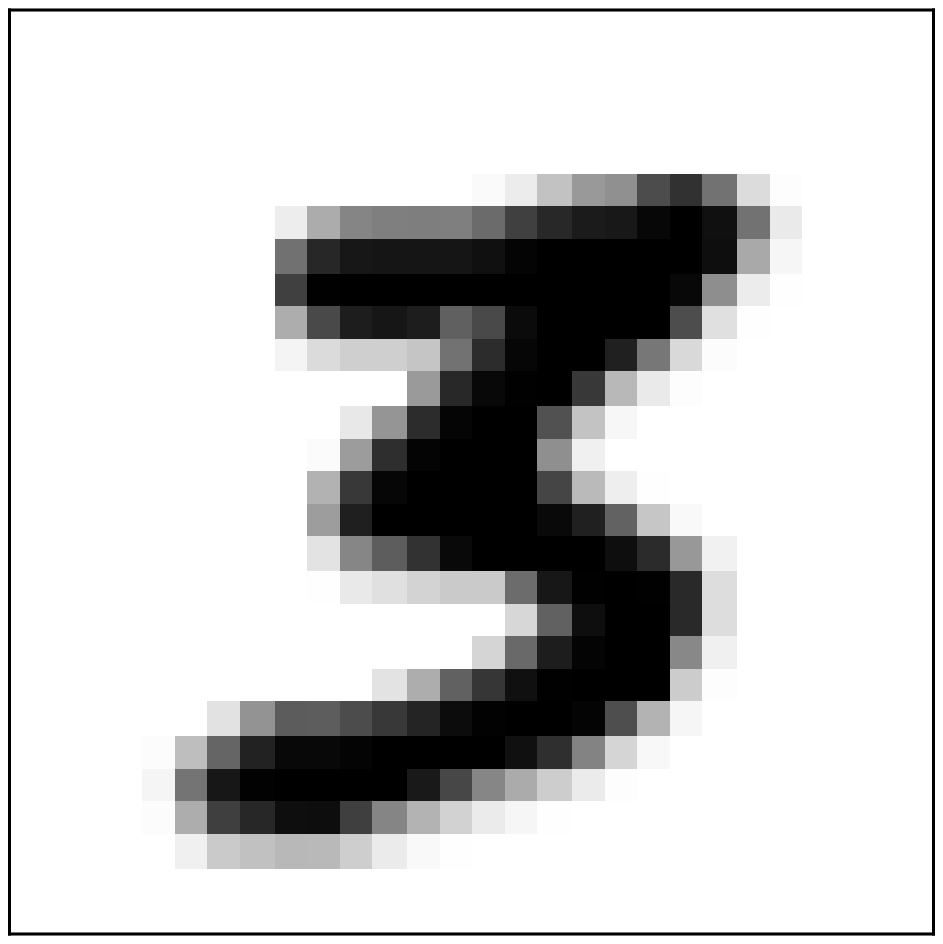}
                    \hspace{-0.2cm}
                    \includegraphics[width=0.49\textwidth]{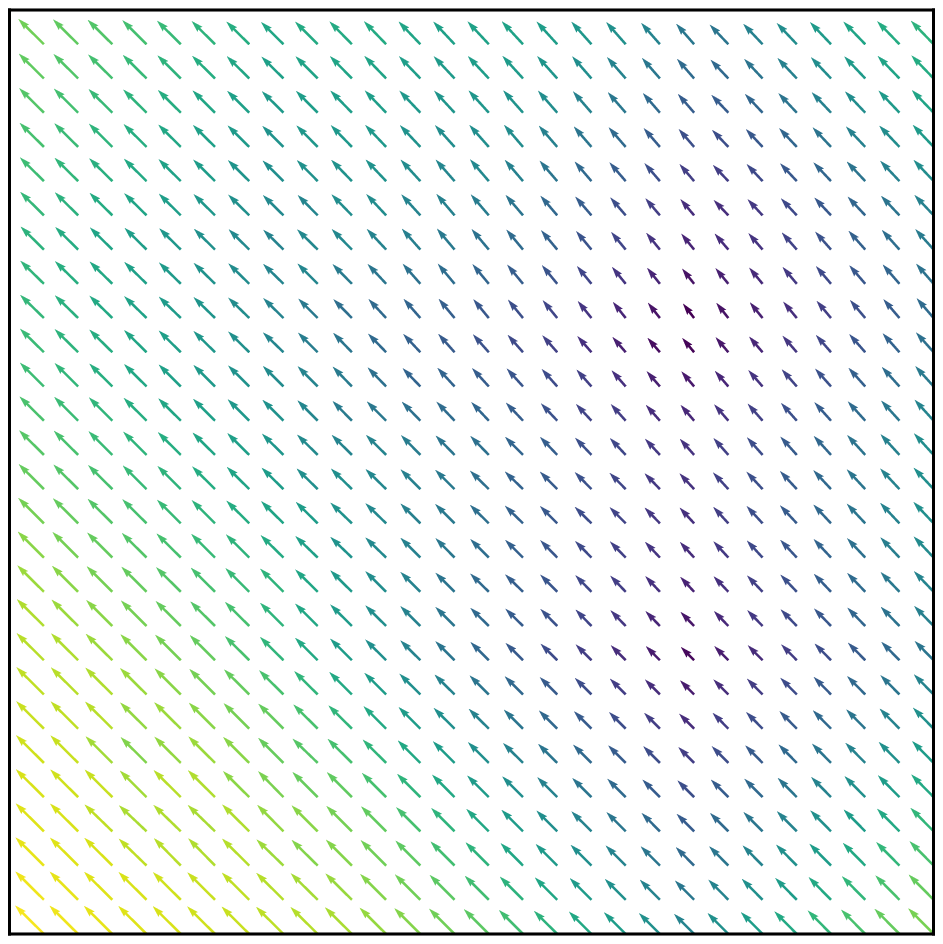}
                    \caption{label 3 (certified)}
                \end{subfigure}
                \begin{subfigure}[t]{.2\linewidth}
                    \includegraphics[width=0.49\textwidth]{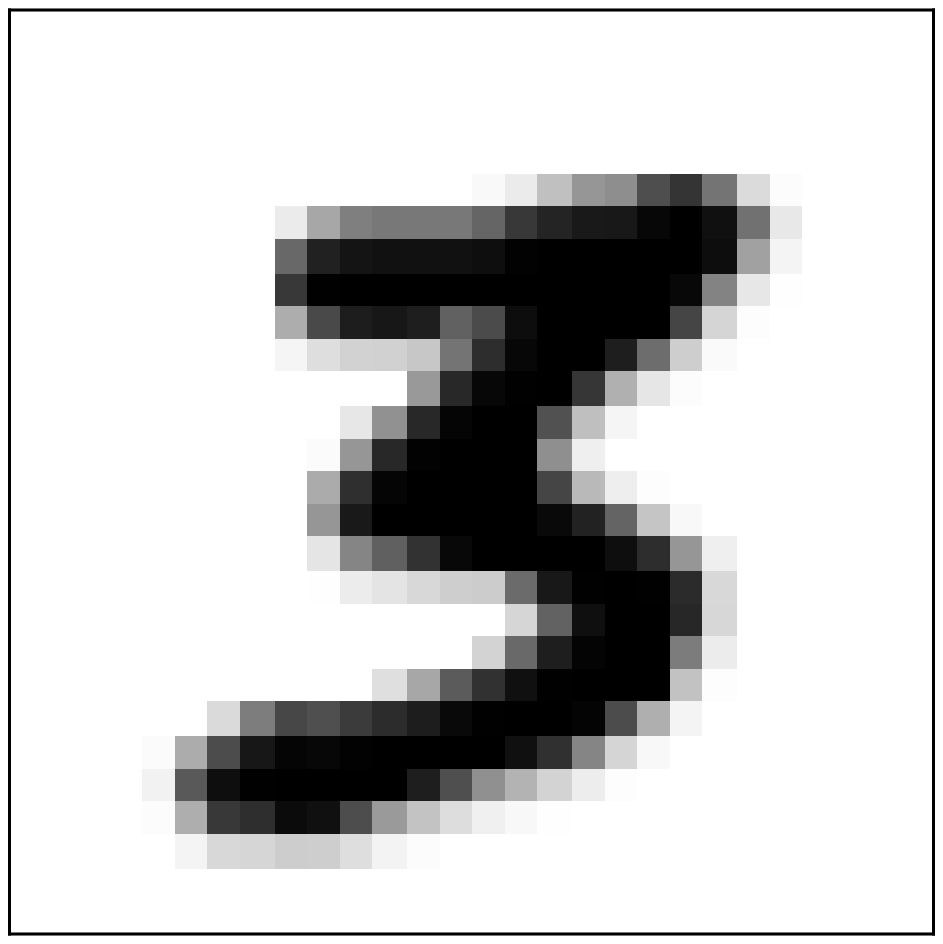}
                    \hspace{-0.2cm}
                    \includegraphics[width=0.49\textwidth]{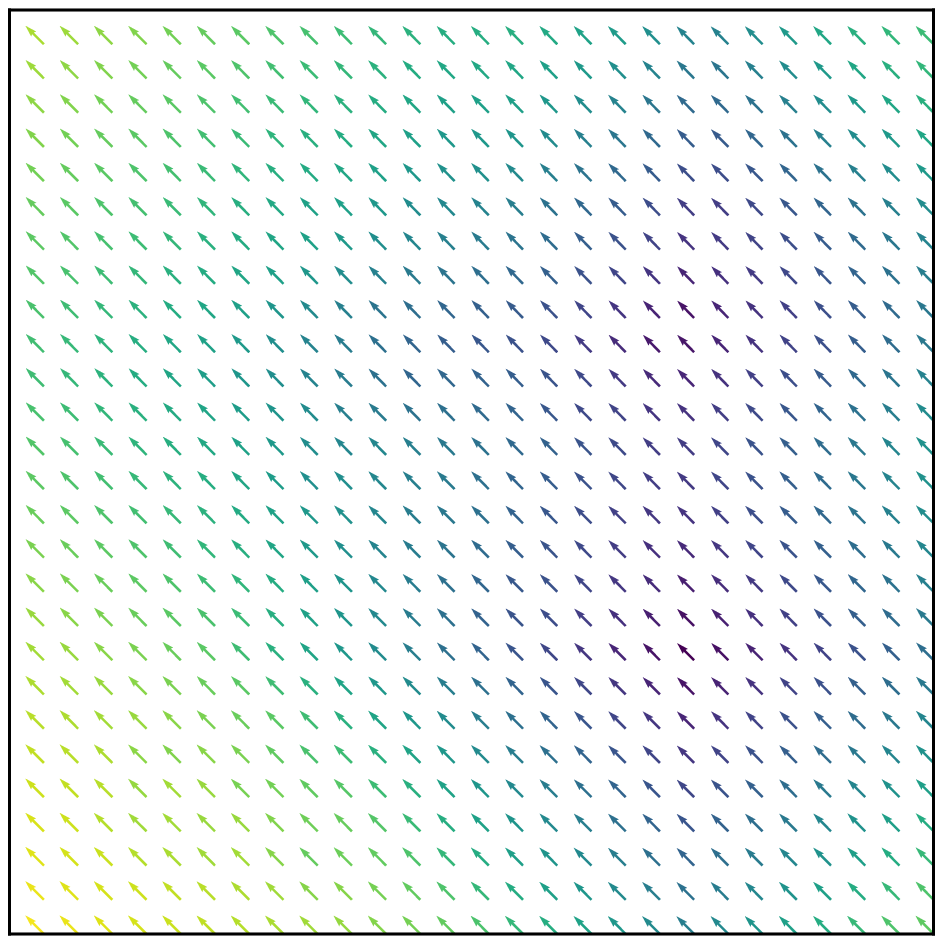}
                    \caption{label 3 (certified)}
                \end{subfigure}
            \end{center}
        \end{subfigure}
        \vspace{0.25cm}
        \begin{subfigure}{\dimexpr0.9\linewidth+20pt\relax}
            \begin{center}
                \makebox[20pt]{\raisebox{20pt}{\rotatebox[origin=c]{90}{$\bm{\delta = 0.6}$}}}
                \begin{subfigure}[t]{.1\linewidth}
                    \includegraphics[width=0.98\textwidth]{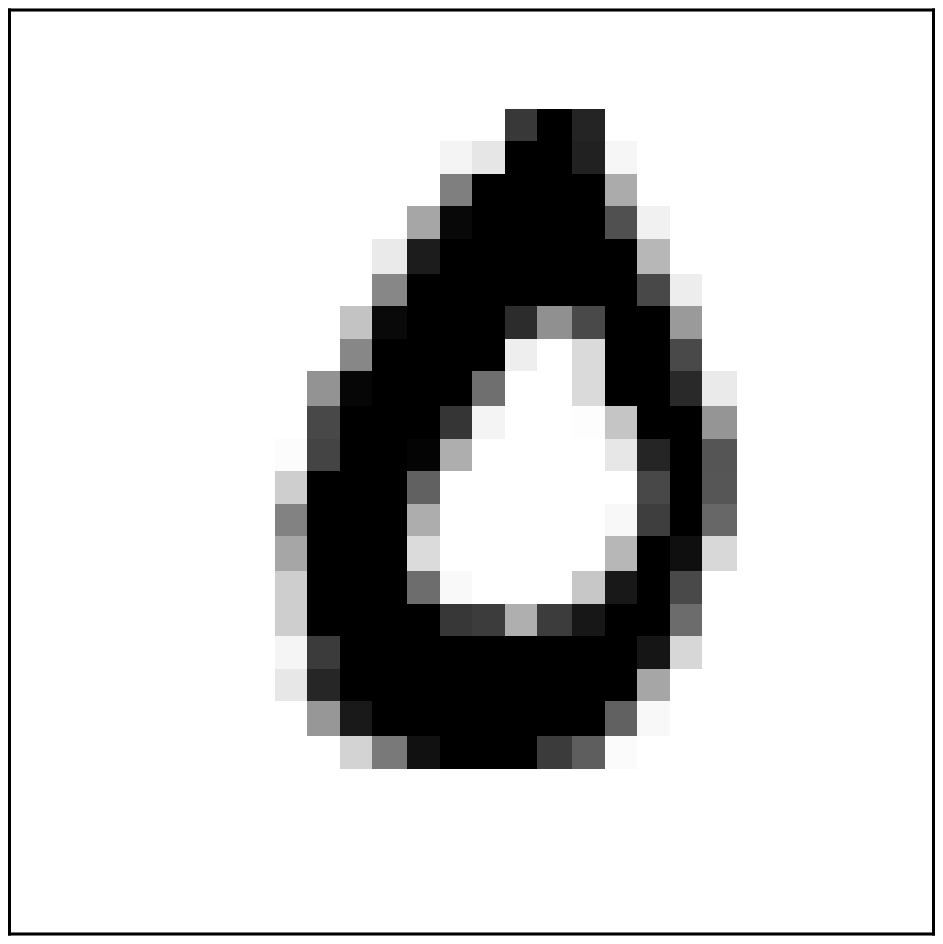}
                \end{subfigure}
                \begin{subfigure}[t]{.2\linewidth}
                    \includegraphics[width=0.49\textwidth]{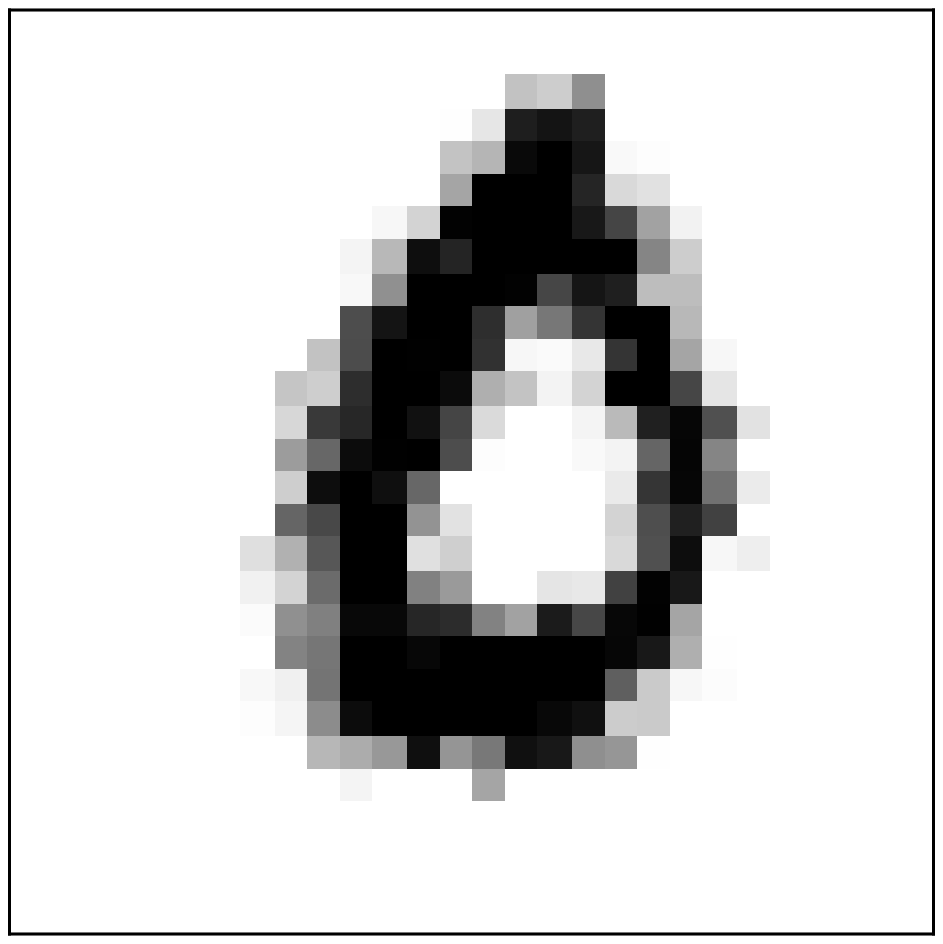}
                    \hspace{-0.2cm}
                    \includegraphics[width=0.49\textwidth]{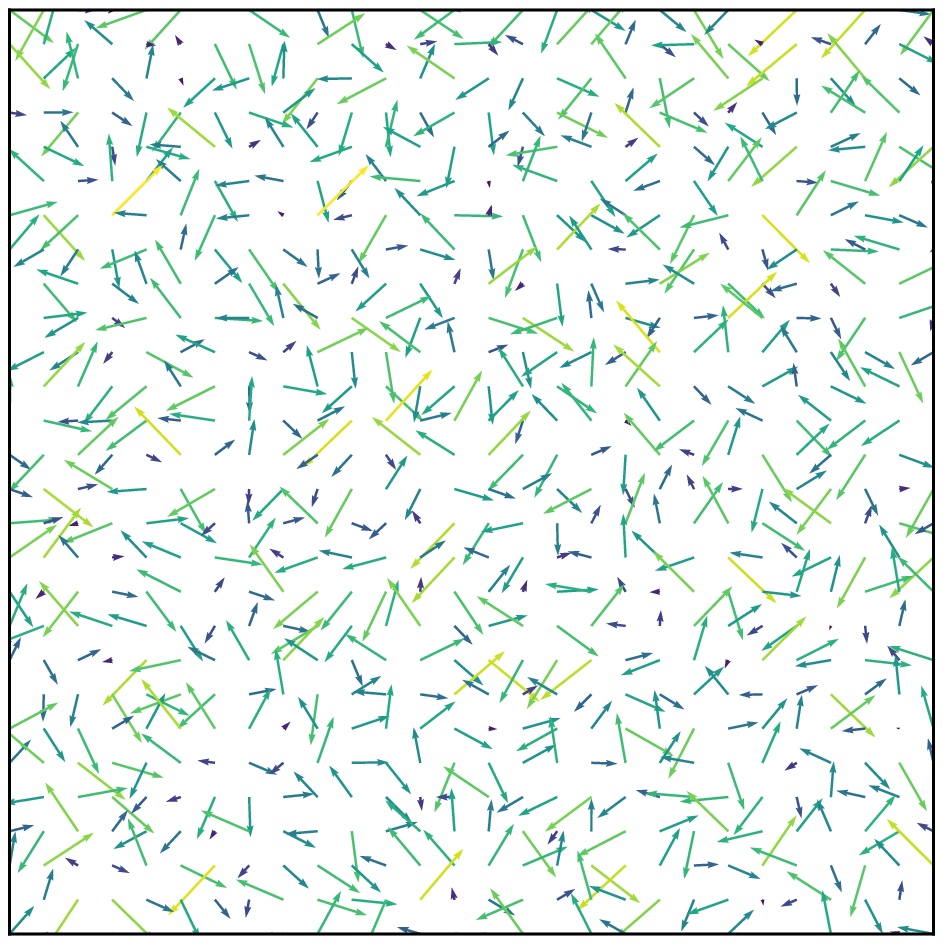}
                    \caption{label 4 (adversarial)}
                \end{subfigure}
                \begin{subfigure}[t]{.2\linewidth}
                    \includegraphics[width=0.49\textwidth]{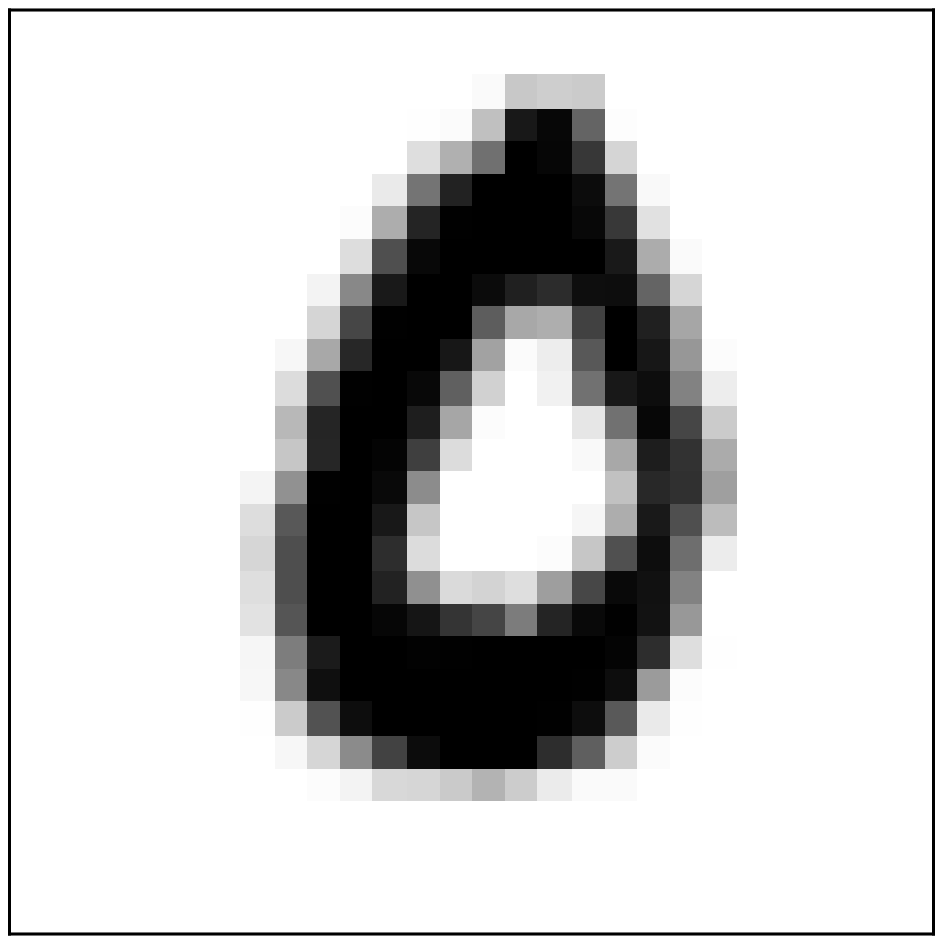}
                    \hspace{-0.2cm}
                    \includegraphics[width=0.49\textwidth]{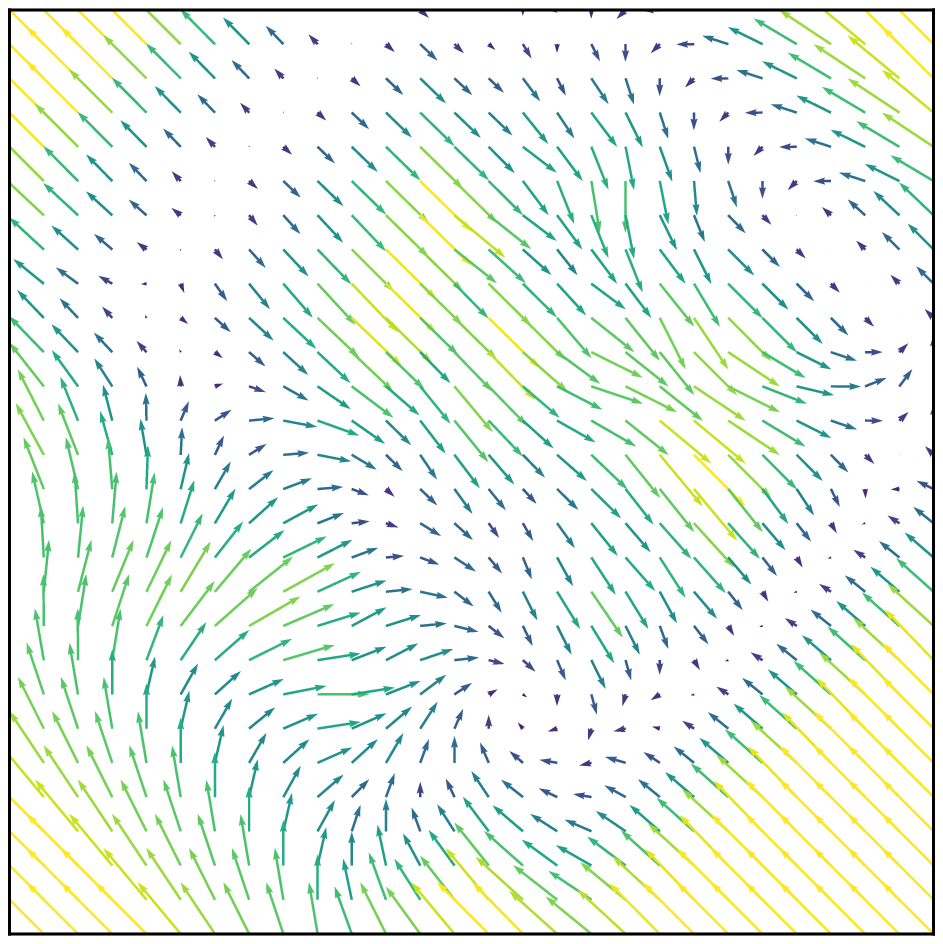}
                    \caption{label 0 (certified)}
                \end{subfigure}
                \begin{subfigure}[t]{.2\linewidth}
                    \includegraphics[width=0.49\textwidth]{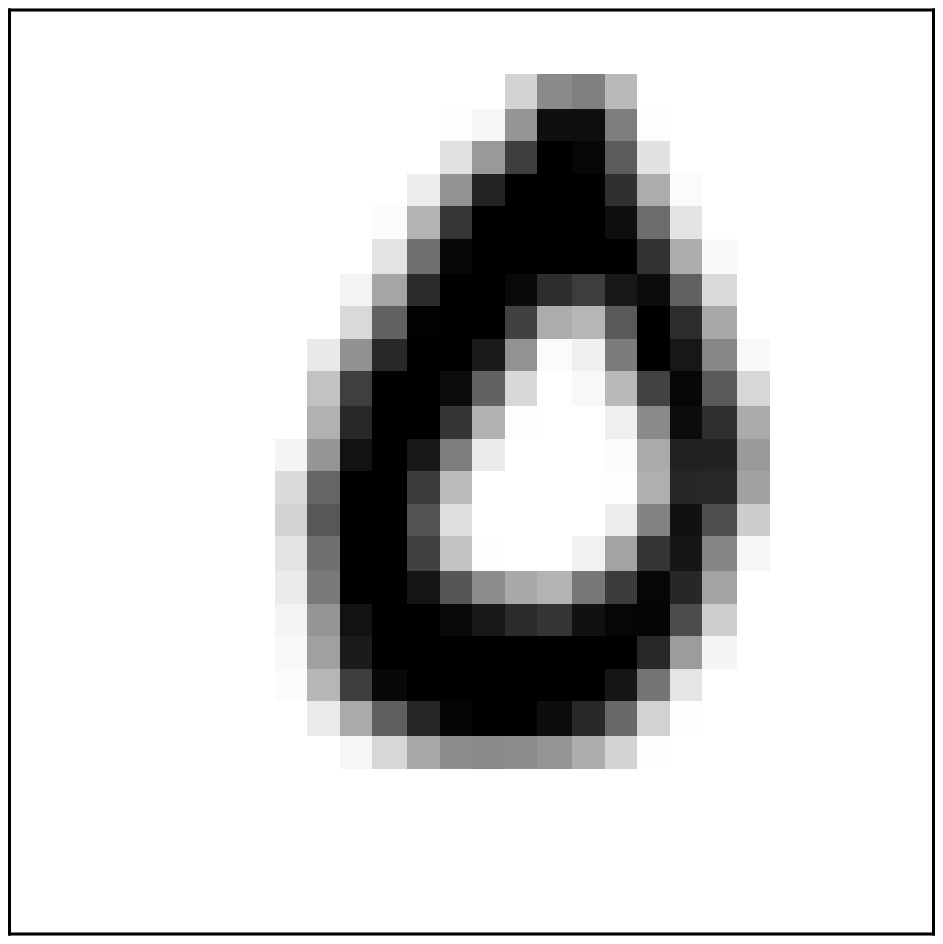}
                    \hspace{-0.2cm}
                    \includegraphics[width=0.49\textwidth]{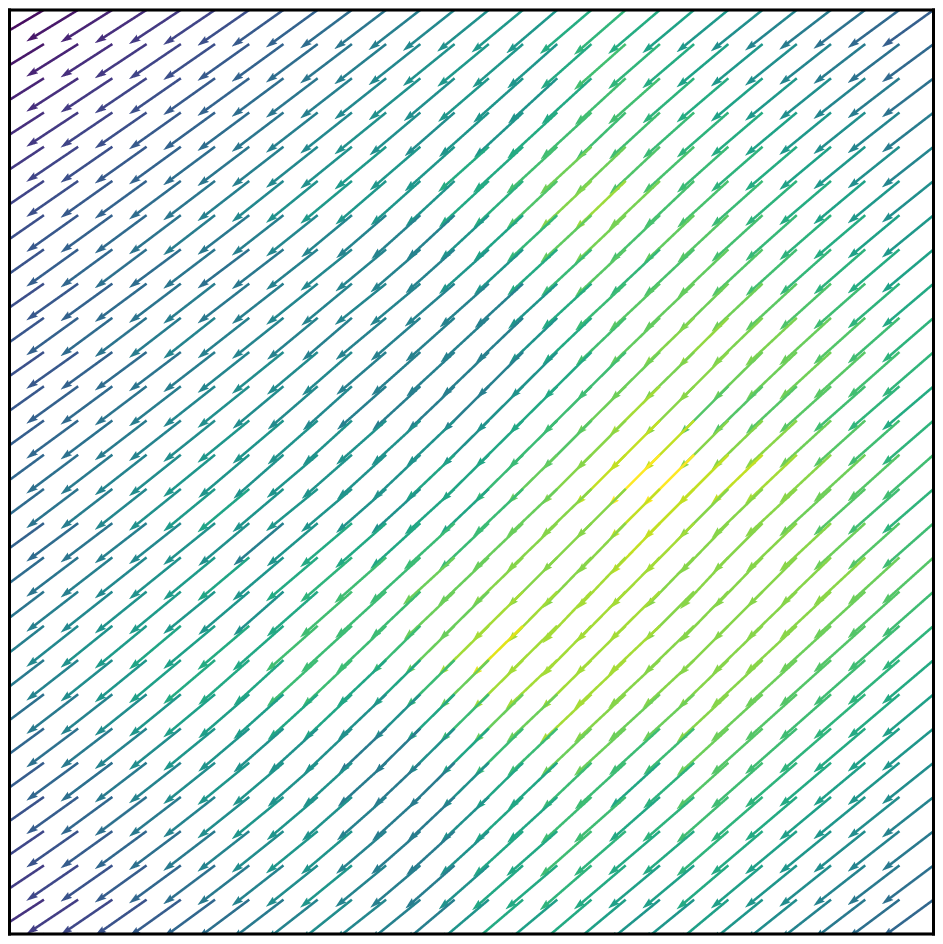}
                    \caption{label 0 (certified)}
                \end{subfigure}
                \begin{subfigure}[t]{.2\linewidth}
                    \includegraphics[width=0.49\textwidth]{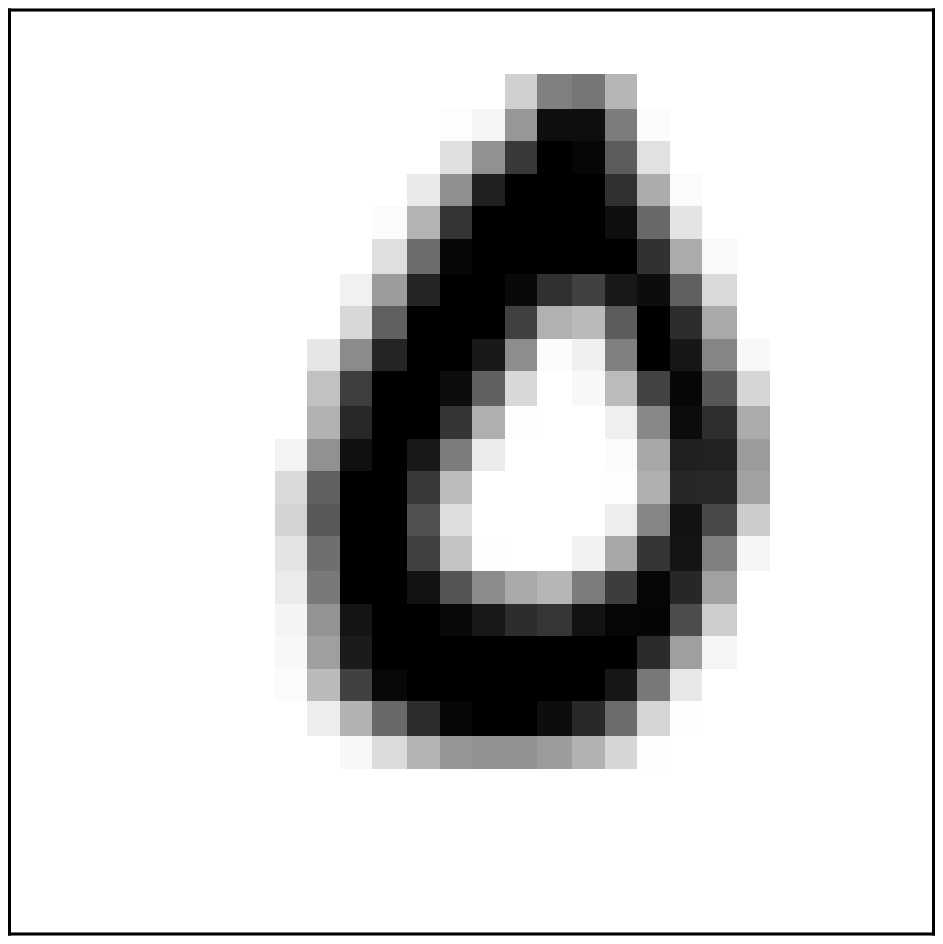}
                    \hspace{-0.2cm}
                    \includegraphics[width=0.49\textwidth]{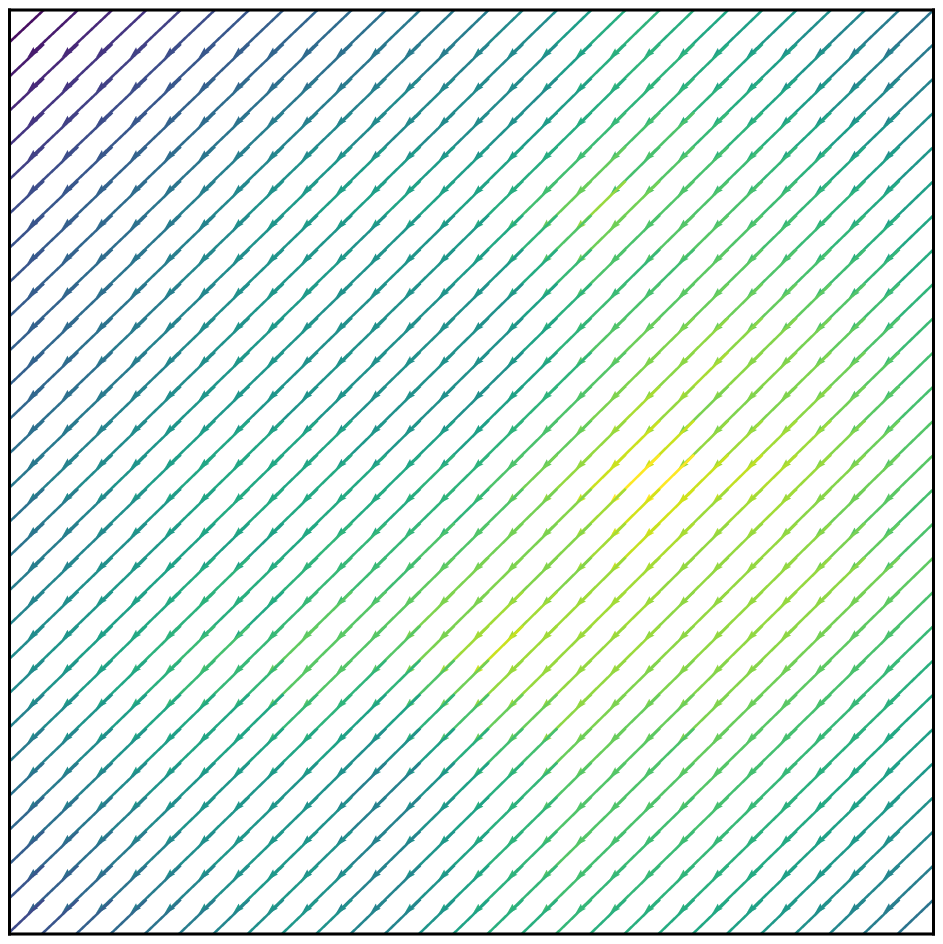}
                    \caption{label 0 (certified)}
                \end{subfigure}
            \end{center}
        \end{subfigure}
        \vspace{0.25cm}
        \begin{subfigure}{\dimexpr0.9\linewidth+20pt\relax}
            \begin{center}
                \makebox[20pt]{\raisebox{20pt}{\rotatebox[origin=c]{90}{$\bm{\delta = 0.7}$}}}
                \begin{subfigure}[t]{.1\linewidth}
                    \includegraphics[width=0.98\textwidth]{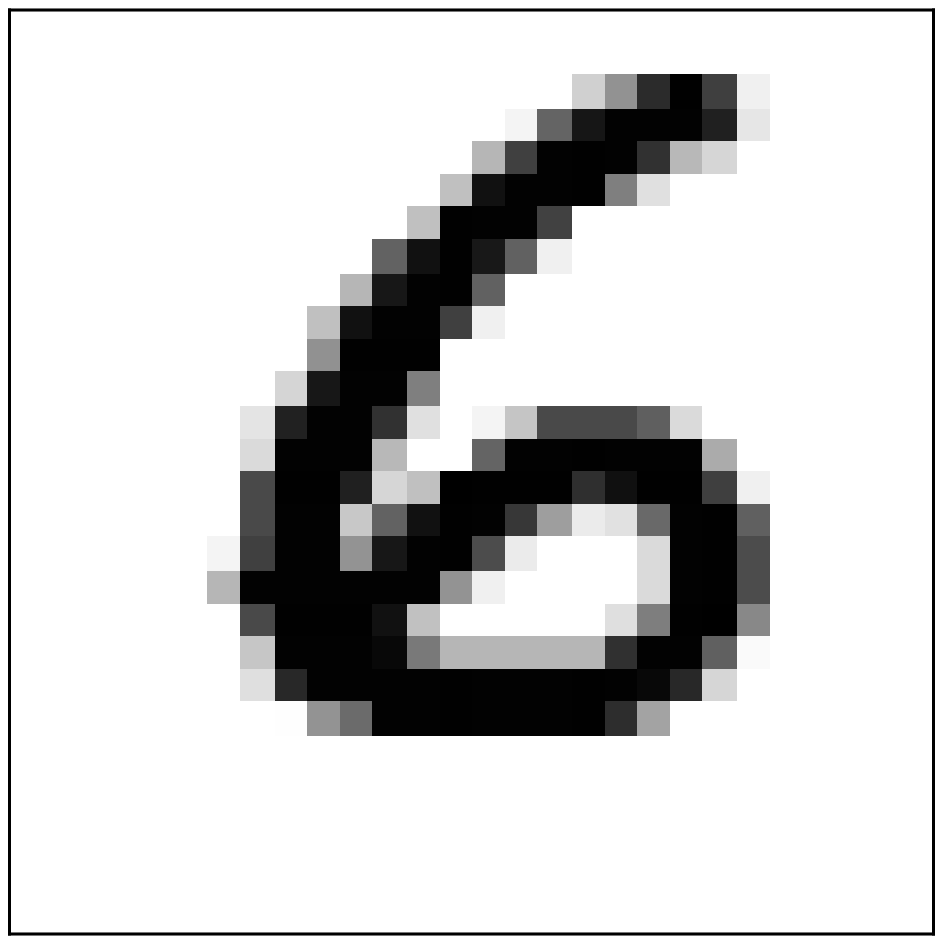}
                \end{subfigure}
                \begin{subfigure}[t]{.2\linewidth}
                    \includegraphics[width=0.49\textwidth]{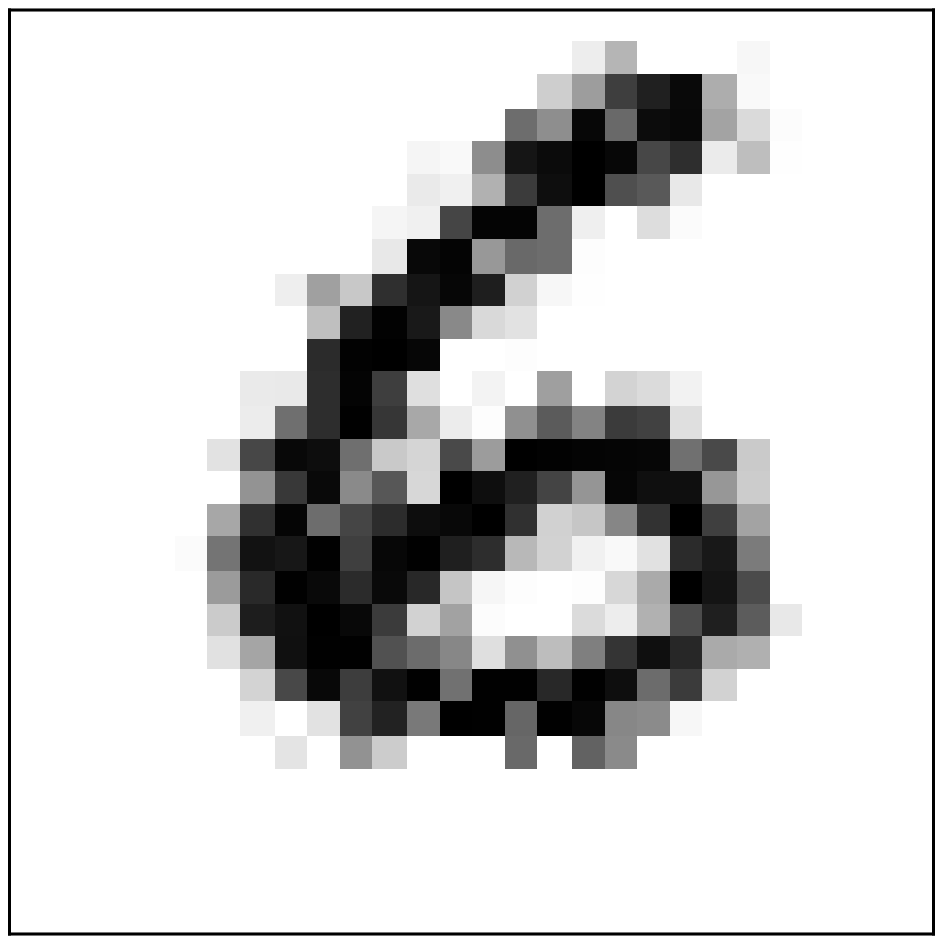}
                    \hspace{-0.2cm}
                    \includegraphics[width=0.49\textwidth]{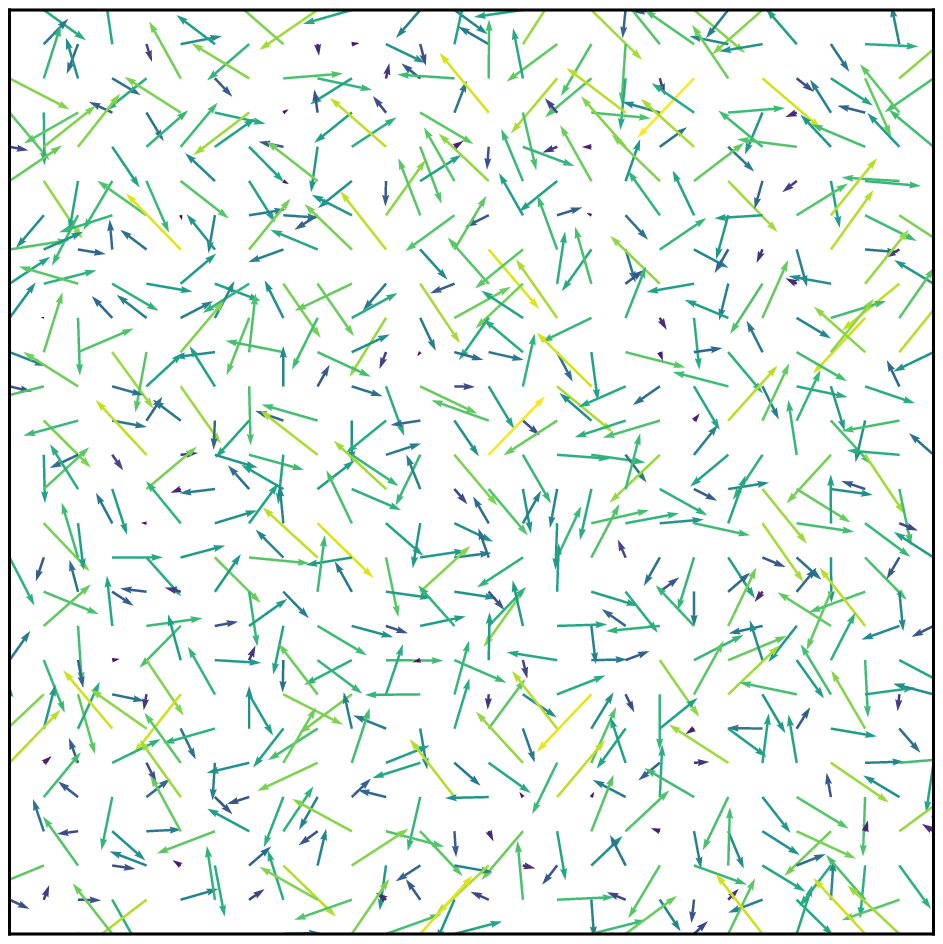}
                    \caption{label 6 (certified)}
                \end{subfigure}
                \begin{subfigure}[t]{.2\linewidth}
                    \includegraphics[width=0.49\textwidth]{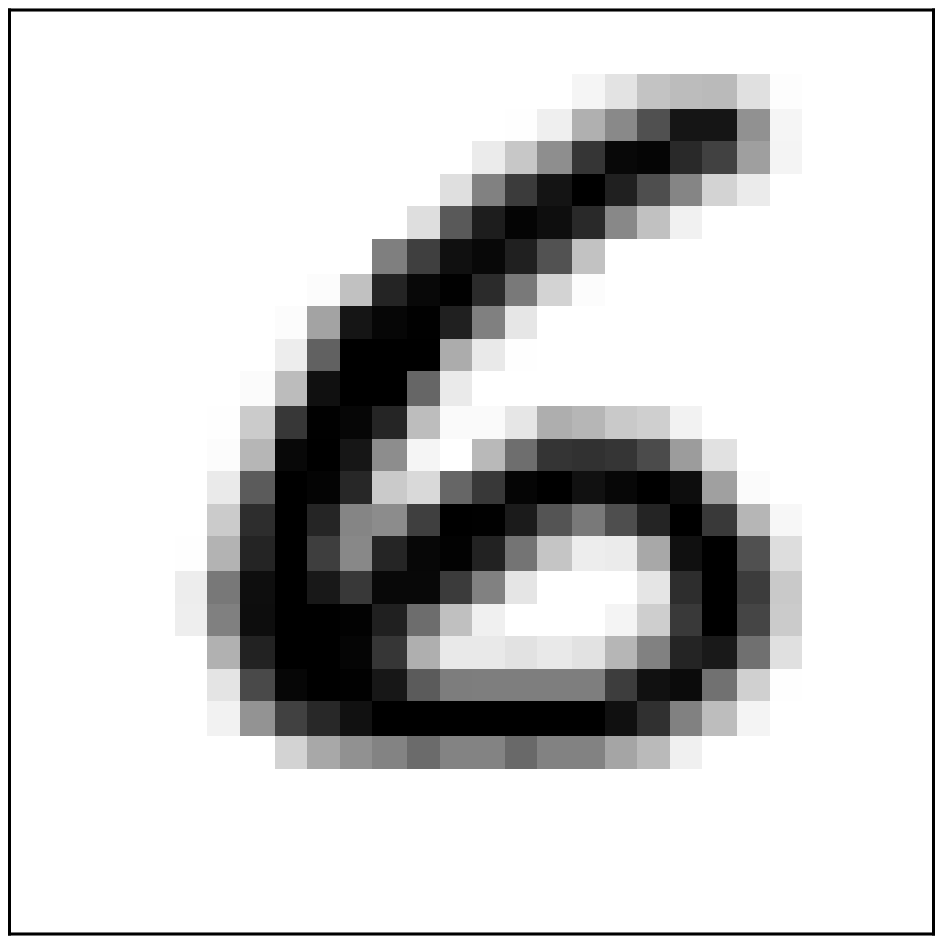}
                    \hspace{-0.2cm}
                    \includegraphics[width=0.49\textwidth]{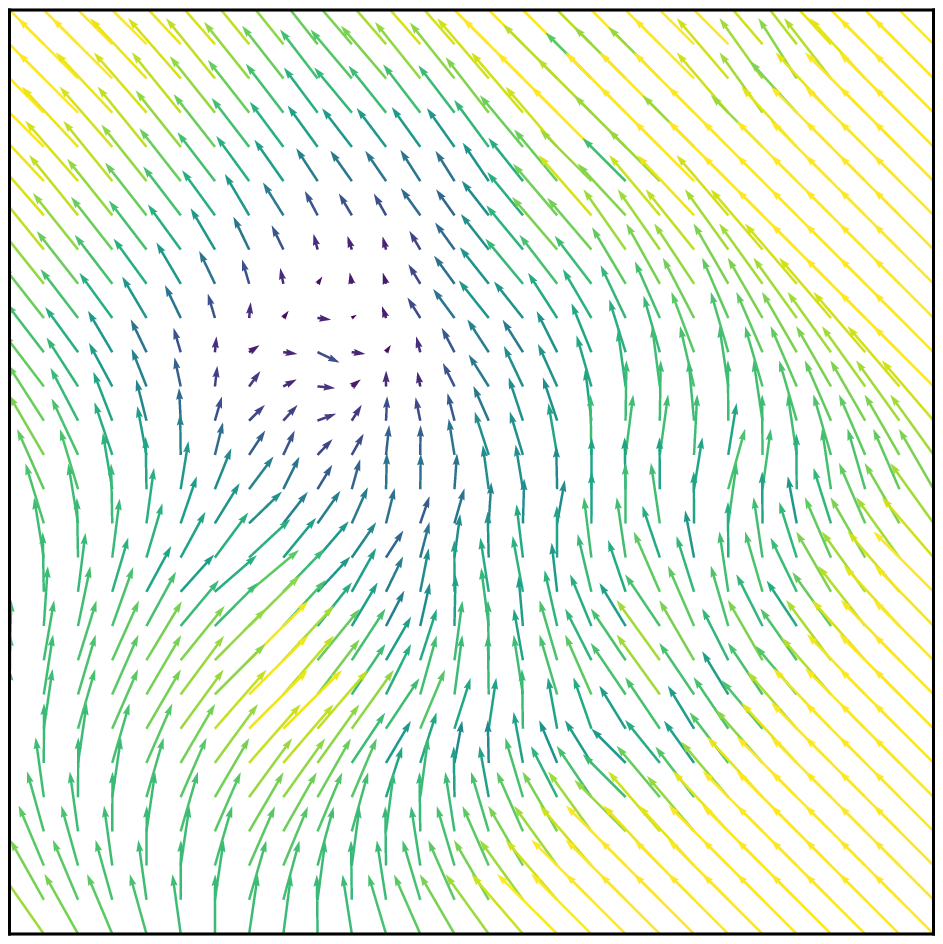}
                    \caption{label 6 (certified)}
                \end{subfigure}
                \begin{subfigure}[t]{.2\linewidth}
                    \includegraphics[width=0.49\textwidth]{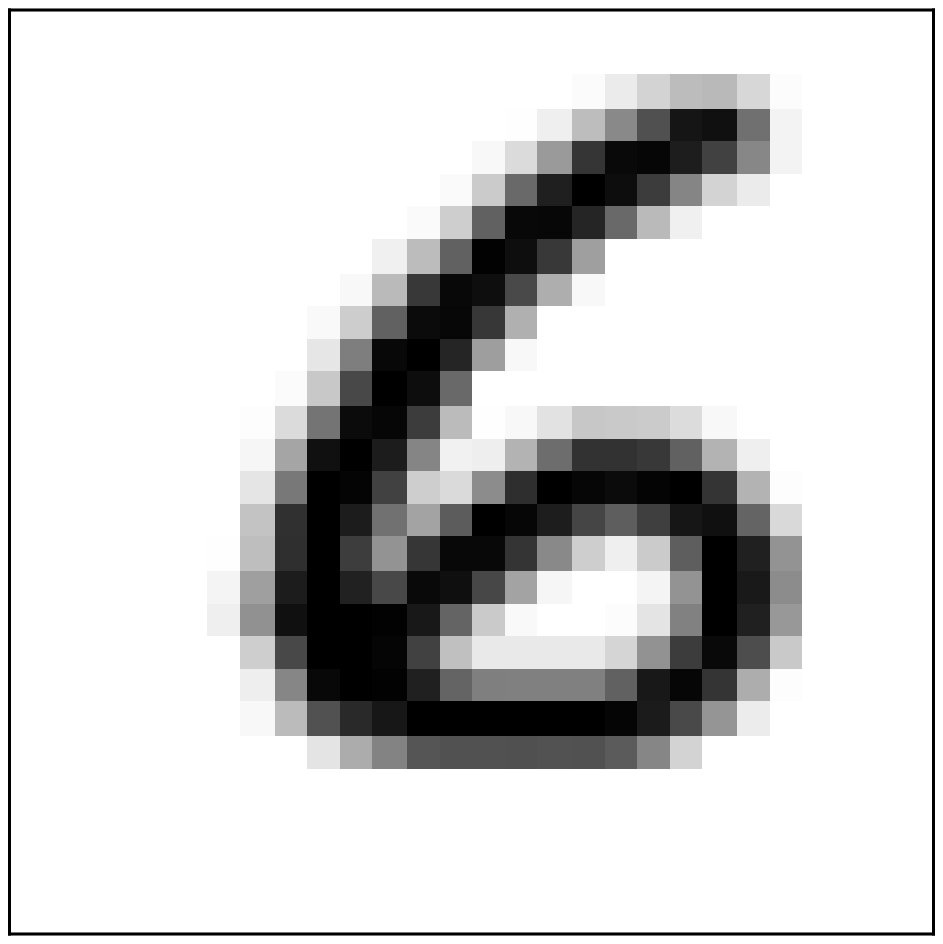}
                    \hspace{-0.2cm}
                    \includegraphics[width=0.49\textwidth]{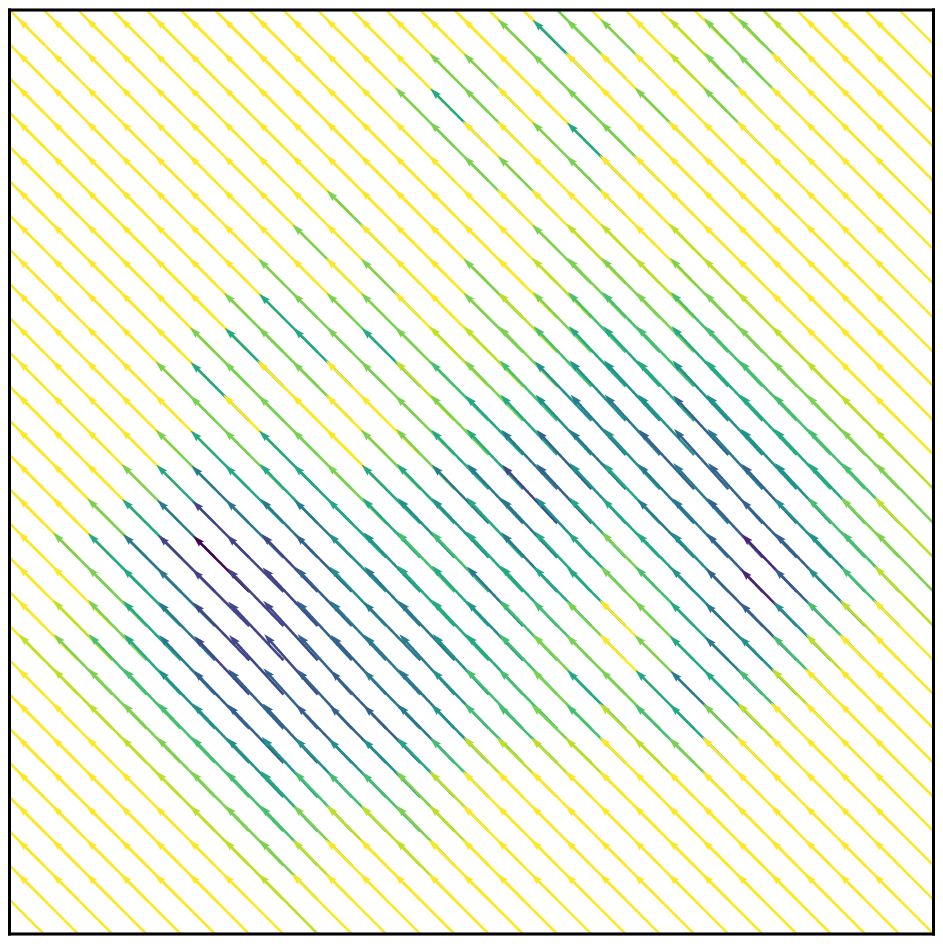}
                    \caption{label 6 (certified)}
                \end{subfigure}
                \begin{subfigure}[t]{.2\linewidth}
                    \includegraphics[width=0.49\textwidth]{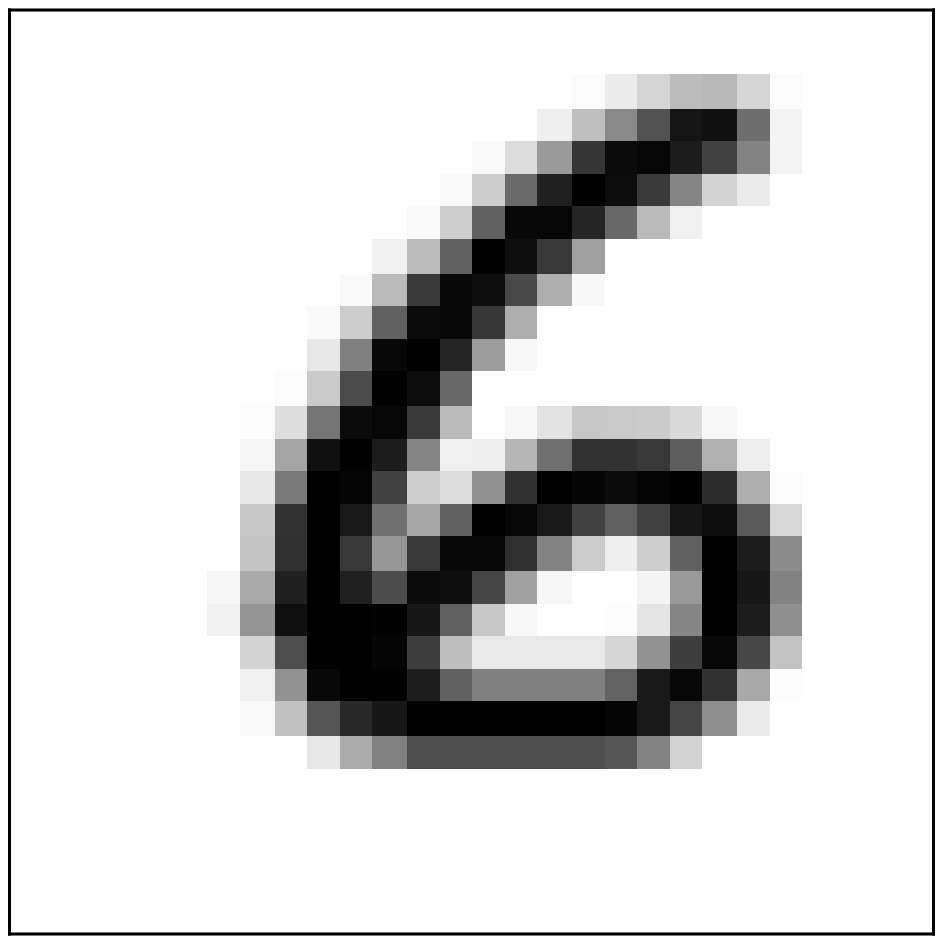}
                    \hspace{-0.2cm}
                    \includegraphics[width=0.49\textwidth]{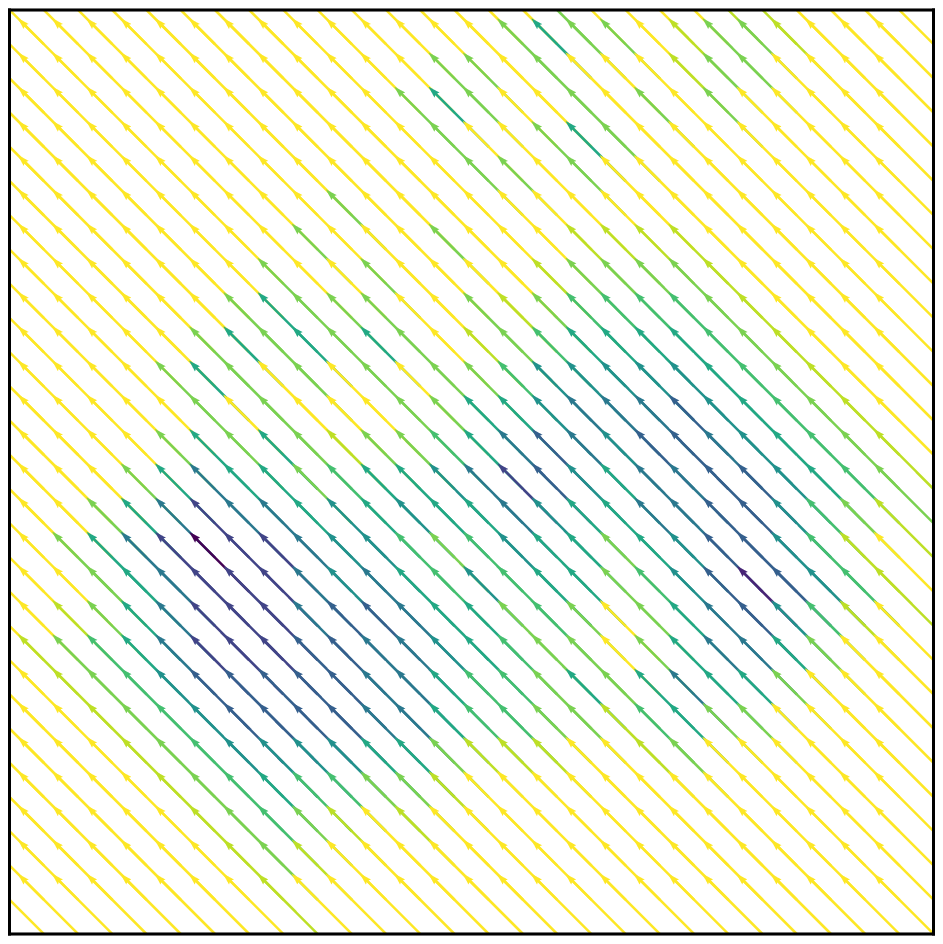}
                    \caption{label 6 (certified)}
                \end{subfigure}
            \end{center}
        \end{subfigure}

    \end{center}
    \caption{
        Image instances and corresponding deforming vector fields $\tau$ with
        displacement magnitude $\| \tau \|_{T_\infty} = \delta$ and flow
        $\gamma$.
    }
    \label{fig:visual-comparison}
\end{figure*}

\begin{table*}
    \begin{center}
        \begin{small}
            \begin{sc}
                \resizebox{0.5\linewidth}{!}{ \begin{tabular}{@{}lrrrrrr@{}}
    \toprule
                & \multicolumn{2}{c}{$T_1$-norm}    & \multicolumn{2}{c}{$T_2$-norm}    & \multicolumn{2}{c}{$T_\infty$-norm} \\
    $\delta$    & DeepPoly  & MILP                  & DeepPoly  & MILP                  & DeepPoly  & MILP \\
    \midrule
    0.3         & 0.1       & 1.4                   & 3.8       & 5.2                   & 0.1       & 1.7 \\
    0.5         & 0.1       & 1.5                   & 4.1       & 6.4                   & 0.1       & 1.9 \\
    0.7         & 0.1       & 2.4                   & 2.6       & 3.6                   & 0.1       & 4.0 \\
    0.9         & 0.1       & 5.6                   & 4.3       & 11.5                  & 0.1       & 40.8 \\
    \bottomrule
\end{tabular}

 }
            \end{sc}
        \end{small}
    \end{center}
    \caption{
        Average running times (in seconds) for $T_p$-norm certification with
        $\gamma = \infty$ of \textsc{ConvSmall DiffAI} on MNIST.
    }
    \label{tab:times-t-norm-comparison}
\end{table*}

\section{Visual Investigation}
\label{sec:visual-investigation}

In this section, we compare original and deformed MNIST images for different
$T_\infty$-norm values $\delta$ and flow-constraints $\gamma$.
To that end, we employ the \textsc{ConvSmall DiffAI} network together with our
convex relaxation with MILP to evaluate whether the deformed images are
adversarial or certified (or neither) and we display the images
in~\cref{fig:visual-comparison}.
Note that since both MILP and our interval bounds are exact for single-channel
images, the fact that~\cref{fig:visual-comparison:certified-but-not-attacked}
is not adversarial is due to failure of the attack, as an adversarial vector
field with $\delta = 0.3$ and $\gamma = \infty$ must exist (else MILP would
certify this image).
On the other hand,
for~\cref{fig:visual-comparison:not-certified-and-not-attacked} we cannot say
whether the attacker fails to find an adversarial image or if our convex
relaxation is not precise enough to certify the image (since we are using an
over-approximation to enforce the flow-constraints).

    }{}


    \message{^^JLASTPAGE \thepage^^J}

\end{document}